\documentclass[preprint, 12pt]{elsarticle}
\journal{elsevier}
\PassOptionsToPackage{table,dvipsnames}{xcolor}
\usepackage{xcolor}

\usepackage{xr}
\usepackage{array,bm,amsmath,amssymb,mathrsfs,epsfig,epstopdf}
\usepackage{amsthm}
\usepackage{setspace} 
\usepackage{CJK}
\usepackage{color}
\usepackage{lscape}
\usepackage{verbatim}
\usepackage{enumerate}
\usepackage{booktabs}
\usepackage{threeparttable}
\usepackage{multicol} 
\usepackage{longtable}
\usepackage{multirow}
\usepackage{bm}
\usepackage
{bbm}
\usepackage{cases}
\usepackage{caption}
\usepackage{natbib}
\usepackage{algorithm}
\usepackage{algorithmic}
\usepackage{geometry}
\usepackage{framed}
\usepackage[pdftex,bookmarksnumbered,bookmarksopen,colorlinks, citecolor=blue,linkcolor=blue,urlcolor=blue]{hyperref}
\usepackage{graphicx}
\usepackage{amsmath}
\usepackage{amsfonts}
\usepackage{float}
\usepackage{ragged2e}
\usepackage[figuresright]{rotating}
\usepackage{mathtools}
\usepackage{pifont}
\usepackage{indentfirst}  
\usepackage{listings}

\usepackage{tikz}
\usetikzlibrary{arrows.meta,positioning,calc}

\definecolor{codegreen}{rgb}{0,0.6,0}
\definecolor{codegray}{rgb}{0.5,0.5,0.5}
\definecolor{codepurple}{rgb}{0.58,0,0.82}
\definecolor{backcolour}{rgb}{0.95,0.95,0.92}
\lstdefinestyle{mystyle}{
backgroundcolor=\color{backcolour},   
commentstyle=\color{codegreen},
keywordstyle=\color{magenta},
numberstyle=\tiny\color{codegray},
stringstyle=\color{codepurple},
basicstyle=\ttfamily\footnotesize,
breakatwhitespace=false,         
breaklines=true,                 
captionpos=b,                    
keepspaces=true,                 
numbers=left,                    
numbersep=5pt,                  
showspaces=false,                
showstringspaces=false,
showtabs=false,                  
tabsize=2
}

\newcommand\gA{{\mathcal{A}}}
\newcommand\gB{{\mathcal{B}}}

\newcommand\gI{{\mathcal{I}}}

\newcommand\E{\mathbb{E}}

\DeclareMathOperator*{\argmin}{arg\,min}

\DeclareMathOperator{\Tr}{Tr}

\newtheorem{theorem}{Theorem}[section]
\newtheorem{lemma}{Lemma}[section]
\newtheorem{proposition}{Proposition}[section]

\newcommand{\tabincell}[2]{\begin{tabular}{@{}#1@{}}#2\end{tabular}}

\newtheorem{assumption}{Assumption}

\theoremstyle{remark}

\newcommand{\psitwo}{\psi_2}
\newcommand{\psione}{\psi_1}
\newcommand{\normpsitwo}[1]{\left\lVert #1 \right\rVert_{\psi_2}}
\newcommand{\normpsione}[1]{\left\lVert #1 \right\rVert_{\psi_1}}
\usepackage{pdflscape}

\newcommand{\farbetter}[1]{\cellcolor{green!18}{#1}}
\newcommand{\similar}[1]{\cellcolor{gray!12}{#1}}
\newcommand{\worse}[1]{\cellcolor{red!15}{#1}}
\newcommand{\legendbox}[1]{\textcolor{#1}{\rule{1.2ex}{1.2ex}}}

\begin{document}
\title{\bf Machine Learning-Assisted High-Dimensional Matrix Estimation}

\author[1,2]{Wan Tian\corref{cor1}}
\ead{wantian61@foxmail.com}

\author[6]{Hui Yang\corref{cor1}}
\ead{yanghui6@stu.pku.edu.cn}

\author[2]{Zhouhui Lian}
\ead{lianzhouhui@pku.edu.cn}

\author[3]{Lingyue Zhang}
\ead{lingyue_zhang@126.com}

\author[4,5,6]{Yijie Peng\corref{cor2}}
\ead{pengyijie@pku.edu.cn}

\cortext[cor1]{Equal contribution}
\cortext[cor2]{Corresponding author}

\address[1]{Advanced Institute of Information Technology, Peking University}
\address[2]{Wangxuan Institute of Computer Technology, Peking University, China, 100871}
\address[3]{School of Statistics, Dongbei University of Finance and Economics, Dalian, China, 116025}
\address[4]{PKU-Wuhan Institute for Artificial Intelligence}
\address[5]{Xiangjiang Laboratory, Changsha 410000, China}
\address[6]{Guanghua School of Management, Peking University, Beijing,  China, 100871}

\begin{abstract}
Efficient estimation of high-dimensional matrices—including covariance and precision matrices—is a cornerstone of modern multivariate statistics. Most existing studies have focused primarily on the theoretical properties of the estimators (e.g., consistency and sparsity), while largely overlooking the computational challenges inherent in high-dimensional settings. Motivated by recent advances in learning-based optimization methods—which integrate data-driven structures with classical optimization algorithms—we explore high-dimensional matrix estimation assisted by machine learning. Specifically, for the optimization problem of high-dimensional matrix estimation, we first present a solution procedure based on the Linearized Alternating Direction Method of Multipliers (LADMM). We then introduce learnable parameters and model the proximal operators in the iterative scheme with neural networks, thereby improving estimation accuracy and accelerating convergence. Theoretically, we first prove the convergence of LADMM, and then establish the convergence, convergence rate, and monotonicity of its reparameterized counterpart; importantly, we show that the reparameterized LADMM enjoys a faster convergence rate. Notably, the proposed reparameterization theory and methodology are applicable to the estimation of both high-dimensional covariance and precision matrices. We validate the effectiveness of our method by comparing it with several classical optimization algorithms across different structures and dimensions of high-dimensional matrices.

\end{abstract}

\maketitle

\section{Introduction} \label{sec1}



High-dimensional matrix estimation—covering both covariance and precision matrix estimation—constitutes a cornerstone of modern statistics and data science \citep{fan2016overview, wainwright2019high, giraud2021introduction}. Accurate covariance estimation enables the characterization of dependence structures among a large number of variables \citep{anderson1958introduction, lam2009sparsistency, pourahmadi2013high}, which is indispensable in diverse domains such as genomics \citep{li2010bayesian, serra2018robust}, neuroscience \citep{pang2016dimensionality}, finance \citep{fan2008high, fan2011high, moura2020comparing}, and climate science \citep{sarhadi2017advances, miftakhova2020statistical}. Precision matrix estimation—often interpreted as the learning of sparse graphical models—sheds light on conditional dependencies and enables network inference in complex systems \citep{fan2016overview, chen2013covariance, avella2018robust, jankova2017honest}, with significant applications in brain connectivity analysis \citep{ting2018multi}, gene regulatory network reconstruction \citep{ke2022high}, and spatio-temporal modeling \citep{bradley2015multivariate}. Consequently, accurate estimation of high-dimensional matrices is of fundamental importance from both statistical and optimization perspectives.

Over the past two decades, substantial progress has been made in the statistical theory of high-dimensional matrix estimation, particularly with respect to the accuracy of estimators, including properties such as sparsistency and consistency \citep{lam2009sparsistency, chen2013covariance, avella2018robust}. These theoretical results are typically derived under asymptotic regimes in which the dimensionality is assumed to grow large, even approaching infinity. However, in empirical studies, the dimensionality is often only on the order of tens to hundreds, and in many cases is comparable to the sample size \citep{wei2023large, fatima2024two, nguyen2022distributionally, zhu2023online}. This observation highlights a notable gap between the statistical theory of estimators and the practical challenges of their computational implementation.

For the computation of high-dimensional covariance matrices, existing solution methods are primarily developed within the convex optimization framework. Among them, the ADMM is the most widely used \citep{wen2019fast, sedghi2014multi, tan2019learning}. By introducing an auxiliary variable, ADMM decouples the quadratic term from the non-smooth $\ell_1$-penalty and leverages proximal operators—such as soft-thresholding and eigenvalue projection—to achieve efficient iterations. This makes it particularly suitable for medium- to high-dimensional settings; however, the $O(p^3)$ complexity of the eigen-decomposition required at each iteration poses a major computational bottleneck in ultra-high-dimensional regimes, where $p$ denotes the matrix dimension. Other classical approaches include proximal gradient descent \citep{xu2022proximal, tibshirani2010proximal} and its accelerated variant, FISTA \citep{kim2018another}. The former applies proximal operators directly to handle the $\ell_1$-penalty and projection constraints, offering a simple and easy-to-implement scheme but with relatively slow convergence. The latter improves the convergence rate from $O(1/k)$ to $O(1/k^2)$ by incorporating Nesterov’s momentum mechanism, making it better suited for low-dimensional problems or rapid prototyping, where $k$ denotes the iteration number. In addition, block coordinate descent \citep{treister2014block, qin2013efficient} methods are effective for structured matrices, as they reduce computational burden by optimizing over matrix blocks. Semidefinite programming (SDP)-based approaches, on the other hand, employ interior-point solvers to obtain highly accurate solutions but require prohibitively large memory, thus limiting their use to low-dimensional settings \citep{amini2008high}.

For the estimation of high-dimensional precision matrices, existing optimization methods are also primarily grounded in convex optimization techniques. Among them, the Graphical Lasso (GLasso) \citep{friedman2008sparse} represents a pioneering approach, achieving efficient estimation by iteratively solving row- or column-wise lasso \citep{tibshirani1996regression} subproblems. While well-suited for moderate-dimensional matrices, its $O(p^3)$ computational complexity limits its applicability in ultra-high-dimensional settings. To overcome this limitation, variants of ADMM have been widely adopted \citep{lin2024estimation, danaher2014joint, zhao2019optimization, hallac2017toeplitz}. By introducing auxiliary variables, these methods decouple the $-\log\det$ term from the non-smooth $\ell_1$-penalty, solving subproblems via proximal operators, and can further accelerate convergence through linearization of the quadratic term and the addition of proximal terms. The QUIC (Quadratic Inverse Covariance) algorithm  \citep{hsieh2014quic, hsieh2013big}, based on a quadratic approximation within a Newton-type method, achieves performance several times faster than standard GLasso while remaining memory-efficient. Meanwhile, the Dual-Primal Graphical Lasso (DP-GLASSO) \citep{zhang2020proximal} leverages a primal-dual framework to reduce the number of iterations and supports parallel computation.

It is thus evident that these methods, developed within the convex optimization framework, exhibit inherent limitations in terms of scalability and adaptability to the data. Even algorithms such as QUIC, which outperform the standard GLasso in computational efficiency, remain sensitive to the choice of regularization parameters, particularly when the matrix is near-singular. Moreover, QUIC has limited capacity to exploit the statistical properties of the data for adaptively guiding the optimization path. Of course, in communities such as statistics and machine learning, researchers often design task-specific loss functions and estimate parameters based on carefully crafted optimization strategies. However, these approaches may also encounter similar computational challenges \citep{xie2019differentiable}.

To alleviate computational challenges, the framework of learning-based optimization (LBO) has been proposed, which combines data-driven learning with traditional optimization algorithms. The key idea of LBO is to reparameterize certain operators within the optimization algorithm and introduce learnable parameters to enhance optimization performance \citep{gregor2010learning, liu2016learning, chen2016trainable}. Motivated by the success of deep learning across various applications, many studies have employed deep neural networks (DNNs) as the learning units to reparameterize the optimization process \citep{li2024pdhg, lin2025pdcs, yang2024efficient, li2024power, chen2024symilo, li2025towards}. For example, \citet{li2024pdhg} and \citet{yang2024efficient} respectively developed primal–dual hybrid gradient (PDHG)-based computational procedures for large-scale linear programming (in supervised learning settings) and quadratic programming (in unsupervised learning settings), introduced learnable parameters, and employed graph neural networks to represent the iterative process. They further provided theoretical guarantees on the number of neurons required to achieve a given accuracy. Similarly, \citet{xie2019differentiable} first derived the iterative scheme for linearly constrained convex optimization based on LADMM and then reparameterized the computation using neural networks. Notably, they were the first to theoretically demonstrate that such reparameterization can accelerate the convergence rate. However, these approaches have not been effectively applied to high-dimensional matrix estimation.

In this paper, we introduce LBO into the problem of high-dimensional matrix estimation to enhance both estimation accuracy and scalability. Specifically, we first derive an iterative scheme based on LADMM, then introduce learnable parameters and reparameterize the proximal operators within the iterations using neural networks. This framework is applicable to both high-dimensional covariance and precision matrix estimation. Theoretically, we first establish the convergence of LADMM; for the reparameterized LADMM, we further prove convergence, an explicit convergence rate, and a monotone descent property. Importantly, we show that the reparameterized variant achieves a strictly faster rate. In the empirical analysis, we compare the proposed method with several classical approaches for high-dimensional matrix estimation across matrices of diverse structures and dimensions, in order to demonstrate its superior performance.

The remainder of this paper is organized as follows. Section \ref{sec2} and Section \ref{sec3} discuss how to perform high-dimensional covariance and precision matrix estimation within the proposed LBO framework, respectively. Section \ref{sec4} presents a theoretical analysis of LADMM and the proposed LBO method. Section \ref{sec5} reports numerical experiments that validate the effectiveness of our approach. Finally, Section \ref{sec6} concludes the paper.




\section{High-dimensional covariance matrix estimation} \label{sec2}
In this section, we show how to leverage LBO to facilitate the estimation of high-dimensional covariance matrices. Specifically, we first formulate the original optimization problem and describe its ADMM and LADMM solution procedures. We then introduce learnable parameters and reparameterize the proximal operators via neural networks, and present the resulting iterative scheme.

The estimation of a high-dimensional covariance matrix \(\Sigma = (\sigma_{ij})_{1\le i,j\le p}\) can be pursued via a variety of methods, among which a particularly influential one is the soft-thresholding estimator \citep{rothman2009generalized}. By shrinking small off-diagonal entries of the sample covariance toward zero, it yields a sparse and interpretable estimator that is computationally scalable and, under standard sparsity assumptions, comes with strong theoretical guarantees (e.g., rate-optimality and consistency). In this paper, we adopt the soft-thresholding estimator as the objective of our optimization framework. This estimator is equivalent to the following optimization problem:
\[
\min_{\Sigma} \frac{1}{2} \lVert \Sigma - S \rVert^2_F + \lambda \lVert \Sigma  \rVert_{1,\text{off}},	
\]
where \(S \in \mathbb{R}^{p \times p}\) is the empirical covariance matrix computed from an observation matrix with \(n\) samples and \(p\) features, \( \|\cdot\|_F \) denotes the Frobenius norm of the matrix, and \(\lambda\) is the regularization parameter for promoting sparsity in the estimation, and \(\lVert \Sigma  \rVert_{1,\text{off}} := \sum_{i\neq j}|\sigma_{ij} |\). However, the positive definiteness of the covariance matrix estimated based on this optimization problem is only guaranteed with high probability in an asymptotic setting, and it may not hold in real-world scenarios. Following \citet{xue2012positive}, we impose a positive definite constraint on the soft-thresholding optimization problem, formalized as:
\begin{equation}\label{raw_high_dimension}
\min_{\Sigma \succeq \epsilon I} \frac{1}{2} \lVert \Sigma - S \rVert^2_F + \lambda \lVert \Sigma  \rVert_{1,\text{off}},
\end{equation}
where $\epsilon$ is an arbitrarily small positive number that does not require tuning, and $I \in \mathbb{R}^{p \times p}$ denotes the identity matrix. Prior to describing the LADMM solution for optimization problem (\ref{raw_high_dimension}), we first review the ADMM-based solution procedure. 

To apply ADMM, we introduce an auxiliary variable \(W\) and rewrite problem~(\ref{raw_high_dimension}) as
\[
\min_{\Sigma,\,W}\;
\frac{1}{2}\|\Sigma - S\|_F^2 \;+\; \lambda \|W\|_{1,\mathrm{off}} \;+\; \mathcal{I}_{\{\Sigma \succeq \epsilon I\}}(\Sigma)
\quad \text{s.t.} \quad \Sigma = W,
\]
where \(\mathcal{I}_{\{\Sigma \succeq \epsilon I\}}\) is the indicator that equals \(0\) if \(\Sigma \succeq \epsilon I\) and \(+\infty\) otherwise. Introducing a dual variable (Lagrange multiplier) \(\Lambda \in \mathbb{R}^{p \times p}\) and a penalty parameter \(\rho > 0\), the augmented Lagrangian for the split formulation is
\begin{equation}\label{raw_argu_large}
L_\rho(\Sigma, W, \Lambda)
= \frac{1}{2} \|\Sigma - S\|_F^2
+ \lambda \|W\|_{1,\mathrm{off}}
+ \mathcal{I}_{\{\Sigma \succeq \epsilon I\}}(\Sigma)
+ \langle \Lambda, \Sigma - W \rangle
+ \frac{\rho}{2} \|\Sigma - W\|_F^2,
\end{equation}
where \(\langle A, B \rangle = \mathrm{Tr}(A^\top B)\) denotes the Frobenius inner product. ADMM then alternates between minimizing \(L_\rho\) with respect to \(\Sigma\) and \(W\), followed by a dual ascent step on \(\Lambda\). This splitting decouples the smooth quadratic fit, the non-smooth off-diagonal \(\ell_1\)-penalty, and the PSD constraint; in particular, the \(W\)-update reduces to off-diagonal soft-thresholding, while the \(\Sigma\)-update enforces \(\Sigma \succeq \epsilon I\).

Given the iterates \(\Sigma^{(k)}, W^{(k)}\), and \(\Lambda^{(k)}\) at iteration \(k\), the \((k+1)\)-st update proceeds as follows. 
The \(\Sigma\)-update is obtained by solving
\begin{equation}\label{admmSigma}
\Sigma^{(k+1)}
= \arg\min_{\Sigma \succeq \epsilon I}\;
\frac{1}{2}\|\Sigma - S\|_F^2
+ \langle \Lambda^{(k)}, \Sigma \rangle
+ \frac{\rho}{2}\|\Sigma - W^{(k)}\|_F^2,
\end{equation}
which is a convex quadratic problem with a positive–semidefinite (PSD) constraint. 
Completing the square shows that \eqref{admmSigma} is equivalent (up to an additive constant) to
\[
\Sigma^{(k+1)}=\arg\min_{\Sigma \succeq \epsilon I}\;
\frac{1+\rho}{2}\,\big\|\Sigma-\overline{\Sigma}^{(k+1)}\big\|_F^2,
\qquad
\overline{\Sigma}^{(k+1)} \;=\; \frac{S+\rho W^{(k)}-\Lambda^{(k)}}{1+\rho}.
\]
Hence, \(\Sigma^{(k+1)}\) is the projection of \(\overline{\Sigma}^{(k+1)}\) onto the convex cone \(\{\Sigma \succeq \epsilon I\}\). 
Let the eigendecomposition of \(\overline{\Sigma}^{(k+1)}\) be
\(\overline{\Sigma}^{(k+1)} = Q\,\mathrm{diag}(\lambda_1,\dots,\lambda_p)\,Q^\top\).
The PSD projection with eigenvalue floor \(\epsilon\) yields
\[
\Sigma^{(k+1)} \;=\; Q\,\mathrm{diag}\!\big(\max\{\lambda_1,\epsilon\},\dots,\max\{\lambda_p,\epsilon\}\big)\,Q^\top.
\]

We next update the auxiliary variable \(W\) by solving
\begin{equation}\label{admmz}
W^{(k+1)}
= \arg\min_{W}\;
\lambda \|W\|_{1,\mathrm{off}}
- \langle \Lambda^{(k)}, W \rangle
+ \frac{\rho}{2}\,\|\Sigma^{(k+1)} - W\|_{F}^{2},
\end{equation}
which, after completing the square, is equivalent (up to an additive constant) to
\[
W^{(k+1)}
= \arg\min_{W}\;
\lambda \|W\|_{1,\mathrm{off}}
+ \frac{\rho}{2}\,\big\|W - B^{(k)}\big\|_{F}^{2},
\qquad
B^{(k)} := \Sigma^{(k+1)} + \frac{1}{\rho}\Lambda^{(k)}.
\]
This subproblem decouples elementwise and admits the proximal (soft–thresholding) solution on the
off–diagonals:
\[
W^{(k+1)}_{ij}=
\begin{cases}
\operatorname{sgn}\!\big(B^{(k)}_{ij}\big)\,
\max(|B^{(k)}_{ij}|-\lambda/\rho,\,0\big), & i\neq j,\\[4pt]
B^{(k)}_{ii}, & i=j,
\end{cases}
\]
i.e., \(W^{(k+1)}=\mathcal{S}^{\mathrm{off}}_{\lambda/\rho}\!\big(B^{(k)}\big)\) where
\(\mathcal{S}^{\mathrm{off}}_{\lambda/\rho}\) applies elementwise soft–thresholding with threshold \(\lambda/\rho\) to the
off–diagonal entries and leaves the diagonal unchanged.
Finally, the dual variable is updated by
\begin{equation}\label{ADMdual}
\Lambda^{(k+1)} \;=\; \Lambda^{(k)} \;+\; \rho\big(\Sigma^{(k+1)} - W^{(k+1)}\big).
\end{equation}

The core idea of LADMM is to linearize the quadratic terms in the covariance matrix estimation step (\ref{admmSigma}) and the auxiliary variable estimation step (\ref{admmz}) of ADMM, and to add proximal terms to ensure convergence. The augmented Lagrangian (\ref{raw_argu_large}) can be equivalently written as:
\begin{equation} \label{raw_argu_large_scale}
L_\rho(\Sigma, W, \Lambda) = \frac{1}{2} \|\Sigma - S\|_F^2 + \lambda \|W\|_{1,\text{off}} + \mathcal{I}_{\{\Sigma \succeq \epsilon I\}}(\Sigma) + \frac{\rho}{2} \|\Sigma - W + \Lambda / \rho\|_F^2.
\end{equation}

At this stage, the update of the covariance matrix $\Sigma$ is approximated by solving the following optimization problem:
\begin{align}
\Sigma^{(k+1)} & = \arg\min_{\Sigma \succeq \epsilon I} 
\frac{1}{2} \|\Sigma - S\|_F^2 
+ \langle \rho (\Sigma^{(k)} - W^{(k)} + \Lambda^{(k)}/\rho), \Sigma - \Sigma^{(k)} \rangle
+ \frac{\rho \phi_1}{2} \|\Sigma - \Sigma^{(k)}\|_F^2\\
& = \arg\min_{\Sigma \succeq \epsilon I} 
\frac{1}{2} \|\Sigma - S\|_F^2 
+ \frac{\rho \phi_1}{2} \|\Sigma - \overline{\Sigma}^{(k)}\|_F^2\\
& = \mathrm{prox}_{f / (\rho \phi_1)}\!\left(\overline{\Sigma}^{(k)}\right),
\end{align}
where \(f(\Sigma) = \frac{1}{2}\|\Sigma - S\|_F^2 + \mathcal{I}_{\{\Sigma \succeq \epsilon I\}}(\Sigma)\), \(\overline{\Sigma}^{(k)} = ((\phi_1 - 1)\Sigma^{(k)} + W^{(k)} - \Lambda^{(k)} / \rho)/\phi_1\). Analogously, the auxiliary variable $W$ is updated by solving the following approximate optimization problem:
\begin{align}
W^{(k+1)} & = \arg\min_{W} \ \lambda \|W\|_1 + \langle \rho (W^{(k)} - \Sigma^{(k+1)} - \Lambda^{(k)} / \rho), W - W^{(k)} \rangle + \frac{\rho \phi_2}{2} \|W - W^{(k)}\|_F^2\\
&= \mathrm{prox}_{\lambda / (\rho \phi_2) \|\cdot\|_1} (\overline{W}^{(k)}),
\end{align}
where \(\overline{W}^{(k)} = ((\phi_2 - 1) W^{(k)} + \Sigma^{(k+1)} + \Lambda^{(k)} / \rho ) / \phi_2\). The appropriate choices of $\phi_1$ and $\phi_2$ will be discussed later. According to the preceding discussion, the iterative procedure for solving the optimization problem (\ref{raw_high_dimension}) using LADMM is as follows:
\begin{equation} \label{LADMMmatrix_covariance}
\left\{
\begin{aligned}
\Sigma^{(k+1)} & = \mathrm{prox}_{f / (\rho \phi_1)}\!\left(\Sigma^{(k)} + \frac{1}{\phi_1}(W^{(k)} - \Lambda^{(k)} / \rho - \Sigma^{(k)}) \right),\\
W^{(k+1)} & = \mathrm{prox}_{\lambda / (\rho \phi_2) \|\cdot\|_1} \left(
W^{(k)} + \frac{1}{\phi_2}(\Sigma^{(k+1)} + \Lambda^{(k)} / \rho - W^{(k)})
\right),\\
\Lambda^{(k+1)} & = \Lambda^{(k)} +  \rho (\Sigma^{(k+1)} - W^{(k+1)}).
\end{aligned}
\right.
\end{equation}

Compared with standard ADMM, LADMM replaces the quadratic penalties in the
\(\Sigma\)- and \(W\)-subproblems with first–order (linearized) approximations
around the current iterate and adds proximal (quadratic) stabilization terms.
This yields cheaper per–iteration updates, improves numerical stability, and—
with appropriate choices of the proximal parameters—retains global convergence
with faster practical progress.  Using the scaled form of the dual variable, the
augmented Lagrangian in \eqref{raw_argu_large} can be equivalently written as
\begin{equation}\label{raw_argu_large_scale}
L_\rho(\Sigma,W,\Lambda)
= \tfrac{1}{2}\|\Sigma-S\|_F^2
+ \lambda\|W\|_{1,\mathrm{off}}
+ \mathcal{I}_{\{\Sigma\succeq \epsilon I\}}(\Sigma)
+ \tfrac{\rho}{2}\,\big\|\Sigma - W + \Lambda/\rho\big\|_F^2 .
\end{equation}

(Linearized \(\Sigma\)-update). At iteration \(k\), linearize the quadratic term in \(\Sigma\) at \(\Sigma^{(k)}\) and
add a proximal term with parameter \(\phi_1>0\):
\begin{align*}
\Sigma^{(k+1)}
&= \arg\min_{\Sigma\succeq \epsilon I}\;
\frac{1}{2}\|\Sigma-S\|_F^2
+ \big\langle \rho\big(\Sigma^{(k)} - W^{(k)} + \Lambda^{(k)}/\rho\big),\, \Sigma-\Sigma^{(k)} \big\rangle
+ \frac{\rho\phi_1}{2}\,\|\Sigma-\Sigma^{(k)}\|_F^2 \\[2pt]
&= \arg\min_{\Sigma\succeq \epsilon I}\;
\frac{1}{2}\|\Sigma-S\|_F^2
+ \frac{\rho\phi_1}{2}\,\big\|\Sigma-\overline{\Sigma}^{(k)}\big\|_F^2
\;=\; \mathrm{prox}_{\,f/(\rho\phi_1)}\!\big(\overline{\Sigma}^{(k)}\big),
\end{align*}
where \(f(\Sigma)=\frac{1}{2}\|\Sigma-S\|_F^2+\mathcal{I}_{\{\Sigma\succeq \epsilon I\}}(\Sigma)\) and
\[
\overline{\Sigma}^{(k)}
= \Sigma^{(k)} - \frac{1}{\phi_1}\Big(\Sigma^{(k)}-W^{(k)}+\Lambda^{(k)}/\rho\Big)
= \frac{(\phi_1-1)\Sigma^{(k)} + W^{(k)} - \Lambda^{(k)}/\rho}{\phi_1}\, .
\]

(Linearized \(W\)-update). Analogously, linearize the quadratic term in \(W\) at \(W^{(k)}\) and add a proximal term with
parameter \(\phi_2>0\):
\begin{align*}
W^{(k+1)}
&= \arg\min_{W}\;
\lambda\|W\|_{1,\mathrm{off}}
+ \big\langle \rho\big(W^{(k)}-\Sigma^{(k+1)}-\Lambda^{(k)}/\rho\big),\, W-W^{(k)} \big\rangle
+ \frac{\rho\phi_2}{2}\,\|W-W^{(k)}\|_F^2 \\[2pt]
&= \mathrm{prox}_{\,(\lambda/(\rho\phi_2))\,\|\cdot\|_{1,\mathrm{off}}}\!\big(\overline{W}^{(k)}\big),\\
\overline{W}^{(k)}
&= W^{(k)} - \frac{1}{\phi_2}\Big(W^{(k)}-\Sigma^{(k+1)}-\Lambda^{(k)}/\rho\Big)
= \frac{(\phi_2-1)W^{(k)} + \Sigma^{(k+1)} + \Lambda^{(k)}/\rho}{\phi_2}\, .
\end{align*}

(Scaled dual update). The (scaled) dual variable is then updated by
\[
\Lambda^{(k+1)} \;=\; \Lambda^{(k)} \;+\; \rho\big(\Sigma^{(k+1)}-W^{(k+1)}\big).
\]

Collecting the above steps, the LADMM scheme for \eqref{raw_high_dimension} is
\begin{equation}\label{LADMMmatrix_covariance}
\left\{
\begin{aligned}
\Sigma^{(k+1)}
&= \mathrm{prox}_{\,f/(\rho\phi_1)}\!\left(
\Sigma^{(k)} + \tfrac{1}{\phi_1}\big(W^{(k)} - \Lambda^{(k)}/\rho - \Sigma^{(k)}\big)
\right),\\[4pt]
W^{(k+1)}
&= \mathrm{prox}_{\,(\lambda/(\rho\phi_2))\,\|\cdot\|_{1,\mathrm{off}}}\!\left(
W^{(k)} + \tfrac{1}{\phi_2}\big(\Sigma^{(k+1)} + \Lambda^{(k)}/\rho - W^{(k)}\big)
\right),\\[4pt]
\Lambda^{(k+1)}
&= \Lambda^{(k)} + \rho\big(\Sigma^{(k+1)} - W^{(k+1)}\big).
\end{aligned}
\right.
\end{equation}

\noindent
Here \(\mathrm{prox}_{u \varphi }(V):=\arg\min_U\{\,\varphi(U)+\tfrac{1}{2u}\|U-V\|_F^2\,\}\) denotes the proximal
operator.



LBO augments model–based iterative solvers with data–driven components, yielding updates that (i) adapt to the target problem family, (ii) amortize computation across instances once trained, and (iii) accelerate convergence while preserving the algorithmic structure and constraints. Motivated by these advantages—and by the theoretically and empirically superior convergence behavior of LADMM over ADMM—we endow the LADMM iteration in \eqref{LADMMmatrix_covariance} with learnable, stage–wise parameters and reparameterize each update as a neural block. The resulting unrolled architecture maintains the interpretability and constraint handling of the original method while enabling task-specific adaptation and improved practical performance.
Concretely, for \(k=0,\dots,K-1\), we set
\begin{equation}\label{LADMMmatrix_neural_covariance}
\left\{
\begin{aligned}
\Sigma^{(k+1)}
&= \eta_{(\omega_1)_k}\!\left(
\Sigma^{(k)}
- \alpha_k \odot \Big(\Lambda^{(k)} + \gamma_k \odot \big(\Sigma^{(k)} - W^{(k)}\big)\Big)
\right),\\[4pt]
W^{(k+1)}
&= \xi_{(\omega_2)_k}\!\left(
W^{(k)}
+ \beta_k \odot \Big(\Lambda^{(k)} + \gamma_k \odot \big(\Sigma^{(k+1)} - W^{(k)}\big)\Big)
\right),\\[4pt]
\Lambda^{(k+1)}
&= \Lambda^{(k)} \;+\; \gamma_k \odot \big(\Sigma^{(k+1)} - W^{(k+1)}\big),
\end{aligned}
\right.
\end{equation}
where \(\odot\) denotes the Hadamard (elementwise) product. The collection
\(\{(\omega_1)_k,(\omega_2)_k,\alpha_k,\beta_k,\gamma_k\}_{k=1}^{K}\) comprises
the learnable parameters: \(\alpha_k,\beta_k,\gamma_k\) are (scalar, diagonal, or
entrywise) step–size/penalty schedules, while \(\eta_{(\omega_1)_k}\) and
\(\xi_{(\omega_2)_k}\) are neural network blocks that act as learned proximal
operators (mapping symmetric matrices to symmetric matrices and preserving the
required structural constraints such as PSD enforcement or diagonal handling).
Once trained, the unrolled \(K\)-block network implements an LBO solver that
retains interpretability and structure while achieving improved practical accuracy
and convergence speed.

Setting \(\eta_{(\omega_1)_k}=\mathrm{prox}_{\,f/(\rho\phi_1)}\),
\(\xi_{(\omega_2)_k}=\mathrm{prox}_{\,(\lambda/(\rho\phi_2))\|\cdot\|_{1,\mathrm{off}}}\),
\(\alpha_k=\tfrac{1}{\phi_1}\), \(\beta_k=\tfrac{1}{\phi_2}\), and \(\gamma_k\equiv 1\)
reduces \eqref{LADMMmatrix_neural_covariance} to the classical LADMM iteration
\eqref{LADMMmatrix_covariance}. Under learning, these components become
data–adaptive while preserving the interpretability and feasibility constraints
of the model–based solver.

Furthermore, as (\ref{raw_high_dimension}) is an unsupervised optimization problem, and the iterative procedure in (\ref{LADMMmatrix_neural_covariance}) requires parameter updates via error backpropagation, we must define an optimization objective. The primal problem is $\min_{\Sigma} f_1(\Sigma) = \frac{1}{2} \|\Sigma - S\|_F^2 + \lambda \|\Sigma\|_{1,\mathrm{off}}$, with the corresponding dual $\max_{\Lambda} d_1
(\Lambda) = -\langle \Lambda, S \rangle - \frac{1}{2} \|\Lambda\|_F^2$. In the experiments, we employ the duality gap 
\[
f_1(\Sigma) - d_1(\Lambda)
\]
as the per-iteration loss to update the block neural network. Assuming the final estimated covariance matrix is $\widehat{\Sigma}$ and the dual multiplier matrix is $\widehat{\Lambda}$,

\section{High-dimensional precision matrix estimation}  \label{sec3}


In this section we consider high-dimensional precision–matrix estimation within the LBO framework. As our optimization target we adopt the graphical Lasso \citep{friedman2008sparse}, which solves
\begin{equation}\label{graphical_lasso}
\min_{\Theta \succeq \epsilon I}\; \mathrm{Tr}(S\Theta)\;-\;\log\det(\Theta)\;+\;\lambda\|\Theta\|_{1,\mathrm{off}}.
\end{equation}
This choice offers several advantages: (i) the objective is convex with an explicit PSD constraint, so a global minimizer exists and can be found reliably; 
(ii) the off–diagonal \(\ell_1\)-penalty induces sparsity in \(\Theta\), yielding an interpretable conditional-independence graph; 
(iii) the problem admits efficient proximal/ADMM/LADMM updates (soft-thresholding on off-diagonals and PSD projection), making it well suited to unrolling and LBO; and 
(iv) it enjoys well-studied statistical guarantees in the high-dimensional regime. 
Moreover, many popular nonconvex penalties—such as SCAD \citep{fan2001variable} and MCP \citep{zhang2010nearly}—can be handled via local linear approximation (LLA) \citep{zou2008one}, which converts each LLA step into a weighted graphical-Lasso subproblem. Thus graphical Lasso serves as a unifying and computationally convenient objective for LBO-based precision estimation.

Analogous to Section \ref{sec2}, we first derive an ADMM scheme for solving the graphical Lasso in
\eqref{graphical_lasso}. Introducing an auxiliary variable \(Z\) and enforcing
\(\Theta = Z\) yields the equivalent split formulation
\[
\min_{\Theta \succeq \epsilon I,\, Z}\;
\mathrm{Tr}(S\Theta) - \log\det(\Theta) + \lambda \|Z\|_{1,\mathrm{off}}
\quad \text{s.t.}\quad \Theta = Z .
\]
With a Lagrange multiplier \(U\) and penalty parameter \(\rho>0\), the augmented
Lagrangian is
\[
L_{\rho}(\Theta,Z,U)
= \mathrm{Tr}(S\Theta) - \log\det(\Theta)
+ \lambda \|Z\|_{1,\mathrm{off}}
+ \langle U,\, \Theta - Z \rangle
+ \frac{\rho}{2}\,\|\Theta - Z\|_{F}^{2},
\]
where \(\langle A,B\rangle=\mathrm{Tr}(A^\top B)\).
At iteration \(k\), ADMM performs the updates
\begin{equation}
\left\{
\begin{aligned}
\Theta^{(k+1)}
&= \arg\min_{\Theta \succeq \epsilon I}\;
\mathrm{Tr}(S\Theta) - \log\det(\Theta)
+ \tfrac{\rho}{2}\,\big\|\Theta - Z^{(k)} + U^{(k)}/\rho\big\|_{F}^{2},
\\[4pt]
Z^{(k+1)}
&= \operatorname{prox}_{\frac{\lambda}{\rho}\|\cdot\|_{1,\mathrm{off}}}
\!\Big(\Theta^{(k+1)} + U^{(k)}/\rho\Big),
\\[4pt]
U^{(k+1)}
&= U^{(k)} + \rho\big(\Theta^{(k+1)} - Z^{(k+1)}\big).
\end{aligned}
\right.
\end{equation}
The \(Z\)-update is an off–diagonal soft–thresholding operator, while the
\(\Theta\)-update is a convex subproblem that can be solved efficiently via an
eigendecomposition-based proximal step and PSD enforcement \(\Theta \succeq \epsilon I\).

Next, we derive the LADMM algorithm. By linearizing the quadratic term and adding a proximal term, we obtain the following approximation:
\begin{align*}
\Theta^{(k+1)} 
& = \arg\min_{\Theta \succeq \epsilon I} \;
\operatorname{Tr}(S\Theta) - \log \det \Theta
+ \big\langle\rho(\Theta^{(k)} - Z^{(k)})+U^{(k)},\Theta-\Theta^{(k)}\big\rangle+\frac{\rho \phi_1}{2}\|\Theta - \Theta^{(k)}\|_F^2 \\
& = \arg\min_{\Theta \succeq \epsilon I} \;
\operatorname{Tr}(S\Theta) - \log \det \Theta
+\frac{\rho \phi_1}{2}\|\Theta - \overline{\Theta}^{(k)}\|_F^2 \\
& = \operatorname{prox}_{g/(\rho \phi_1)}(\overline{\Theta}^{(k)}),
\end{align*}
where $\overline{\Theta}^{(k)}=\Theta^{(k)}- (\Theta^{(k)} - Z^{(k)}+U^{(k)}/\rho) / \phi_1, g(\Theta)=\operatorname{Tr}(S\Theta) - \log \det \Theta+\mathcal{I}_{\{\Theta\succeq\epsilon I\}}(\Theta)$. 
Analogously,
\begin{align*}
Z^{(k+1)} &= \arg\min_{Z}\lambda\|Z\|_{1,\text{off}}+\left\langle\rho(Z^{(k)}-\Theta^{(k+1)})-U^{(k)}, Z-Z^{(k)}\right\rangle + \frac{\rho \phi_2}{2}\|Z-Z^{k}\|_F^2, \\
&= \operatorname{prox}_{(\lambda/(\rho \phi_2))\|\cdot\|_{\text{1, off}}}(\overline{V}^{(k)}),
\end{align*}
where $\quad\overline{V}^{(k)} = Z^{(k)}- (Z^{(k)} - \Theta^{(k+1)}-U^{(k)}/\rho) / \phi_2.$ Based on the preceding discussion, the LADMM-based iterative procedure for Optimization Problem (\ref{graphical_lasso}) consists of the following three steps:
\begin{equation}
\label{LADMMmatrix_precision}
\left\{
\begin{aligned}
\Theta^{(k+1)} & = \operatorname{prox}_{g/(\rho \phi_1)} \left(
\Theta^{(k)}- \frac{1}{\phi_1}(\Theta^{(k)} - Z^{(k)}+U^{(k)}/\rho)
\right), \\
Z^{(k+1)} &= \operatorname{prox}_{(\lambda/(\rho \phi_2))\|\cdot\|_{\text{1, off}}}\left(
Z^{(k)}- \frac{1}{ \phi_2}(Z^{(k)} - \Theta^{(k+1)}-U^{(k)}/\rho)
\right), \\
U^{(k+1)} &= U^{(k)} + \rho \,(\Theta^{(k+1)} - Z^{(k+1)}).
\end{aligned}
\right.
\end{equation}

Next, we derive a  LADMM scheme.  At each iteration we
linearize the quadratic coupling term in the ADMM subproblems at the current
iterate and add a proximal (quadratic) stabilization.  This yields cheaper,
single–matrix updates while preserving global convergence for suitable choices
of the proximal parameters \(\phi_1,\phi_2>0\) (typically \(\phi_i\ge 1\)).

(Linearized \(\Theta\)-update). Starting from the \(\Theta\)-subproblem in ADMM, we linearize
\(\frac{\rho}{2}\|\Theta-Z^{(k)}+U^{(k)}/\rho\|_F^2\) at \(\Theta^{(k)}\) and add
\(\frac{\rho\phi_1}{2}\|\Theta-\Theta^{(k)}\|_F^2\), which gives the surrogate
\begin{align*}
\Theta^{(k+1)}
&= \arg\min_{\Theta\succeq \epsilon I}\;
\operatorname{Tr}(S\Theta) - \log\det(\Theta)
+ \big\langle \rho(\Theta^{(k)}-Z^{(k)})+U^{(k)},\,\Theta-\Theta^{(k)} \big\rangle
+ \frac{\rho\phi_1}{2}\,\|\Theta-\Theta^{(k)}\|_F^2 \\
&= \arg\min_{\Theta\succeq \epsilon I}\;
\operatorname{Tr}(S\Theta) - \log\det(\Theta)
+ \frac{\rho\phi_1}{2}\,\big\|\Theta-\overline{\Theta}^{(k)}\big\|_F^2
\\
&=\; \operatorname{prox}_{\,g/(\rho\phi_1)}\!\big(\overline{\Theta}^{(k)}\big),
\end{align*}
where
\[
\overline{\Theta}^{(k)}
= \Theta^{(k)} - \frac{1}{\phi_1}\Big(\Theta^{(k)} - Z^{(k)} + U^{(k)}/\rho\Big),
\qquad
g(\Theta)=\operatorname{Tr}(S\Theta)-\log\det(\Theta)+\mathcal{I}_{\{\Theta\succeq \epsilon I\}}(\Theta).
\]
The proximal map of \(g\) admits a closed form via eigendecomposition: letting
\(\overline{\Theta}^{(k)}-\frac{1}{\rho\phi_1}S = Q\operatorname{diag}(d_i)Q^\top\), one obtains
\[
\operatorname{prox}_{\,g/(\rho\phi_1)}\!\big(\overline{\Theta}^{(k)}\big)
= Q\,\operatorname{diag}\!\Big(\max\!\Big(\tfrac{d_i+\sqrt{d_i^{2}+\tfrac{4}{\rho\phi_1}}}{2},\,\epsilon\Big)\Big)\,Q^\top .
\]

(Linearized \(Z\)-update). Analogously, linearizing the quadratic term at \(Z^{(k)}\) and adding
\(\frac{\rho\phi_2}{2}\|Z-Z^{(k)}\|_F^2\) yields
\begin{align*}
Z^{(k+1)}
&= \arg\min_{Z}\;
\lambda\|Z\|_{1,\mathrm{off}}
+ \big\langle \rho(Z^{(k)}-\Theta^{(k+1)})-U^{(k)},\,Z-Z^{(k)} \big\rangle
+ \frac{\rho\phi_2}{2}\,\|Z-Z^{(k)}\|_F^2 \\
&= \operatorname{prox}_{\,(\lambda/(\rho\phi_2))\|\cdot\|_{1,\mathrm{off}}}
\!\big(\overline{V}^{(k)}\big),
\end{align*}
where
\[
\overline{V}^{(k)}
= Z^{(k)} - \frac{1}{\phi_2}\Big(Z^{(k)}-\Theta^{(k+1)}-U^{(k)}/\rho\Big),
\]
This proximal map is the off–diagonal soft–thresholding operator with threshold
\(\lambda/(\rho\phi_2)\), leaving the diagonal unchanged. Collecting the updates, the LADMM procedure for \eqref{graphical_lasso} reads
\begin{equation}\label{LADMMmatrix_precision}
\left\{
\begin{aligned}
\Theta^{(k+1)}
&= \operatorname{prox}_{\,g/(\rho \phi_1)} \!\left(
\Theta^{(k)} - \tfrac{1}{\phi_1}\big(\Theta^{(k)} - Z^{(k)} + U^{(k)}/\rho\big)
\right), \\[3pt]
Z^{(k+1)}
&= \operatorname{prox}_{\,(\lambda/(\rho \phi_2))\|\cdot\|_{1,\mathrm{off}}}\!\left(
Z^{(k)} - \tfrac{1}{\phi_2}\big(Z^{(k)} - \Theta^{(k+1)} - U^{(k)}/\rho\big)
\right), \\[3pt]
U^{(k+1)}
&= U^{(k)} + \rho \big(\Theta^{(k+1)} - Z^{(k+1)}\big).
\end{aligned}
\right.
\end{equation}

Choosing \(\phi_1,\phi_2 \ge 1\) majorizes the linearized quadratic terms, guarantees
monotone descent of the augmented objective, and yields numerically stable,
single–pass updates well suited for unrolling within the LBO framework. Motivated by the LBO algorithm, we introduce learnable, stage–wise
parameters and replace the proximal operators by neural blocks. The resulting
\(K\)-stage unrolled scheme takes the form
\begin{equation}
\label{LADMMmatrix_neural_precision}
\left\{
\begin{aligned}
\Theta^{(k+1)}
&= \eta_{(\omega_1)_k}\!\Big(
\Theta^{(k)} - \alpha_k \circ \big(U^{(k)} + \gamma_k \circ (\Theta^{(k)} - Z^{(k)})\big)
\Big), \\[1ex]
Z^{(k+1)}
&= \xi_{(\omega_2)_k}\!\Big(
Z^{(k)} + \beta_k \circ \big(U^{(k)} + \gamma_k \circ (\Theta^{(k+1)} - Z^{(k)})\big)
\Big), \\[1ex]
U^{(k+1)}
&= U^{(k)} + \gamma_k \circ \big(\Theta^{(k+1)} - Z^{(k+1)}\big),
\end{aligned}
\right.
\end{equation}
for \(k=0,\dots,K-1\), where \(\circ\) denotes the Hadamard (elementwise) product.
The parameter set \(\{(\omega_1)_k,(\omega_2)_k,\alpha_k,\beta_k,\gamma_k\}_{k=0}^{K-1}\)
is learned from data: \(\alpha_k,\beta_k,\gamma_k\) may be scalars, diagonal
(preconditioning) matrices, or entrywise tensors, while
\(\eta_{(\omega_1)_k}\) and \(\xi_{(\omega_2)_k}\) are neural blocks acting as
learned proximal operators that map symmetric inputs to symmetric outputs (and, if
desired, incorporate PSD enforcement and diagonal handling).


Consequently, the classical LADMM updates are recovered by setting
\(\eta_{(\omega_1)_k}=\operatorname{prox}_{\,g/(\rho\phi_1)}\),
\(\xi_{(\omega_2)_k}=\operatorname{prox}_{\,(\lambda/(\rho\phi_2))\|\cdot\|_{1,\mathrm{off}}}\),
\(\alpha_k=\tfrac{1}{\phi_1}\), \(\beta_k=\tfrac{1}{\phi_2}\), and \(\gamma_k\equiv 1\).
Under learning, these components become data–driven: they adapt to the target
problem family and amortize computation across instances, while preserving the
algorithmic structure and feasibility constraints of the model–based solver.

To solve optimization problem (\ref{graphical_lasso}) using the proposed LBO algorithm, we need to define a corresponding loss function. The primal problem is $\min_{\Theta} f_2(\Theta) = -\log\det(\Theta) + \Tr(S \Theta) + \lambda \|\Theta\|_{1,\mathrm{off}}$, with the corresponding dual $\max_{\Xi} d_2(\Xi) = -\log\det(S - \Xi) - p$, subject to $|\xi_{ij}| \leq \lambda$ for $i \neq j$, where \(\Xi = (\xi_{ij})_{1\leq i,j\leq p}\) is the dual variable. In the experiments, we employ the duality gap
\[
f_2(\Theta) - d_2(\Xi)
\]
as the per-iteration loss to update the parameters.

\section{Theoretical Properties}  \label{sec4}

In this section, we first discuss the convergence of LADMM and the reparameterized LADMM algorithm (i.e., our proposed method). We then analyze the statistical optimization error bound between the high-dimensional matrix estimator obtained from the optimization procedure and the true high-dimensional matrix. Moreover, since the essence of the reparameterized LADMM lies in approximating the proximal operator, we further analyze its approximation properties.

\subsection{Convergence properties of algorithms}
In this section we present convergence guarantees for both the classical LADMM
and the proposed reparameterized (learned) LADMM. We first show that LADMM
converges to a Karush-Kuhn-Tucker (KKT) point of the unified convex formulation, and then establish
that the reparameterized scheme also converges under mild conditions. Moreover,
we argue that suitable choices of the learned parameters can yield a strictly
faster convergence rate, thereby demonstrating the potential superiority of the
reparameterized method.

The covariance and precision problems in \eqref{raw_high_dimension} and
\eqref{graphical_lasso} can be written in the split form
\begin{equation}
\label{unified_opt}
\min_{X,Y}\; F(X)+G(Y), \qquad \text{s.t. } X=Y,
\end{equation}
where \(X,Y\in\mathbb{R}^{p\times p}\) and \(F,G\) are proper, closed, convex functions.
Here \(\mathcal{I}_{\{\Sigma \succeq \epsilon I\}}(\Sigma)\) denotes the indicator that equals
\(0\) if \(\Sigma \succeq \epsilon I\) and \(+\infty\) otherwise. The specific choices of
\(F\) and \(G\) are as follows. For covariance estimation \eqref{raw_high_dimension} with \((X,Y)=(\Sigma,W)\),
\[
F(X)=\tfrac{1}{2}\|X-S\|_F^2+\mathcal{I}_{\{X \succeq \epsilon I\}}(X),
\qquad
G(Y)=\lambda\,\|Y\|_{1,\mathrm{off}} .
\]

For precision (graphical Lasso) \eqref{graphical_lasso} with \((X,Y)=(\Theta,Z)\),
\[
F(X)=\mathrm{Tr}(S X)-\log\det(X)+\mathcal{I}_{\{X \succeq \epsilon I\}}(X),
\qquad
G(Y)=\lambda\,\|Y\|_{1,\mathrm{off}} .
\]
The PSD indicator with floor \(\epsilon>0\) enforces feasibility and, in the precision case,
ensures that \(-\log\det(X)\) is well defined on the domain \(X\succeq \epsilon I\).
Let \(V\) denote the Lagrange multiplier and \(\rho>0\) the penalty parameter used in the
augmented Lagrangian and LADMM updates. At this point, for the optimization problem (\ref{unified_opt}), its LADMM update at iteration \(k\) is:
\begin{equation}\label{unified_LADMM}
\left\{
\begin{aligned}
X^{(k+1)} &= \operatorname{prox}_{\,F/(\rho\phi_1)}\!\left(
X^{(k)} - \frac{1}{\phi_1}\big(X^{(k)} - Y^{(k)} + V^{(k)}/\rho\big)
\right),\\[6pt]
Y^{(k+1)} &= \operatorname{prox}_{\,G/(\rho\phi_2)}\!\left(
Y^{(k)} - \frac{1}{\phi_2}\big(Y^{(k)} - X^{(k+1)} - V^{(k)}/\rho\big)
\right),\\[6pt]
V^{(k+1)} &= V^{(k)} + \rho\,\big(X^{(k+1)} - Y^{(k+1)}\big),
\end{aligned}
\right.
\end{equation}
where \(\phi_1,\phi_2>0\).
Introducing learnable, stage-wise parameters and replacing the proximal maps by
neural blocks yields
\begin{equation}
\label{unified_neural_LADMM}
\left\{
\begin{aligned}
X^{(k+1)} &= \eta_{(w_1)_k}\!\left(
X^{(k)} - \alpha_k \circ \big(V^{(k)} + \gamma_k\circ(X^{(k)} - Y^{(k)})\big)
\right), \\[1ex]
Y^{(k+1)} &= \xi_{(w_2)_k}\!\left(
Y^{(k)} + \beta_k \circ \big(V^{(k)} + \gamma_k\circ(X^{(k+1)} - Y^{(k)})\big)
\right), \\[1ex]
V^{(k+1)} &= V^{(k)} + \gamma_k\circ\big(X^{(k+1)} - Y^{(k+1)}\big),
\end{aligned}
\right.
\end{equation}
where \(\circ\) denotes the Hadamard product; the step/penalty schedules
\(\alpha_k,\beta_k,\gamma_k\) may be scalars, diagonal preconditioners, or
entrywise tensors; and \(\eta_{(w_1)_k},\xi_{(w_2)_k}\) are symmetric-preserving
neural blocks acting as learned proximal operators.

\begin{figure}
    \centering
    \includegraphics[width=\linewidth]{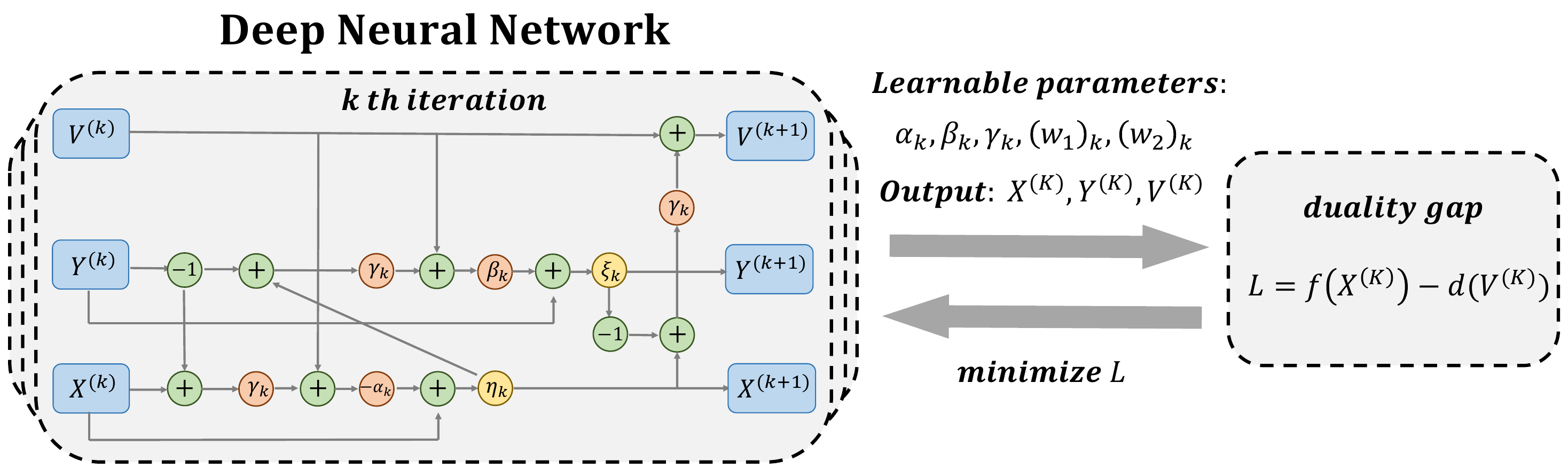}
    \caption{Overview of the proposed LBO algorithm framework. The left panel illustrates the forward process with total $K$ iterations of the algorithm, while the right panel shows the target loss function which is used to update learnable parameters. The operators $\eta_k,\xi_k$ are parameterized by $(w_1)_k,(w_2)_k$, respectively.}
    \label{fig:lbo_flowchart}
\end{figure}

We begin by presenting the convergence results of LADMM (\ref{unified_LADMM}) for the unified optimization problem (\ref{unified_opt}).
\begin{theorem}[Convergence of LADMM]\label{convergence_of_LADMM}
If $\phi_1,\phi_2>1$, then the sequence $\{X^{(k)},Y^{(k)},V^{(k)}\}$ generated by \eqref{unified_LADMM} converges to a KKT point of problem \eqref{unified_opt}.  
\end{theorem}

The requirement \(\phi_1,\phi_2>1\) enforces a proper majorization of the linearized quadratic terms, turning each subproblem into a strongly convex proximal step and
thereby yielding a 
{Fejér decrease} with respect to the KKT set via the Lyapunov potential
\[
\mathcal{E}_k
:= (\phi_1-1)\,\|X^{(k)}-X^*\|^2 \;+\; \|Y^{(k)}-Y^*\|^2
\;+\; \rho^{-2}\|V^{(k)}-V^*\|^2,
\]
for any KKT point \((X^*,Y^*,V^*)\). In particular, the proof
establishes
\[
\mathcal E_{k+1}\;\le\;\mathcal E_k
-(\phi_1-1)\|X^{(k+1)}-X^{(k)}\|^2
-(\phi_2-1)\|Y^{(k+1)}-Y^{(k)}\|^2
-\big\|\rho^{-1}(V^{(k+1)}-V^{(k)})+Y^{(k+1)}-Y^{(k)}\big\|^2,
\]
so that \(\|X^{(k+1)}-X^{(k)}\|\to0\), \(\|Y^{(k+1)}-Y^{(k)}\|\to0\),
\(\|V^{(k+1)}-V^{(k)}\|\to0\), and the whole sequence converges to a KKT point
of \eqref{unified_opt}.
One may {initialize} with \(\phi_i=1\)
and employ a
backtracking rule that increases \(\phi_i\) on-the-fly until a prescribed
majorization is met; this empirical variant often
works well. The theorem states a simple
sufficient (not necessary) condition that covers both covariance and precision
instances through the unified split form \eqref{unified_opt}. 



We then discuss the convergence properties of the re-parameterized LADMM. For analysis we consider the case where the learned blocks behave as weighted
proximal maps of \(F\) and \(G\):
\[
\eta_{(w_1)_k}=\operatorname{prox}_{\,(w_1)_k F}(M)
:=\arg\min_X\Big\{F(X)+\tfrac{1}{2}\big\|\tfrac{1}{\sqrt{(w_1)_k}}\circ(X-M)\big\|_F^2\Big\},
\]
\[
\xi_{(w_2)_k}=\operatorname{prox}_{\,(w_2)_k G}(M)
:=\arg\min_Y\Big\{G(Y)+\tfrac{1}{2}\big\|\tfrac{1}{\sqrt{(w_2)_k}}\circ(Y-M)\big\|_F^2\Big\}.
\]
Choose \((w_1)_k=\alpha_k\), \((w_2)_k=\beta_k\), and \(\gamma_k=1/\beta_k\). Then
\eqref{unified_neural_LADMM} can be rewritten as
\begin{equation}
\label{unified_neural_LADMM2}
\left\{
\begin{aligned}
X^{(k+1)} &= \arg\min_X\left\{
F(X)+\tfrac{1}{2}\left\|\tfrac{1}{\sqrt{\alpha_k}}\circ\Big(
X - X^{(k)} + \alpha_k\circ\big(V^{(k)}+\tfrac{1}{\beta_k}\circ(X^{(k)}-Y^{(k)})\big)
\Big)\right\|_F^2\right\}, \\[1ex]
Y^{(k+1)} &= \arg\min_Y\left\{
G(Y)+\tfrac{1}{2}\left\|\tfrac{1}{\sqrt{\beta_k}}\circ\Big(
Y - X^{(k+1)} - \beta_k\circ V^{(k)}
\Big)\right\|_F^2\right\}, \\[1ex]
V^{(k+1)} &= V^{(k)} + \tfrac{1}{\beta_k}\circ\big(X^{(k+1)} - Y^{(k+1)}\big).
\end{aligned}
\right.
\end{equation}

%
%

Before discussing the theoretical results of re-parameterized LADMM, we first introduce some basic notations. Let \(\mathcal{S}:=\{(\alpha,\beta):\, 0<\alpha<\beta,\ \alpha,\beta\in\mathbb{R}^{p\times p}\}\),
and denote \(\omega_k:=(X^{(k)},Y^{(k)},V^{(k)})^\top\).
Let \(\omega^*:=(X^*,Y^*,V^*)^\top\) be a KKT point of \eqref{unified_opt}, and
let \(\Omega^*\) be the set of all such points. Define the block-diagonal,
entrywise positive operator
\[
H_k(\omega):=\Big(\big(\tfrac{1}{\alpha_k}-\tfrac{1}{\beta_k}\big)\circ X,\;
\tfrac{1}{\beta_k}\circ Y,\; \beta_k\circ V\Big)^\top,\qquad
\phi(\omega):=(V,-V,Y-X)^\top.
\]
For any \(\omega\), set \(\|\omega\|_{H_k}^2:=\langle \omega, H_k(\omega)\rangle\).
We call \(H_k\succ 0\) (positive definite) if \(\|\omega\|_{H_k}^2>0\) for all
\(\omega\neq 0\). The induced operator norm is
\(\|H_k\|:=\sup_{\omega\neq0}\big|\|\omega\|_{H_k}^2\big|/\|\omega\|_F^2\). Then the convergence theorem can be established as follows.


\begin{theorem}[Convergence of re-parameterized LADMM]\label{convergence_of_neural_LADMM}

There exist parameters \((\alpha_k,\beta_k)\in\mathcal{S}\) such that the
sequence \(\{\omega_k\}\) generated by \eqref{unified_neural_LADMM2} converges to
a KKT point of \eqref{unified_opt}.
\end{theorem}

The learned scheme is analyzed under the \emph{variable-metric} interpretation:
the neural blocks act as weighted proximal operators and the parameters
\((\alpha_k,\beta_k)\in\mathcal{S}\) induce an iteration-dependent geometry.
This connects the reparameterized LADMM to preconditioned operator-splitting
methods. The proof process of this theorem shows that, provided the variable metric remains positive
definite (i.e., \(0<\alpha_k<\beta_k\)), and slowly varying, i.e., $\|H_{k+1}-H_k\|\le\mathcal{O}(1/(k+1)^2)$, the method inherits global convergence
to a KKT point of \eqref{unified_opt}. In practice, even more general
nonexpansive/averaged learned maps often work, but the proximal surrogate
assumption yields clean guarantees.

Building on the above discussion, we observe that the reparameterized scheme subsumes LADMM as a special case and enables data-adaptive preconditioning through \(\alpha_k,\beta_k,\gamma_k\) together with learned proximal surrogates \(\eta\) and \(\xi\). The metric \(H_k\) formalizes the iteration-dependent geometry induced by these weights. Theorems~\ref{monotonicity_of_neural_LADMM} and~\ref{convergence_rate_of_neural_LADMM} establish monotonicity and convergence rates for \eqref{unified_neural_LADMM2}. Moreover, Theorem~\ref{superiority_of_neural_LADMM} shows that, under mild regularity and suitable parameter schedules, the method achieves provably faster progress than baseline LADMM.



Define the distance to the solution set in the \(H_k\)-metric by
\(\operatorname{dist}_{H_k}(\omega,\Omega^*):=\inf_{\omega^*\in\Omega^*}\|\omega-\omega^*\|_{H_k}\). Then we can show the monotonicity property of \eqref{unified_neural_LADMM2}.

\begin{theorem}[Monotonicity of re-parameterized LADMM]\label{monotonicity_of_neural_LADMM}
There exist parameters \((\alpha_k,\beta_k)\in\mathcal{S}\) such that the
sequence \(\{\omega_k\}\) generated by \eqref{unified_neural_LADMM2} satisfies
that \(\operatorname{dist}_{H_k}(\omega_k,\Omega^*)\) is nonincreasing for all
sufficiently large \(k\).
\end{theorem}

The quantity \(\operatorname{dist}_{H_k}(\omega_k,\Omega^*)\) plays the role of a
Lyapunov function in a \emph{time-varying metric} \(H_k\). Monotonicity "for
sufficiently large \(k\)" reflects that once the parameter schedule stabilizes
(or varies slowly), the iteration contracts toward \(\Omega^*\) in the induced
metric. This clarifies why gentle parameter updates (or piecewise-constant
schedules) are numerically robust, and why aggressive, rapidly changing weights
can transiently break monotonicity.

For rate statements, define the update operator
\(\mathcal{T}(\alpha_k,\beta_k)(\omega_k)=\omega_{k+1}\) and assume
\((\alpha_k,\beta_k)\to(\alpha^*,\beta^*)\in\mathcal{S}\). We then have the convergence rate of \eqref{unified_neural_LADMM2}.


\begin{theorem}[Convergence rate of re-parameterized LADMM]\label{convergence_rate_of_neural_LADMM}
Let \(\{\omega_k\}\) be generated by \eqref{unified_neural_LADMM2}. Suppose that
for all sufficiently large \(k\),
\(\operatorname{dist}^2_{H^*}(\widetilde{\omega}_{k+1},\Omega^*) \le
(\kappa/16)\,\|\widetilde{\omega}_{k+1}-\omega_k\|_{H^*}^2\), where \(H^*\) is
defined by \((\alpha^*,\beta^*)\) and
\(\widetilde{\omega}_{k+1}:=\mathcal{T}(\alpha^*,\beta^*)(\omega_k)\).
Then there exist parameters \((\alpha_k,\beta_k)\in\mathcal{S}\) such that
\[
\operatorname{dist}^2_{H_{k+1}}(\omega_{k+1},\Omega^*)
\;\le\; \gamma\,\operatorname{dist}^2_{H_k}(\omega_k,\Omega^*)
\qquad\text{with }~0<\gamma<1,
\]
i.e., the convergence is (asymptotically) linear in the \(H_k\)-metric.
\end{theorem}

The assumed inequality with \(H^*\) is an \emph{error-bound}/\emph{quadratic-growth}
type condition near the solution set and is closely related to the
Kurdyka–Łojasiewicz framework \citep{attouch2010proximal}. Under this local regularity and stabilization of
\((\alpha_k,\beta_k)\to(\alpha^*,\beta^*)\), the mapping becomes a contraction in
the \(H^*\)-metric, giving (asymptotically) linear rate with factor
\(\gamma\in(0,1)\). In strongly convex instances one can often promote the
assumption globally; otherwise the result should be read as a local rate around
\(\Omega^*\).

\begin{theorem}[Superiority of re-parameterized LADMM]\label{superiority_of_neural_LADMM}
Assume \(\rho<1\) in \eqref{unified_LADMM} and that the learned blocks
\(\eta,\xi\) in \eqref{unified_neural_LADMM} are bijective. Then for any
\(\omega_k\notin \Omega^*\cup\Omega_0\) and any \(\omega^*\in\Omega^*\), where
\(\Omega_0\) is a measure-zero set, there exist
\((w_1)_k,(w_2)_k,\alpha_k,\beta_k,\gamma_k\) such that
\[
\big\|\widehat{\omega}_{k+1}-\omega^*\big\|_F
\;<\;
\big\|\widetilde{\omega}_{k+1}-\omega^*\big\|_F,
\]
where \(\widetilde{\omega}_{k+1}\) and \(\widehat{\omega}_{k+1}\) are generated by
\eqref{unified_LADMM} and \eqref{unified_neural_LADMM}, respectively.

\end{theorem}

This result is \emph{one-step} and \emph{existential}: outside a measure-zero
exceptional set, there exist learned parameters that reduce the next-iterate
error more than classical LADMM. It formalizes the intuition that adaptive
preconditioning and learned proximal surrogates can accelerate progress.
However, it does not assert uniform dominance for arbitrary learned parameters:
poorly trained blocks may stagnate or diverge. The bijectivity assumption
ensures a well-defined, reversible local mapping, while \(\rho<1\) aligns the
comparison with a stable LADMM baseline.

\subsection{Statistical–optimization error analysis}
In this part, we sequentially analyze the statistical optimization errors of the high-dimensional covariance and precision matrix estimators. Before doing so, we introduce some necessary notations and preliminaries. Let \(\widetilde{X}_1, \widetilde{X}_2, \dots, \widetilde{X}_n \in \mathbb{R}^p\) be i.i.d. sub-Gaussian random vectors with mean zero and true covariance matrix \(\Sigma^\star = (\sigma^\star_{ij})_{1\leq i,j \leq p} \in \mathbb{S}^p_+\), \(\Theta^\star = (\theta^\star_{ij})_{1\leq i,j \leq p} =  (\Sigma^\star)^{-1}\) denote the corresponding true precision matrix, and let \(S=\tfrac1n\sum_{i=1}^n \widetilde{X}_i \widetilde{X}_i^\top\) be the sample covariance. We estimate \(\Sigma^\star\) by solving the optimization problem (\ref{raw_high_dimension}), and let \(\widehat{\Sigma}^\star \in \argmin_{\Sigma \succeq \epsilon I} f_1(\Sigma)\). Considering that we are estimating a high-dimensional matrix, we define \(\gA := \{(i,j): i \neq j,; \sigma^\star_{ij} \neq 0\}\) to be the off-diagonal support with size \(|\gA|\).

We know that \(\Sigma^{(k)}\) is the iterate produced by a reparameterized optimization procedure after \(k\) steps. We then separate the total error \(\lVert \Sigma^{(k)} - \Sigma^\star \rVert_F\) into a statistical part (intrinsic) \(\lVert \widehat{\Sigma}^\star - \Sigma^\star \rVert_F\) and an optimization part (algorithmic) \(\lVert \Sigma^{(k)} - \widehat{\Sigma}^\star \rVert_F\). Define the optimization suboptimality as
\[
\varepsilon_{\mathrm{opt}}(k)
:= f_1(\Sigma^{(k)})-f_1(\widehat{\Sigma}^\star)\ge 0.
\]

%

Before presenting a theoretical upper bound for the total error \(\lVert \Sigma^{(k)} - \Sigma^\star \rVert_F\), we first introduce a basic assumption and some preliminary results.
\begin{assumption}[Coordinate sub-Gaussianity \citep{vershynin2018hdp}]\label{asmp:subg}
There exists $K<\infty$ such that each coordinate $\widetilde{X}_{ij}$ of $\widetilde{X}_i$ is sub-Gaussian with
$\psitwo$-Orlicz norm bounded by $K$, i.e.\ $\normpsitwo{X_i}\le K$ for all $j=1,\dots,p$.
\end{assumption}

\begin{theorem}[Entrywise concentration]\label{thm:entry}
Under Assumption~\ref{asmp:subg}, there exist constants $c_0,C_0>0$ (depending only on the sub-Gaussian parameter $K$ and absolute constants) such that
\[
\Pr\!\left(\,\|S-\Sigma^\star\|_\infty \;\le\; C_0\,\sqrt{\frac{\log p}{n}}\,\right)
\;\ge\; 1 - 2\,p^{-c_0}.
\]
\end{theorem}

\begin{theorem}[Total error decomposition and statistically optimal early stopping]\label{thm:total}
Under Assumption~\ref{asmp:subg}, choose \(\lambda \ge 2\|S-\Sigma^\star\|_\infty\).
Then with probability at least \(1-2p^{-c_0}\),
\begin{equation}\label{eq:total_bound}
\|\Sigma^{(k)}-\Sigma^\star\|_F
\;\le\;
\underbrace{\sqrt{2\,\varepsilon_{\mathrm{opt}}(k)}}_{\text{optimization}}
\;+\;
\underbrace{4\sqrt{|\gA|}\,\lambda + 4\,\|(\Sigma^\star)_{\gA^c}\|_1/\sqrt{|\gA|}}_{\text{statistical}}.
\end{equation}
If we stop when
\(
\varepsilon_{\mathrm{opt}}(k)\;\le\; C^2\, |\gA|\, \tfrac{\log p}{n}
\)
(for a universal \(C\) absorbing constants in \(\lambda\)), then
\(
\|\Sigma^{(k)}-\Sigma^\star\|_F
\lesssim \sqrt{|\gA|\,\tfrac{\log p}{n}} + \|(\Sigma^\star)_{\gA^c}\|_1/\sqrt{|\gA|},
\)
which is minimax-optimal up to constants for sparse models.
\end{theorem}

Similarly, for high-dimensional precision matrix estimation, we have an analogous bound on the total error. We define 
\[
\varepsilon_{\mathrm{opt}}^{\mathrm{GL}}(k) := f_2(\Theta^{(k)})-f_2(\widehat{\Theta}^\star)\ge 0.
\]

\begin{theorem}[Total error decomposition and statistically optimal early stopping]\label{thm:total-gl}
Assume sub-Gaussian sampling so that, with probability at least $1-2p^{-c_0}$,
$\|S-\Sigma^\star\|_\infty \le C_0\sqrt{\tfrac{\log p}{n}}$.
Choose $\lambda \ge 2\|S-\Sigma^\star\|_\infty$.
Assume there exist $0<\epsilon\le M<\infty$ such that $\epsilon I\preceq \Theta^\star\preceq MI$ and both
$\widehat\Theta^\star$ and $\Theta^{(k)}$ lie in $\{\,\Theta: \epsilon I\preceq \Theta\preceq MI\,\}$.
Let $\gB=\{(i,j):i\neq j,\ \theta^\star_{ij}\neq 0\}$ with size $|\gB|$.
Then,
\[
\|\Theta^{(k)}-\Theta^\star\|_F
\ \le\
\underbrace{\sqrt{2M^2\,\varepsilon_{\mathrm{opt}}^{\mathrm{GL}}(k)}}_{\mathrm{optimization}}
\;+\;
\underbrace{M^2\Big(3\lambda\sqrt{|\gB|} + 4\,\|(\Theta^\star)_{\gB^c}\|_1/\sqrt{|\gB|}\Big)}_{\mathrm{statistical}}.
\]
If we stop when $\varepsilon_{\mathrm{opt}}^{\mathrm{GL}}(T)\ \lesssim\ |\gB|\,\tfrac{\log p}{n}$ and take $\lambda \asymp \sqrt{\tfrac{\log p}{n}}$, then
$\|\Theta^{(k)}-\Theta^\star\|_F \lesssim M^2 \sqrt{|\gB|\,\tfrac{\log p}{n}}
+ M^2\,\|(\Theta^\star)_{\gB^c}\|_1/\sqrt{|\gB|}$.
\end{theorem}

\subsection{Approximation analysis of proximal operators}

In this section, we discuss the approximation property of parameterized neural networks $\eta,\zeta$.
We consider the case where $\eta,\zeta$ are parameterized by narrow one-hidden-layer ReLU networks. We first analyze the property of target proximal operators.
\begin{proposition}\label{prop:lipschitz_continuity}
Suppose that $F:\mathbb{R}^{p\times p}\rightarrow{\mathbb{R}^{p\times p}}$ is a proper, closed, convex function and $w\in{\mathbb{R}^{p\times p}}$ satisfies that $c_1 \le w_{i,j} \le c_2,\;\forall\, 1\le i,j\le p $, where $c_1, c_2$ are constants which may depend on $\alpha^*,\beta^*$. 
Then $\operatorname{prox}_{w,F}$ is Lipschitz continuous with
\[
\big\|\operatorname{prox}_{w,F}(M_1) - \operatorname{prox}_{w,F}(M_2)\big\|_F
\;\le\;
\sqrt{\frac{c_2}{c_1}}\;\|M_1 - M_2\|_F,
\qquad\forall\,M_1,M_2\in\mathbb{R}^{p\times p}.
\]
\end{proposition}

Based on the above result, we conclude that $\operatorname{prox}_{w,F}$ can be approximated by single layer ReLU network with
the approximation rate of $O(d^{-\frac{1}{p^2}})$, where $d$ is the number of neurons.
\begin{theorem}\label{thm:nn_approximate}
Under the same conditions of Proposition \ref{prop:lipschitz_continuity}, there exists a single layer network $\mathcal{NN}:\mathbb{R}^{p^2}\rightarrow{\mathbb{R}^{p^2}}$ that exhibits the form 
\begin{equation*}
\mathcal{NN}(x) = W_2^\top\sigma\left(W_1^\top x+b_1\right)+b_2,
\end{equation*}
such that 
\begin{equation}
\sup_{ \|X\|_{2}\le M}\|\mathcal{NN}(\operatorname{vec}(X))-\operatorname{vec}(\operatorname{prox}_{\omega,F}(X))\|_F\le C(p, M)\sqrt{\frac{c_2}{c_1}}\frac{\log(d)}{d^{1/{p^2}}}.
\end{equation}
Here $W_1\in\mathbb{R}^{d\times p^2},b_1\in\mathbb{R}^d,W_2\in\mathbb{R}^{p^2\times d},b_2\in\mathbb{R}^{p^2}$, $\sigma(\cdot)=(\cdot)_{+}$ is ReLU and the operator $\operatorname{vec}(\cdot)$  maps a matrix to a vector by stacking its columns.
\end{theorem}

\section{Simulation Studies} \label{sec5}
In this section, we apply the proposed LBO algorithm to high-dimensional matrix estimation problems with diverse structures and compare it against a wide array of existing methods to validate its superior performance. As we aim to compare computational times and the reliability of results across these methods, it is essential to first report the computational environment. All experiments were conducted on a workstation running Ubuntu 22.04.2 LTS (Linux kernel 5.15, x86\_64). The machine is equipped with dual Intel Haswell processors (4 logical CPUs), 19 GB RAM, and 2 GB swap space. We used a single NVIDIA GeForce RTX 3090 GPU (24 GB memory) with driver version 550.163.01 and CUDA 12.4. The software environment is based on Python 3.10.

\subsection{High-dimensional covariance matrix} \label{secHighcovariance}
Our objective is to estimate the covariance matrix $\Sigma^{\star} = (\sigma^\star_{ij})_{1\leq i,j \leq p}$ based on an observation matrix with \(n\) samples and \(p\) features sampled from the multivariate normal distribution $\mathcal{N}(0, \Sigma^{\star})$. In this section, we consider four distinct representative structures for the covariance matrix. The first structure considered is the Toeplitz covariance matrix, defined as 
\[
\sigma^\star_{ij} = \varrho^{|i-j|},
\]
for $i,j = 1, \dots, p$, where $|\varrho| < 1$. Here, $\varrho$ represents the correlation decay coefficient, with larger values corresponding to stronger long-range dependencies. We examine the following values of $\varrho$: 0.1, 0.3, 0.5, 0.7, and 0.9. Figure \ref{toeplitzvis} illustrates visualizations of the Toeplitz covariance matrices for $p=50$ under various values of $\rho$.

\begin{figure}[H]
\centering
\includegraphics[scale=0.53]{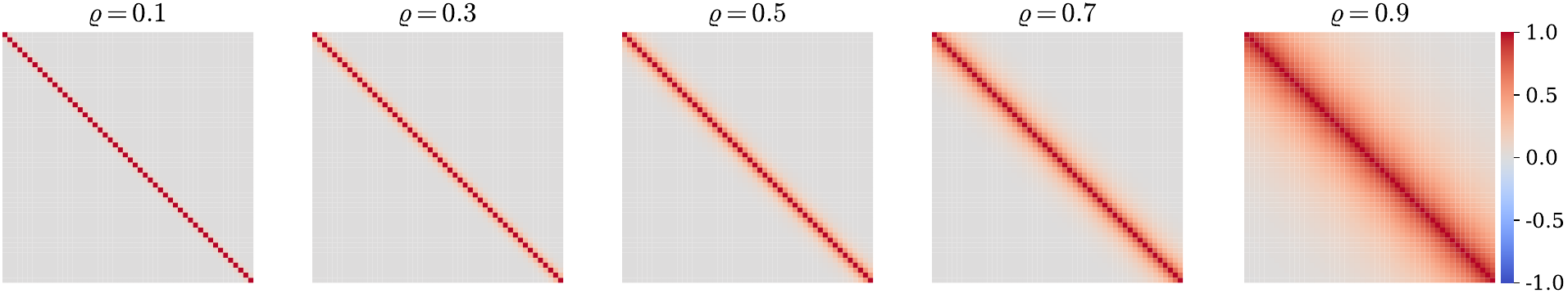}
\caption{Toeplitz covariance matrices under different correlation decay coefficients \(\varrho\).}
\label{toeplitzvis}
\end{figure}

The second structure is the factor model (low-rank plus diagonal noise), in which the factor loading matrix $B \in \mathbb{R}^{p \times m}$ is first generated entrywise independently from $\mathcal{N}(0, \sigma_B)$, and the covariance matrix is then defined as 
\[
d = BB^\top + \sigma_N I,
\]
where $m$ denotes the number of factors satisfying $1 \leq m \ll p$, which controls the dimensionality of the low-rank component; $\sqrt{\sigma_N}$ is the noise standard deviation, governing the strength of the diagonal noise; and $\sigma_B$ scales the magnitude of $B$, thereby determining the signal-to-noise ratio. In essence, larger values of $m$ and $\sigma_B$ yield more prominent principal components and a sharper spectrum for $\Sigma^\star$. In the experiments, we fix the factor strength $\sigma_B = 1$ and the noise variance $\sigma_N = 0.04$, while varying the number of factors $m$ to take values 3, 5, 7, 9, and 10. Figure \ref{factorvis} illustrates visualizations of the covariance matrices for dimension $p=50$ under different numbers of factors.

\begin{figure}[H]
\centering
\includegraphics[scale=0.53]{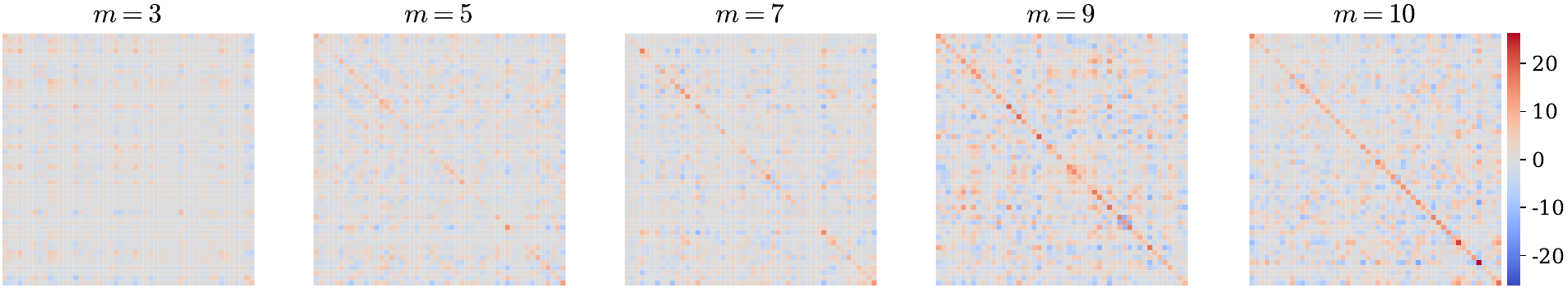}
\caption{Factor model under different numbers of low-rank components \(m\).}
\label{factorvis}
\end{figure}

The third structure is the sparse covariance matrix, characterized by random sparse off-diagonals and reinforced diagonal elements. The construction proceeds as follows: the diagonal elements are independently uniformly sampled, while the off-diagonal elements are set to non-zero with probability $q$ (with magnitudes uniformly sampled and randomly signed), followed by symmetrization:
\[
\sigma^\star_{ii} \sim \mathcal{U}(a, b),\ 
\sigma^\star_{ij} = \sigma^\star_{ji} = \begin{cases}
\pm \ \mathcal{U}(c,d), & \text{with probability } 1-q, \\ 0, & \text{with probability } q, 
\end{cases} i \neq j,
\]
where $q$ denotes the sparsity level, which controls the proportion of non-zero off-diagonals; $(a, b)$ is the diagonal interval; and $(c,d)$ is the non-zero intensity interval, governing the magnitude and upper bound of the non-zero correlations. In the experiments, we fix $(a,b) = (0.5, 2.0)$ and $(c,d) = (0.1, 0.8)$, while primarily varying the sparsity level $q \in \{0.1, 0.3, 0.5, 0.7, 0.9\}$. Figure \ref{sparsevis}  presents visualizations of the covariance matrices for dimension $p=50$ under these different sparsity levels.
\begin{figure}[H]
\centering
\includegraphics[scale=0.53]{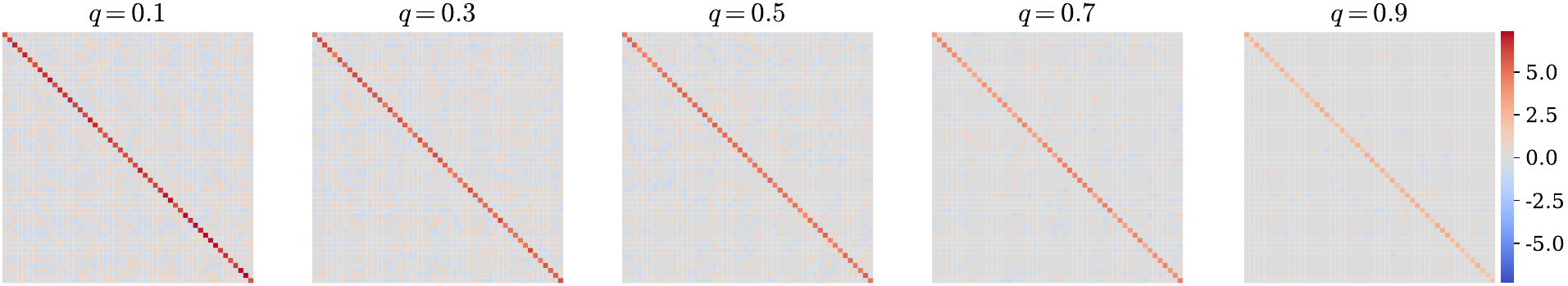}
\caption{Covariance matrices under different sparsity levels $q$.}
\label{sparsevis}
\end{figure}

The fourth structure is the block-diagonal covariance matrix, featuring strong correlations within clusters and sparse inter-block connections. The specific construction proceeds as follows: first, a block size is selected, and the $p$ variables are partitioned into blocks; within each block, an equicorrelation structure is adopted, given by the matrix with $1$s on the diagonal and $\varrho_w$ off the diagonal, i.e.,
\[
\begin{pmatrix}
1 & \varrho_w & \cdots & \varrho_w \\
\varrho_w & 1 & \cdots & \varrho_w \\
\vdots & \vdots & \ddots & \vdots \\
\varrho_w & \varrho_w & \cdots & 1
\end{pmatrix},
\]
where $\varrho_w$ denotes the within-block equicorrelation strength. For inter-block connections, with probability $\pi_b$, a weak edge of strength $\varrho_b$ is added to a pair of random positions across blocks, followed by symmetrization. Thus, $\varrho_b$ can be interpreted as the between-block weak correlation strength, while $\pi_b$ balances cluster independence against global connectivity. In the experiments, we fix $\varrho_w = 0.7$, $\varrho_b = 0.1$, and $\pi_b = 0.3$, while varying the block size to take values in ${10, 20, 25, 40, 50}$. Figure \ref{blockvis} visualizes the resulting block covariance matrices for dimension $p=90$ (solely to illustrate their structure).
\begin{figure}[H]
\centering
\includegraphics[scale=0.53]{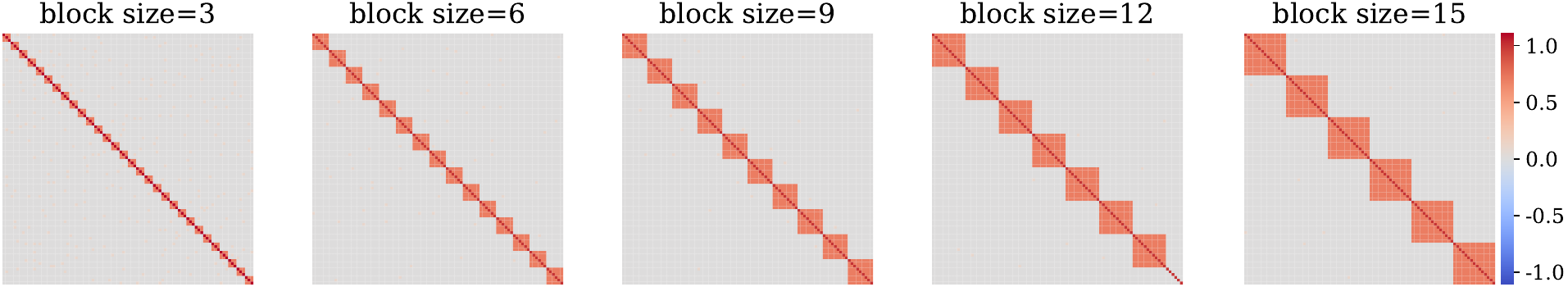}
\caption{Covariance matrices under different block size.}
\label{blockvis}
\end{figure}

In addition to the aforementioned ADMM and LADMM methods, we compare our proposed approach against five other representative optimization algorithms: the three-operator splitting algorithm (TOSA) \citep{davis2017three}, the proximal forward-backward splitting algorithm (PFBS) \citep{combettes2005signal}, the fast iterative shrinkage-thresholding algorithm (FISTA) \citep{beck2009fast}.  In fact, at the initial stage of our simulations, we also considered a semidefinite programming solver implemented in CVXPY \citep{diamond2016cvxpy} and the majorize–minimize algorithm (MMA) \citep{hunter2004tutorial}. However, neither method demonstrated any advantage over the proposed approach or the competing baselines (see \ref{CVXPYMMAsparse}). We therefore omit these two methods from the subsequent simulation studies.

Given that the comparative methods are computationally intensive in high dimensions and may encounter ill-conditioning issues, we fix the sample size at $n=500$ and vary the dimension $p \in \{1000, 2000, 3000, 4000\}$. We evaluate the estimation performance using the convergence time, Frobenius norm, nuclear norm, and duality gap as metrics. The experimental results of all methods are summarized in Tables~\ref{Toeplitztable}--\ref{Blocktable}. 
The visualization results of LBO on the four covariance structures are presented in Figures \ref{LBO_toeplitzexample}--\ref{LBO_blockexample}.

\begin{landscape}
\begin{table}[H]  
\footnotesize
\centering
\begin{threeparttable}
\caption{Experimental results of different methods on the Toeplitz covariance structure}\label{Toeplitztable}
\renewcommand{\arraystretch}{1.2}
\setlength{\tabcolsep}{1.45mm}{
\begin{tabular}{c|c|cccccc|cccccc|cccccc}
\hline
\multirow{2}{*}{\tabincell{c}{$\varrho$}} & \multirow{2}{*}{\tabincell{c}{Dimension}}& \multicolumn{6}{c|}{Time (s)} & \multicolumn{6}{c|}{Frobenius norm} & \multicolumn{6}{c}{Nuclear norm}\\
\cline{3-20}
& & \multicolumn{6}{c|}{LBO/ADMM/LADMM/TOSA/PFBS/FISTA} & \multicolumn{6}{c|}{LBO/ADMM/LADMM/TOSA/PFBS/FISTA} &\multicolumn{6}{c}{LBO/ADMM/LADMM/TOSA/PFBS/FISTA}\\
\hline
\multirow{4}{*}{\tabincell{c}{\(\varrho=0.1\)}} 
& \(p=\)1000 & \similar{$6.348^{1}$}  & $5.131^{1}$ & $1.757^{1}$ & $4.813^{1}$ & $7.757^{0}$ & $1.790^{1}$
& \farbetter{$2.089^{0}$} & $3.194^{1}$ & $5.515^{0}$ & $5.594^{0}$ & $5.617^{0}$ & $5.552^{0}$ 
& \farbetter{$6.575^{1}$}  & $10.00^{2}$ & $1.484^{2}$ & $1.507^{2}$ & $1.507^{2}$ & $1.493^{2}$ \\
& \(p=\)2000 & \farbetter{$1.162^{1}$} & $3.997^{2}$ & $7.069^{1}$ & $1.961^{2}$ & $3.013^{1}$ & $7.686^{1}$ 
& \farbetter{$9.074^{-2}$} & $4.517^{1}$ & $9.198^{0}$ & $9.237^{0}$ & $9.163^{0}$ & $9.202^{0}$ 
& \farbetter{$2.391^{0}$} & $2.000^{3}$ & $3.476^{2}$ & $3.486^{2}$ & $3.460^{2}$ & $3.474^{2}$ \\
& \(p=\)3000 & \farbetter{$1.529^{1}$} & $1.517^{3}$ & $1.750^{2}$ & $4.935^{2}$ & $7.599^{1}$ & $2.008^{2}$ 
& \farbetter{$1.879^{-1}$} & $5.532^{1}$ & $1.275^{1}$ & $1.271^{1}$ & $1.281^{1}$ & $1.271^{1}$ 
& \farbetter{$8.285^{0}$}  & $3.000^{3}$ & $5.868^{2}$ & $5.852^{2}$ & $5.904^{2}$ & $5.858^{2}$ \\
& \(p=\)4000 & \farbetter{$2.312^{1}$} & $4.029^{3}$ & $3.610^{2}$ & $9.581^{2}$ & $1.498^{2}$ & $3.892^{2}$ 
& \farbetter{$1.346^{0}$} & $6.388^{1}$ & $1.629^{1}$ & $1.630^{1}$ & $1.625^{1}$ & $1.620^{1}$ 
& \farbetter{$8.468^{1}$} & $4.000^{3}$ & $8.637^{2}$ & $8.633^{2}$ & $8.625^{2}$ & $8.592^{2}$ \\ 
\hline
\multirow{4}{*}{\tabincell{c}{\(\varrho=0.3\)}} 
& \(p=\)1000 & \similar{$5.812^{1}$} & $4.954^{1}$ & $1.650^{1}$ & $4.619^{1}$ & $7.281^{0}$ & $1.763^{1}$ 
& \farbetter{$2.183^{0}$} & $3.461^{1}$ & $7.292^{0}$ & $7.346^{0}$ & $7.409^{0}$ & $7.332^{0}$ 
& \farbetter{$6.876^{1}$}  & $10.00^{2}$ & $1.852^{2}$ & $1.863^{2}$ & $1.876^{2}$ & $1.867^{2}$ \\
& \(p=\)2000 & \similar{$5.527^{1}$} & $4.017^{2}$ & $6.739^{1}$ & $1.946^{2}$ & $3.007^{1}$ & $7.558^{1}$ 
& \farbetter{$3.052^{-1}$} & $4.894^{1}$ & $1.156^{1}$ & $1.150^{1}$ & $1.150^{1}$ & $1.146^{1}$ 
& \farbetter{$1.272^{1}$} & $2.000^{3}$ & $4.186^{2}$ & $4.165^{2}$ & $4.175^{2}$ & $4.151^{2}$ \\
& \(p=\)3000 & \farbetter{$5.174^{1}$} & $1.538^{3}$ & $1.812^{2}$ & $4.890^{2}$ & $7.493^{1}$ & $2.015^{2}$ 
& \farbetter{$5.462^{-1}$} & $5.994^{1}$ & $1.528^{1}$ & $1.519^{1}$ & $1.524^{1}$ & $1.522^{1}$ 
& \farbetter{$2.799^{1}$}  & $3.000^{3}$ & $6.830^{2}$ & $6.794^{2}$ & $6.816^{2}$ & $6.802^{2}$ \\
& \(p=\)4000 & \farbetter{$6.721^{1}$} & $4.056^{3}$ & $3.567^{2}$ & $9.682^{2}$ & $1.501^{2}$ & $3.961^{2}$ 
& \farbetter{$1.265^{0}$} & $6.922^{1}$ & $1.887^{1}$ & $1.896^{1}$ & $1.895^{1}$ & $1.892^{1}$
& \farbetter{$7.659^{1}$}  & $4.000^{3}$ & $9.780^{2}$ & $9.839^{2}$ & $9.831^{2}$ & $9.817^{2}$ \\
\hline
\multirow{4}{*}{\tabincell{c}{\(\varrho=0.5\)}} 
& \(p=\)1000 & \similar{$5.054^{1}$} & $4.698^{1}$ & $1.576^{1}$ & $4.424^{1}$ & $7.309^{0}$ & $1.841^{1}$ 
& \farbetter{$2.591^{0}$} & $4.081^{1}$ & $9.508^{0}$ & $9.631^{0}$ & $9.753^{0}$ & $9.620^{0}$ 
& \farbetter{$8.173^{1}$}  & $10.00^{2}$ & $2.129^{2}$ & $2.148^{2}$ & $2.167^{2}$ & $2.148^{2}$ \\
& \(p=\)2000 & \similar{$5.882^{1}$} & $3.946^{2}$ & $6.903^{1}$ & $1.970^{2}$ & $3.033^{1}$ & $7.814^{1}$ 
& \farbetter{$5.933^{-1}$} & $5.773^{1}$ & $1.448^{1}$ & $1.441^{1}$ & $1.442^{1}$ & $1.441^{1}$ 
& \farbetter{$1.875^{1}$}  & $2.000^{3}$ & $4.700^{2}$ & $4.693^{2}$ & $4.685^{2}$ & $4.668^{2}$ \\
& \(p=\)3000 & \similar{$1.175^{2}$} & $1.494^{3}$ & $1.747^{2}$ & $4.897^{2}$ & $7.439^{1}$ & $1.988^{2}$ 
& \farbetter{$9.657^{-1}$} & $7.070^{1}$ & $1.874^{1}$ & $1.865^{1}$ & $1.865^{1}$ & $1.865^{1}$ 
& \farbetter{$4.285^{1}$}  & $3.000^{3}$ & $7.602^{2}$ & $7.563^{2}$ & $7.549^{2}$ & $7.557^{2}$ \\
& \(p=\)4000 & \similar{$2.152^{2}$} & $4.040^{3}$ & $3.626^{2}$ & $9.694^{2}$ & $1.520^{2}$ & $3.928^{2}$ 
& \farbetter{$1.816^{0}$} & $8.164^{1}$ & $2.269^{1}$ & $2.272^{1}$ & $2.259^{1}$ & $2.271^{1}$ 
& \farbetter{$1.070^{2}$}  & $4.000^{3}$ & $1.080^{3}$ & $1.079^{3}$ & $1.078^{3}$ & $1.081^{3}$ \\
\hline
\multirow{4}{*}{\tabincell{c}{\(\varrho=0.7\)}} 
& \(p=\)1000 & \similar{$5.919^{1}$} & $5.142^{1}$ & $1.536^{1}$ & $4.573^{1}$ & $6.969^{0}$ & $1.743^{1}$ 
& \farbetter{$6.057^{-1}$} & $5.401^{1}$ & $1.352^{1}$ & $1.365^{1}$ & $1.305^{1}$ & $1.310^{1}$ 
& \farbetter{$1.752^{1}$}  & $10.00^{2}$ & $2.391^{2}$ & $2.428^{2}$ & $2.367^{2}$ & $2.381^{2}$ \\
& \(p=\)2000 & \similar{$5.484^{1}$}& $4.032^{2}$ & $6.977^{1}$ & $1.945^{2}$ & $2.989^{1}$ & $7.839^{1}$ 
& \farbetter{$8.071^{-1}$} & $7.642^{1}$ & $1.974^{1}$ & $1.948^{1}$ & $1.942^{1}$ & $1.958^{1}$ 
& \farbetter{$2.160^{1}$}  & $2.000^{3}$ & $5.187^{2}$ & $5.110^{2}$ & $5.121^{2}$ & $5.160^{2}$ \\
& \(p=\)3000 & \similar{$1.145^{2}$} & $1.516^{3}$ & $1.749^{2}$ & $4.829^{2}$ & $7.804^{1}$ & $2.002^{2}$ 
& \farbetter{$1.846^{0}$} & $9.360^{1}$ & $2.459^{1}$ & $2.464^{1}$ & $2.453^{1}$ & $2.514^{1}$ 
& \farbetter{$9.119^{1}$}  & $3.000^{3}$ & $8.183^{2}$ & $8.183^{2}$ & $8.206^{2}$ & $8.260^{2}$ \\
& \(p=\)4000 & \similar{$2.165^{2}$} & $4.157^{3}$ & $3.673^{2}$ & $1.008^{3}$ & $1.632^{2}$ & $3.905^{2}$ 
& \farbetter{$2.578^{0}$} & $1.081^{2}$ & $2.968^{1}$ & $2.943^{1}$ & $2.970^{1}$ & $2.932^{1}$ 
& \farbetter{$1.474^{2}$}  & $4.000^{3}$ & $1.157^{3}$ & $1.154^{3}$ & $1.158^{3}$ & $1.150^{3}$ \\
\hline
\multirow{4}{*}{\tabincell{c}{\(\varrho=0.9\)}} 
& \(p=\)1000 & \similar{$5.738^{1}$} & $4.612^{1}$ & $1.606^{1}$ & $4.624^{1}$ & $7.659^{0}$ & $1.807^{1}$ 
& \farbetter{$1.165^{0}$} & $9.553^{1}$ & $2.491^{1}$ & $2.419^{1}$ & $2.372^{1}$ & $2.454^{1}$ 
& \farbetter{$2.190^{1}$}  & $9.889^{2}$ & $2.729^{2}$ & $2.697^{2}$ & $2.656^{2}$ & $2.704^{2}$ \\
& \(p=\)2000 & \similar{$6.656^{1}$} & $3.796^{2}$ & $6.354^{1}$ & $1.918^{2}$ & $3.043^{1}$ & $7.621^{1}$ 
& \farbetter{$1.641^{0}$} & $1.377^{2}$ & $3.542^{1}$ & $3.484^{1}$ & $3.525^{1}$ & $3.573^{1}$ 
& \farbetter{$3.438^{1}$}  & $1.998^{3}$ & $5.625^{2}$ & $5.675^{2}$ & $5.635^{2}$ & $5.637^{2}$ \\
& \(p=\)3000 & \similar{$9.785^{1}$} & $1.521^{3}$ & $1.767^{2}$ & $5.143^{2}$ & $8.116^{1}$ & $1.998^{2}$ 
& \farbetter{$2.089^{0}$} & $1.689^{2}$ & $4.372^{1}$ & $4.335^{1}$ & $4.394^{1}$ & $4.356^{1}$ 
& \farbetter{$5.919^{1}$}  & $3.000^{3}$ & $8.733^{2}$ & $8.778^{2}$ & $8.822^{2}$ & $8.789^{2}$ \\
& \(p=\)4000 & \similar{$1.856^{2}$} & $4.159^{3}$ & $3.498^{2}$ & $9.877^{2}$ & $1.570^{2}$ & $3.849^{2}$ 
& \farbetter{$2.851^{0}$} & $1.951^{2}$ & $5.088^{1}$ & $5.005^{1}$ & $5.005^{1}$ & $5.037^{1}$ 
& \farbetter{$1.177^{2}$}  & $4.000^{3}$ & $1.210^{3}$ & $1.209^{3}$ & $1.209^{3}$ & $1.204^{3}$ \\
\hline
\end{tabular}
}
\begin{tablenotes}
\footnotesize
\item[*] Scientific notation is used for the numbers reported in the tables. For example, $1.177^{2}$ denotes $1.177 \times 10^{2}$. The same notation applies throughout the following tables.
\item[*] We use a color code to indicate relative performance: 
\legendbox{green!18} far better, 
\legendbox{gray!12} comparable, 
and \legendbox{red!15} worse (relative to the competing methods). 
The same color scheme is used in the following tables.
\end{tablenotes}
\end{threeparttable}
\end{table}
\end{landscape}

\begin{landscape}
\begin{table}[H]
\footnotesize
\centering
\caption{Experimental results of different methods on the Factor covariance structure}\label{Factortable}
\renewcommand{\arraystretch}{1.2}
\setlength{\tabcolsep}{1.45mm}{
\begin{tabular}{c|c|cccccc|cccccc|cccccc}
\hline
\multirow{2}{*}{\tabincell{c}{Factor \\ number}} & \multirow{2}{*}{\tabincell{c}{Dimension}}& \multicolumn{6}{c|}{Time (s)} & \multicolumn{6}{c|}{Frobenius norm} & \multicolumn{6}{c}{Nuclear norm}\\
\cline{3-20}
& & \multicolumn{6}{c|}{LBO/ADMM/LADMM/TOSA/PFBS/FISTA} & \multicolumn{6}{c|}{LBO/ADMM/LADMM/TOSA/PFBS/FISTA} &\multicolumn{6}{c}{LBO/ADMM/LADMM/TOSA/PFBS/FISTA}\\
\hline
\multirow{4}{*}{\tabincell{c}{\(m=3\)}} 
& \(p=\)1000 & \farbetter{$6.605^{1}$} & $9.947^{1}$ & $7.239^{1}$ & $1.780^{2}$ & $8.210^{1}$ & $1.221^{3}$ 
& \farbetter{$3.218^{1}$} & $1.465^{2}$ & $1.465^{2}$ & $1.469^{2}$ & $1.463^{2}$ & $1.465^{2}$ 
& \similar{$8.149^{2}$} & $6.546^{2}$ & $6.546^{2}$ & $7.388^{2}$ & $6.406^{2}$ & $6.546^{2}$ \\
& \(p=\)2000 & \farbetter{$4.229^{1}$} & $6.151^{2}$ & $6.249^{2}$ & $7.954^{2}$ & $2.664^{2}$ & $5.530^{3}$ 
& \farbetter{$4.357^{1}$} & $4.465^{2}$ & $4.465^{2}$ & $4.471^{2}$ & $4.463^{2}$ & $4.465^{2}$ 
& \farbetter{$1.483^{2}$} & $1.586^{3}$ & $1.586^{3}$ & $1.812^{3}$ & $1.543^{3}$ & $1.586^{3}$ \\
& \(p=\)3000 & \farbetter{$8.342^{1}$} & $1.542^{3}$ & $1.697^{3}$ & $1.705^{3}$ & $3.798^{2}$ & $1.171^{4}$ 
& \farbetter{$6.528^{1}$} & $4.385^{2}$ & $4.385^{2}$ & $4.400^{2}$ & $4.380^{2}$ & $4.385^{2}$ 
& \farbetter{$2.377^{2}$} & $2.211^{3}$ & $2.211^{3}$ & $2.591^{3}$ & $2.135^{3}$ & $2.211^{3}$ \\
& \(p=\)4000 & \farbetter{$2.877^{1}$} & $1.040^{4}$ & $1.437^{4}$ & $2.114^{4}$ & $5.231^{3}$ & $1.690^{5}$ 
& \farbetter{$8.746^{1}$} & $5.919^{2}$ & $5.919^{2}$ & $5.940^{2}$ & $5.912^{2}$ & $5.919^{2}$ 
& \farbetter{$3.018^{2}$} & $3.148^{3}$ & $3.148^{3}$ & $3.708^{3}$ & $3.032^{3}$ & $3.148^{3}$ \\
\hline
\multirow{4}{*}{\tabincell{c}{\(m=5\)}} 
& \(p=\)1000 & \farbetter{$1.189^{1}$} & $2.089^{3}$ & $5.206^{2}$ & $3.077^{3}$ & $1.079^{3}$ & $1.183^{4}$ 
& \farbetter{$3.193^{1}$} & $1.974^{2}$ & $1.974^{2}$ & $1.977^{2}$ & $1.973^{2}$ & $1.974^{2}$ 
& \farbetter{$6.097^{2}$} & $1.068^{3}$ & $1.068^{3}$ & $1.235^{3}$ & $1.057^{3}$ & $1.068^{3}$ \\
& \(p=\)2000 & \farbetter{$4.068^{1}$} & $5.176^{3}$ & $3.495^{3}$ & $6.752^{3}$ & $2.179^{3}$ & $3.032^{4}$ 
& \farbetter{$5.683^{1}$} & $6.391^{2}$ & $6.391^{2}$ & $6.396^{2}$ & $6.390^{2}$ & $6.391^{2}$ 
& \farbetter{$4.793^{2}$} & $3.018^{3}$ & $3.018^{3}$ & $3.504^{3}$ & $2.981^{3}$ & $3.018^{3}$ \\
& \(p=\)3000 & \farbetter{$8.566^{1}$} & $1.709^{4}$ & $6.982^{3}$ & $1.425^{4}$ & $5.011^{3}$ & $7.197^{4}$ 
& \farbetter{$8.451^{1}$} & $6.360^{2}$ & $6.360^{2}$ & $6.372^{2}$ & $6.358^{2}$ & $6.360^{2}$ 
& \farbetter{$3.811^{2}$} & $4.288^{3}$ & $4.288^{3}$ & $5.191^{3}$ & $4.218^{3}$ & $4.288^{3}$ \\
& \(p=\)4000 & \farbetter{$3.368^{1}$} & $2.104^{4}$ & $1.839^{4}$ & $4.719^{4}$ & $1.016^{4}$ & $4.030^{4}$ 
& \farbetter{$1.129^{2}$} & $1.211^{3}$ & $1.211^{3}$ & $1.212^{3}$ & $1.210^{3}$ & $1.211^{3}$ 
& \farbetter{$5.058^{2}$} & $6.669^{3}$ & $6.669^{3}$ & $8.038^{3}$ & $6.557^{3}$ & $6.669^{3}$ \\
\hline
\multirow{4}{*}{\tabincell{c}{\(m=7\)}} 
& \(p=\)1000 & \farbetter{$1.747^{1}$} & $1.782^{2}$ & $1.010^{2}$ & $2.332^{2}$ & $8.350^{1}$ & $1.097^{3}$ 
& \farbetter{$3.345^{1}$} & $3.691^{2}$ & $3.691^{2}$ & $3.693^{2}$ & $3.691^{2}$ & $3.691^{2}$ 
& \farbetter{$1.777^{2}$} & $1.589^{3}$ & $1.589^{3}$ & $1.784^{3}$ & $1.581^{3}$ & $1.589^{3}$ \\
& \(p=\)2000 & \farbetter{$4.095^{1}$} & $7.901^{2}$ & $4.753^{2}$ & $1.099^{3}$ & $3.798^{2}$ & $4.760^{3}$ 
& \farbetter{$6.671^{1}$} & $6.626^{2}$ & $6.626^{2}$ & $6.631^{2}$ & $6.625^{2}$ & $6.626^{2}$ 
& \farbetter{$3.598^{2}$} & $3.646^{3}$ & $3.646^{3}$ & $4.239^{3}$ & $3.622^{3}$ & $3.646^{3}$ \\
& \(p=\)3000 & \farbetter{$8.113^{1}$} & $1.772^{3}$ & $1.215^{3}$ & $2.461^{3}$ & $8.408^{2}$ & $1.042^{4}$ 
& \farbetter{$1.001^{2}$} & $1.085^{3}$ & $1.085^{3}$ & $1.086^{3}$ & $1.085^{3}$ & $1.085^{3}$ 
& \farbetter{$5.365^{2}$} & $5.933^{3}$ & $5.933^{3}$ & $7.041^{3}$ & $5.886^{3}$ & $5.933^{3}$ \\
& \(p=\)4000 & \farbetter{$2.731^{1}$} & $3.320^{3}$ & $1.965^{3}$ & $2.125^{3}$ & $7.795^{2}$ & $9.776^{3}$ 
& \farbetter{$1.332^{2}$} & $1.282^{3}$ & $1.282^{3}$ & $1.283^{3}$ & $1.282^{3}$ & $1.282^{3}$ 
& \farbetter{$6.846^{2}$} & $8.324^{3}$ & $8.324^{3}$ & $1.005^{4}$ & $8.255^{3}$ & $8.324^{3}$ \\
\hline
\multirow{4}{*}{\tabincell{c}{\(m=9\)}} 
& \(p=\)1000 & \farbetter{$2.339^{1}$} & $1.665^{2}$ & $9.556^{1}$ & $2.466^{2}$ & $8.351^{1}$ & $8.821^{2}$ 
& \farbetter{$3.769^{1}$} & $4.074^{2}$ & $4.074^{2}$ & $4.076^{2}$ & $4.074^{2}$ & $4.074^{2}$ 
& \farbetter{$2.251^{2}$} & $1.873^{3}$ & $1.873^{3}$ & $2.078^{3}$ & $1.867^{3}$ & $1.873^{3}$ \\
& \(p=\)2000 & \farbetter{$4.018^{1}$} & $5.576^{2}$ & $3.443^{2}$ & $7.893^{2}$ & $2.796^{2}$ & $3.591^{3}$ 
& \farbetter{$7.550^{1}$} & $8.869^{2}$ & $8.869^{2}$ & $8.873^{2}$ & $8.868^{2}$ & $8.869^{2}$ 
& \farbetter{$5.526^{2}$} & $4.509^{3}$ & $4.509^{3}$ & $5.146^{3}$ & $4.490^{3}$ & $4.509^{3}$ \\
& \(p=\)3000 & \farbetter{$8.941^{1}$} & $1.404^{3}$ & $8.694^{2}$ & $1.962^{3}$ & $7.131^{2}$ & $8.492^{3}$ 
& \farbetter{$1.133^{2}$} & $1.291^{3}$ & $1.291^{3}$ & $1.292^{3}$ & $1.291^{3}$ & $1.291^{3}$ 
& \farbetter{$6.700^{2}$} & $7.354^{3}$ & $7.354^{3}$ & $8.556^{3}$ & $7.320^{3}$ & $7.354^{3}$ \\
& \(p=\)4000 & \farbetter{$3.551^{1}$} & $1.015^{4}$ & $7.950^{3}$ & $1.875^{4}$ & $9.517^{3}$ & $8.776^{4}$ 
& \farbetter{$1.512^{2}$} & $1.667^{3}$ & $1.667^{3}$ & $1.668^{3}$ & $1.667^{3}$ & $1.667^{3}$ 
& \farbetter{$8.821^{2}$} & $1.016^{4}$ & $1.016^{4}$ & $1.204^{4}$ & $1.011^{4}$ & $1.016^{4}$ \\
\hline
\multirow{4}{*}{\tabincell{c}{\(m=10\)}} 
& \(p=\)1000 & \farbetter{$1.200^{1}$} & $1.518^{3}$ & $6.111^{2}$ & $2.019^{3}$ & $7.805^{2}$ & $8.644^{3}$ 
& \farbetter{$5.424^{1}$} & $3.955^{2}$ & $3.955^{2}$ & $3.957^{2}$ & $3.955^{2}$ & $3.955^{2}$ 
& \farbetter{$1.386^{3}$} & $1.942^{3}$ & $1.942^{3}$ & $2.150^{3}$ & $1.936^{3}$ & $1.942^{3}$ \\
& \(p=\)2000 & \farbetter{$4.085^{1}$} & $4.138^{3}$ & $2.237^{3}$ & $5.862^{3}$ & $1.989^{3}$ & $2.580^{4}$ 
& \farbetter{$8.209^{1}$} & $9.161^{2}$ & $9.161^{2}$ & $9.165^{2}$ & $9.161^{2}$ & $9.161^{2}$ 
& \farbetter{$1.347^{3}$} & $4.658^{3}$ & $4.658^{3}$ & $5.305^{3}$ & $4.641^{3}$ & $4.658^{3}$ \\
& \(p=\)3000 & \farbetter{$8.636^{1}$} & $7.705^{3}$ & $4.750^{3}$ & $1.093^{4}$ & $2.926^{3}$ & $7.969^{4}$ 
& \farbetter{$1.196^{2}$} & $1.385^{3}$ & $1.385^{3}$ & $1.386^{3}$ & $1.385^{3}$ & $1.385^{3}$ 
& \farbetter{$7.412^{2}$} & $7.687^{3}$ & $7.687^{3}$ & $8.923^{3}$ & $7.656^{3}$ & $7.687^{3}$ \\
& \(p=\)4000 & \farbetter{$2.647^{1}$} & $2.740^{4}$ & $1.582^{4}$ & $3.222^{4}$ & $1.256^{4}$ & $1.423^{5}$ 
& \farbetter{$1.595^{2}$} & $1.828^{3}$ & $1.828^{3}$ & $1.828^{3}$ & $1.827^{3}$ & $1.828^{3}$ 
& \farbetter{$9.894^{2}$} & $1.101^{4}$ & $1.101^{4}$ & $1.294^{4}$ & $1.096^{4}$ & $1.101^{4}$ \\
\hline
\end{tabular}
}
\end{table}
\end{landscape}

\begin{landscape}
\begin{table}[H]  
\footnotesize
\centering
\caption{Experimental results of different methods on the sparse covariance structure}\label{Sparsetable}
\renewcommand{\arraystretch}{1.2}
\setlength{\tabcolsep}{1.45mm}{
\begin{tabular}{c|c|cccccc|cccccc|cccccc}
\hline
\multirow{2}{*}{\tabincell{c}{Sparsity \\ level}} & \multirow{2}{*}{\tabincell{c}{Dimension}}& \multicolumn{6}{c|}{Time (s)} & \multicolumn{6}{c|}{Frobenius norm} & \multicolumn{6}{c}{Nuclear norm}\\
\cline{3-20}
& & \multicolumn{6}{c|}{LBO/ADMM/LADMM/TOSA/PFBS/FISTA} & \multicolumn{6}{c|}{LBO/ADMM/LADMM/TOSA/PFBS/FISTA} &\multicolumn{6}{c}{LBO/ADMM/LADMM/TOSA/PFBS/FISTA}\\
\hline
\multirow{4}{*}{\tabincell{c}{\(q=0.1\)}} 
& \(p=\)1000 & \farbetter{$1.791^{1}$} & $1.696^{2}$ & $1.110^{2}$ & $2.662^{2}$ & $4.011^{1}$ & $2.374^{2}$ 
& \farbetter{$1.704^{2}$} & $1.252^{3}$ & $1.252^{3}$ & $1.251^{3}$ & $1.252^{3}$ & $1.252^{3}$ 
& \farbetter{$5.369^{3}$} & $3.002^{4}$ & $3.002^{4}$ & $3.000^{4}$ & $3.002^{4}$ & $3.002^{4}$ \\
& \(p=\)2000 & \farbetter{$4.180^{1}$} & $6.248^{2}$ & $3.582^{2}$ & $9.583^{2}$ & $1.582^{2}$ & $1.082^{3}$ 
& \farbetter{$1.001^{3}$} & $3.582^{3}$ & $3.582^{3}$ & $3.577^{3}$ & $3.582^{3}$ & $3.582^{3}$ 
& \farbetter{$4.470^{4}$} & $1.138^{5}$ & $1.138^{5}$ & $1.133^{5}$ & $1.138^{5}$ & $1.138^{5}$ \\
& \(p=\)3000 & \farbetter{$1.709^{1}$} & $1.452^{3}$ & $8.957^{2}$ & $2.432^{3}$ & $3.849^{2}$ & $3.093^{3}$ 
& \farbetter{$1.259^{3}$} & $6.649^{3}$ & $6.649^{3}$ & $6.639^{3}$ & $6.649^{3}$ & $6.649^{3}$ 
& \farbetter{$6.878^{4}$} & $2.374^{5}$ & $2.374^{5}$ & $2.360^{5}$ & $2.374^{5}$ & $2.374^{5}$ \\
& \(p=\)4000 & \farbetter{$2.527^{1}$} & $9.748^{3}$ & $5.041^{3}$ & $1.859^{4}$ & $3.254^{3}$ & $2.711^{4}$ 
& \farbetter{$7.315^{3}$} & $1.024^{4}$ & $1.024^{4}$ & $1.023^{4}$ & $1.024^{4}$ & $1.024^{4}$ 
& $4.623^{5}$ & $3.869^{5}$ & $3.869^{5}$ & $3.843^{5}$ & $3.869^{5}$ & $3.869^{5}$ \\
\hline
\multirow{4}{*}{\tabincell{c}{\(q=0.3\)}} 
& \(p=\)1000 & \farbetter{$1.649^{1}$} & $7.722^{2}$ & $3.542^{2}$ & $1.063^{3}$ & $3.105^{2}$ & $1.134^{3}$ 
& \farbetter{$6.421^{0}$} & $1.089^{3}$ & $1.089^{3}$ & $1.088^{3}$ & $1.089^{3}$ & $1.089^{3}$ 
& \farbetter{$1.702^{2}$} & $2.616^{4}$ & $2.616^{4}$ & $2.614^{4}$ & $2.616^{4}$ & $2.616^{4}$ \\
& \(p=\)2000 & \farbetter{$9.777^{0}$} & $2.427^{3}$ & $1.220^{3}$ & $4.435^{3}$ & $6.604^{2}$ & $4.640^{3}$ 
& \farbetter{$3.174^{1}$} & $3.180^{3}$ & $3.180^{3}$ & $3.175^{3}$ & $3.180^{3}$ & $3.180^{3}$ 
& \farbetter{$7.241^{2}$} & $1.011^{5}$ & $1.011^{5}$ & $1.007^{5}$ & $1.011^{5}$ & $1.011^{5}$ \\
& \(p=\)3000 & \farbetter{$1.019^{1}$} & $6.433^{3}$ & $4.159^{3}$ & $1.113^{4}$ & $2.108^{3}$ & $1.372^{4}$ 
& \farbetter{$2.635^{1}$} & $5.800^{3}$ & $5.800^{3}$ & $5.790^{3}$ & $5.800^{3}$ & $5.800^{3}$ 
& \farbetter{$1.009^{3}$} & $2.070^{5}$ & $2.070^{5}$ & $2.056^{5}$ & $2.070^{5}$ & $2.070^{5}$ \\
& \(p=\)4000 & \farbetter{$9.667^{1}$} & $1.749^{4}$ & $1.283^{4}$ & $2.776^{4}$ & $4.261^{3}$ & $3.263^{4}$ 
& \farbetter{$3.891^{1}$} & $9.005^{3}$ & $9.005^{3}$ & $8.991^{3}$ & $9.006^{3}$ & $9.005^{3}$ 
& \farbetter{$1.728^{3}$} & $3.397^{5}$ & $3.397^{5}$ & $3.372^{5}$ & $3.398^{5}$ & $3.397^{5}$ \\
\hline
\multirow{4}{*}{\tabincell{c}{\(q=0.5\)}} 
& \(p=\)1000 & \farbetter{$5.211^{1}$} & $1.080^{3}$ & $6.405^{2}$ & $1.785^{3}$ & $2.660^{2}$ & $2.568^{3}$ 
& \farbetter{$2.323^{1}$} & $9.207^{2}$ & $9.207^{2}$ & $9.202^{2}$ & $9.207^{2}$ & $9.207^{2}$ 
& \farbetter{$6.123^{2}$} & $2.212^{4}$ & $2.212^{4}$ & $2.209^{4}$ & $2.212^{4}$ & $2.212^{4}$ \\
& \(p=\)2000 & \farbetter{$1.619^{1}$} & $7.065^{3}$ & $3.424^{3}$ & $8.235^{3}$ & $1.410^{3}$ & $1.008^{4}$ 
& \farbetter{$9.973^{2}$} & $2.632^{3}$ & $2.632^{3}$ & $2.627^{3}$ & $2.632^{3}$ & $2.632^{3}$ 
& \farbetter{$4.438^{4}$} & $8.369^{4}$ & $8.369^{4}$ & $8.326^{4}$ & $8.371^{4}$ & $8.369^{4}$ \\
& \(p=\)3000 & \farbetter{$8.791^{0}$} & $1.115^{4}$ & $6.821^{3}$ & $1.626^{4}$ & $2.164^{3}$ & $2.534^{4}$ 
& \farbetter{$4.207^{2}$} & $4.899^{3}$ & $4.899^{3}$ & $4.889^{3}$ & $4.899^{3}$ & $4.899^{3}$ 
& \farbetter{$2.268^{4}$} & $1.745^{5}$ & $1.745^{5}$ & $1.732^{5}$ & $1.745^{5}$ & $1.745^{5}$ \\
& \(p=\)4000 & \farbetter{$7.747^{1}$} & $1.866^{4}$ & $1.307^{4}$ & $3.439^{4}$ & $4.378^{3}$ & $3.304^{4}$ 
& \farbetter{$6.246^{1}$} & $7.588^{3}$ & $7.588^{3}$ & $7.574^{3}$ & $7.588^{3}$ & $7.588^{3}$ 
& \farbetter{$2.996^{3}$} & $2.861^{5}$ & $2.861^{5}$ & $2.835^{5}$ & $2.861^{5}$ & $2.861^{5}$ \\
\hline
\multirow{4}{*}{\tabincell{c}{\(q=0.7\)}} 
& \(p=\)1000 & \farbetter{$2.716^{1}$} & $1.054^{3}$ & $6.755^{2}$ & $1.732^{3}$ & $3.837^{2}$ & $1.776^{3}$ 
& \farbetter{$1.865^{1}$} & $6.913^{2}$ & $6.913^{2}$ & $6.908^{2}$ & $6.913^{2}$ & $6.913^{2}$ 
& \farbetter{$4.845^{2}$} & $1.660^{4}$ & $1.660^{4}$ & $1.658^{4}$ & $1.661^{4}$ & $1.660^{4}$ \\
& \(p=\)2000 & \farbetter{$9.958^{0}$} & $4.075^{3}$ & $3.201^{3}$ & $1.096^{4}$ & $1.084^{3}$ & $8.237^{3}$ 
& \farbetter{$1.706^{2}$} & $1.996^{3}$ & $1.996^{3}$ & $1.991^{3}$ & $1.996^{3}$ & $1.996^{3}$ 
& \farbetter{$7.298^{3}$} & $6.332^{4}$ & $6.332^{4}$ & $6.289^{4}$ & $6.334^{4}$ & $6.332^{4}$ \\
& \(p=\)3000 & \farbetter{$1.500^{1}$} & $1.943^{4}$ & $1.320^{4}$ & $1.551^{4}$ & $1.557^{3}$ & $1.613^{4}$ 
& \farbetter{$5.924^{2}$} & $3.724^{3}$ & $3.724^{3}$ & $3.714^{3}$ & $3.724^{3}$ & $3.724^{3}$ 
& \farbetter{$3.223^{4}$} & $1.324^{5}$ & $1.324^{5}$ & $1.312^{5}$ & $1.325^{5}$ & $1.324^{5}$ \\
& \(p=\)4000 & \farbetter{$1.185^{2}$} & $7.080^{3}$ & $1.965^{3}$ & $4.827^{3}$ & $7.751^{2}$ & $6.137^{3}$
& \worse{$1.919^{4}$} & $5.779^{3}$ & $5.779^{3}$ & $5.765^{3}$ & $5.779^{3}$ & $5.779^{3}$ 
& \worse{$1.212^{6}$} & $2.172^{5}$ & $2.172^{5}$ & $2.147^{5}$ & $2.173^{5}$ & $2.172^{5}$ \\
\hline
\multirow{4}{*}{\tabincell{c}{\(q=0.9\)}} 
& \(p=\)1000 & \farbetter{$2.910^{1}$} & $1.812^{2}$ & $9.909^{1}$ & $2.852^{2}$ & $4.342^{1}$ & $3.124^{2}$ 
& \farbetter{$1.476^{2}$} & $3.703^{2}$ & $3.703^{2}$ & $3.698^{2}$ & $3.703^{2}$ & $3.703^{2}$ 
& \farbetter{$4.574^{3}$} & $8.909^{3}$ & $8.909^{3}$ & $8.890^{3}$ & $8.912^{3}$ & $8.909^{3}$ \\
& \(p=\)2000 & \farbetter{$8.030^{0}$} & $7.273^{2}$ & $3.326^{2}$ & $9.573^{2}$ & $1.438^{2}$ & $1.319^{3}$ 
& \farbetter{$3.762^{2}$} & $1.111^{3}$ & $1.111^{3}$ & $1.106^{3}$ & $1.111^{3}$ & $1.111^{3}$ 
& \farbetter{$1.648^{4}$} & $3.519^{4}$ & $3.519^{4}$ & $3.480^{4}$ & $3.522^{4}$ & $3.519^{4}$ \\
& \(p=\)3000 & \farbetter{$1.779^{1}$} & $2.080^{3}$ & $9.098^{2}$ & $2.577^{3}$ & $3.219^{2}$ & $3.562^{3}$ 
& \farbetter{$6.054^{2}$} & $2.067^{3}$ & $2.067^{3}$ & $2.058^{3}$ & $2.067^{3}$ & $2.067^{3}$ 
& \farbetter{$3.277^{4}$} & $7.302^{4}$ & $7.302^{4}$ & $7.183^{4}$ & $7.308^{4}$ & $7.302^{4}$ \\
& \(p=\)4000 & \farbetter{$5.866^{1}$} & $3.513^{3}$ & $1.486^{3}$ & $4.449^{3}$ & $7.098^{2}$ & $6.150^{3}$ 
& \worse{$8.990^{4}$} & $3.218^{3}$ & $3.218^{3}$ & $3.205^{3}$ & $3.219^{3}$ & $3.218^{3}$ 
& \worse{$5.683^{6}$} & $1.202^{5}$ & $1.202^{5}$ & $1.178^{5}$ & $1.203^{5}$ & $1.202^{5}$ \\
\hline
\end{tabular}
}
\end{table}

\end{landscape}

\begin{landscape}
\begin{table}[H] 
\footnotesize
\centering
\caption{Experimental results of different methods on the block covariance structure}\label{Blocktable}
\renewcommand{\arraystretch}{1.2}
\setlength{\tabcolsep}{1.45mm}{
\begin{tabular}{c|c|cccccc|cccccc|cccccc}
\hline
\multirow{2}{*}{\tabincell{c}{Block \\ size}} & \multirow{2}{*}{\tabincell{c}{Dimension}}& \multicolumn{6}{c|}{Time (s)} & \multicolumn{6}{c|}{Frobenius norm} & \multicolumn{6}{c}{Nuclear norm}\\
\cline{3-20}
& & \multicolumn{6}{c|}{LBO/ADMM/LADMM/TOSA/PFBS/FISTA} & \multicolumn{6}{c|}{LBO/ADMM/LADMM/TOSA/PFBS/FISTA} &\multicolumn{6}{c}{LBO/ADMM/LADMM/TOSA/PFBS/FISTA}\\
\hline
\multirow{4}{*}{\tabincell{c}{\(10\)}} 
& \(p=\)1000 & \similar{$6.510^{1}$} & $6.161^{1}$ & $3.976^{1}$ & $2.071^{2}$ & $3.660^{1}$ & $8.718^{1}$ 
& \farbetter{$9.091^{0}$} & $1.398^{1}$ & $1.398^{1}$ & $1.398^{1}$ & $1.398^{1}$ & $1.398^{1}$ 
& \farbetter{$2.754^{2}$} & $2.898^{2}$ & $2.898^{2}$ & $2.898^{2}$ & $2.898^{2}$ & $2.898^{2}$ \\
& \(p=\)2000 & \farbetter{$3.115^{1}$} & $5.935^{2}$ & $3.615^{4}$ & $1.016^{4}$ & $1.122^{2}$ & $2.695^{2}$ 
& \farbetter{$1.747^{1}$} & $2.667^{1}$ & $2.667^{1}$ & $2.667^{1}$ & $2.667^{1}$ & $2.667^{1}$ 
& \farbetter{$7.655^{2}$} & $8.791^{2}$ & $8.791^{2}$ & $8.791^{2}$ & $8.791^{2}$ & $8.791^{2}$ \\
& \(p=\)3000 & \farbetter{$7.809^{1}$} & $1.501^{3}$ & $1.240^{4}$ & $1.927^{3}$ & $2.875^{2}$ & $1.030^{4}$ 
& \farbetter{$2.510^{1}$} & $4.230^{1}$ & $4.230^{1}$ & $4.201^{1}$ & $4.273^{1}$ & $4.230^{1}$ 
& \farbetter{$1.356^{3}$} & $1.771^{3}$ & $1.771^{3}$ & $1.750^{3}$ & $1.791^{3}$ & $1.771^{3}$ \\
& \(p=\)4000 & \farbetter{$1.110^{2}$} & $2.822^{3}$ & $2.972^{3}$ & $4.434^{3}$ & $7.198^{2}$ & $1.337^{4}$ 
& \farbetter{$3.961^{1}$} & $6.118^{1}$ & $6.118^{1}$ & $6.069^{1}$ & $6.291^{1}$ & $6.118^{1}$ 
& \farbetter{$2.480^{3}$} & $2.962^{3}$ & $2.962^{3}$ & $2.917^{3}$ & $3.051^{3}$ & $2.962^{3}$ \\
\hline
\multirow{4}{*}{\tabincell{c}{\(20\)}} 
& \(p=\)1000 & \similar{$1.463^{2}$} & $1.086^{2}$ & $6.271^{1}$ & $2.053^{2}$ & $2.297^{1}$ & $7.434^{1}$ 
& \farbetter{$1.026^{1}$} & $1.659^{1}$ & $1.659^{1}$ & $1.659^{1}$ & $1.659^{1}$ & $1.659^{1}$ 
& \similar{$2.875^{2}$} & $2.281^{2}$ & $2.281^{2}$ & $2.281^{2}$ & $2.281^{2}$ & $2.281^{2}$ \\
& \(p=\)2000 & \farbetter{$1.083^{2}$} & $4.595^{2}$ & $2.816^{2}$ & $8.456^{2}$ & $1.250^{2}$ & $3.036^{2}$ 
& \similar{$2.958^{1}$} & $2.731^{1}$ & $2.731^{1}$ & $2.731^{1}$ & $2.731^{1}$ & $2.731^{1}$ 
& \worse{$1.277^{3}$} & $6.107^{2}$ & $6.107^{2}$ & $6.107^{2}$ & $6.107^{2}$ & $6.107^{2}$ \\
& \(p=\)3000 & \farbetter{$1.378^{2}$} & $1.810^{3}$ & $9.195^{2}$ & $3.200^{3}$ & $3.506^{2}$ & $8.730^{2}$ 
& \farbetter{$2.417^{1}$} & $3.432^{1}$ & $3.432^{1}$ & $3.432^{1}$ & $3.432^{1}$ & $3.432^{1}$ 
& \similar{$1.273^{3}$} & $1.082^{3}$ & $1.082^{3}$ & $1.082^{3}$ & $1.082^{3}$ & $1.082^{3}$ \\
& \(p=\)4000 & \farbetter{$2.131^{2}$} & $5.439^{3}$ & $1.088^{4}$ & $3.391^{3}$ & $5.844^{2}$ & $9.225^{3}$ 
& \farbetter{$2.883^{1}$} & $4.526^{1}$ & $4.526^{1}$ & $4.512^{1}$ & $4.554^{1}$ & $4.526^{1}$ 
& \farbetter{$1.760^{3}$} & $1.795^{3}$ & $1.795^{3}$ & $1.778^{3}$ & $1.814^{3}$ & $1.795^{3}$ \\
\hline
\multirow{4}{*}{\tabincell{c}{\(25\)}} 
& \(p=\)1000 & \similar{$9.749^{1}$} & $5.469^{1}$ & $3.902^{1}$ & $1.182^{2}$ & $2.031^{1}$ & $4.207^{1}$ 
& \farbetter{$5.476^{0}$} & $1.803^{1}$ & $1.803^{1}$ & $1.803^{1}$ & $1.803^{1}$ & $1.803^{1}$ 
& \farbetter{$1.513^{2}$} & $2.246^{2}$ & $2.246^{2}$ & $2.246^{2}$ & $2.246^{2}$ & $2.246^{2}$ \\
& \(p=\)2000 & \farbetter{$5.059^{1}$} & $2.774^{2}$ & $1.618^{2}$ & $4.594^{2}$ & $6.818^{1}$ & $1.559^{2}$ 
& \farbetter{$1.959^{1}$} & $2.541^{1}$ & $2.541^{1}$ & $2.541^{1}$ & $2.541^{1}$ & $2.541^{1}$ 
& \worse{$8.293^{2}$} & $4.896^{2}$ & $4.896^{2}$ & $4.896^{2}$ & $4.896^{2}$ & $4.896^{2}$ \\
& \(p=\)3000 & \farbetter{$1.031^{2}$} & $6.965^{2}$ & $4.540^{2}$ & $3.223^{4}$ & $9.475^{3}$ & $2.300^{4}$ 
& \farbetter{$3.132^{1}$} & $3.610^{1}$ & $3.610^{1}$ & $3.610^{1}$ & $3.610^{1}$ & $3.610^{1}$ 
& \worse{$1.659^{3}$} & $1.001^{3}$ & $1.001^{3}$ & $1.001^{3}$ & $1.001^{3}$ & $1.001^{3}$ \\
& \(p=\)4000 & \farbetter{$2.002^{2}$} & $7.551^{4}$ & $1.606^{4}$ & $2.168^{3}$ & $3.224^{2}$ & $7.248^{3}$ 
& \farbetter{$3.152^{1}$} & $4.276^{1}$ & $4.276^{1}$ & $4.273^{1}$ & $4.279^{1}$ & $4.276^{1}$ 
& \similar{$1.894^{3}$} & $1.475^{3}$ & $1.475^{3}$ & $1.472^{3}$ & $1.478^{3}$ & $1.475^{3}$ \\
\hline
\multirow{4}{*}{\tabincell{c}{\(40\)}} 
& \(p=\)1000 & \similar{$6.530^{1}$} & $1.542^{2}$ & $9.000^{1}$ & $2.263^{2}$ & $3.570^{1}$ & $8.461^{1}$ 
& \farbetter{$8.135^{0}$} & $2.085^{1}$ & $2.085^{1}$ & $2.085^{1}$ & $2.085^{1}$ & $2.085^{1}$ 
& \farbetter{$2.087^{2}$} & $2.176^{2}$ & $2.176^{2}$ & $2.176^{2}$ & $2.176^{2}$ & $2.176^{2}$ \\
& \(p=\)2000 & \farbetter{$4.518^{1}$} & $3.075^{4}$ & $1.531^{4}$ & $8.015^{2}$ & $1.164^{2}$ & $2.636^{2}$ 
& \farbetter{$2.149^{1}$} & $3.009^{1}$ & $3.009^{1}$ & $3.009^{1}$ & $3.009^{1}$ & $3.009^{1}$ 
& \worse{$8.180^{2}$} & $4.781^{2}$ & $4.781^{2}$ & $4.781^{2}$ & $4.781^{2}$ & $4.781^{2}$ \\
& \(p=\)3000 & \farbetter{$9.057^{1}$} & $1.387^{3}$ & $7.799^{2}$ & $2.261^{3}$ & $3.378^{2}$ & $8.672^{2}$ 
& \similar{$5.749^{1}$} & $3.979^{1}$ & $3.979^{1}$ & $3.979^{1}$ & $3.979^{1}$ & $3.979^{1}$ 
& \worse{$3.062^{3}$} & $7.840^{2}$ & $7.840^{2}$ & $7.840^{2}$ & $7.840^{2}$ & $7.840^{2}$ \\
& \(p=\)4000 & \farbetter{$2.496^{2}$} & $1.417^{4}$ & $3.208^{3}$ & $8.529^{3}$ & $6.988^{2}$ & $1.353^{4}$ 
& \similar{$8.680^{1}$} & $4.754^{1}$ & $4.754^{1}$ & $4.754^{1}$ & $4.754^{1}$ & $4.754^{1}$ 
& \worse{$5.411^{3}$} & $1.104^{3}$ & $1.104^{3}$ & $1.104^{3}$ & $1.104^{3}$ & $1.104^{3}$ \\
\hline 
\multirow{4}{*}{\tabincell{c}{\(50\)}} 
& \(p=\)1000 & \similar{$6.172^{1}$} & $9.923^{1}$ & $6.181^{1}$ & $1.823^{2}$ & $2.731^{1}$ & $6.458^{1}$ 
& \farbetter{$1.018^{1}$} & $2.493^{1}$ & $2.493^{1}$ & $2.493^{1}$ & $2.493^{1}$ & $2.493^{1}$ 
& \similar{$2.688^{2}$} & $2.278^{2}$ & $2.278^{2}$ & $2.278^{2}$ & $2.278^{2}$ & $2.278^{2}$ \\
& \(p=\)2000 & \farbetter{$4.625^{1}$} & $3.972^{2}$ & $2.225^{2}$ & $7.073^{2}$ & $1.049^{2}$ & $2.211^{2}$ 
& \farbetter{$1.426^{1}$} & $3.751^{1}$ & $3.751^{1}$ & $3.751^{1}$ & $3.751^{1}$ & $3.751^{1}$ 
& \farbetter{$3.403^{2}$} & $4.930^{2}$ & $4.930^{2}$ & $4.930^{2}$ & $4.930^{2}$ & $4.930^{2}$ \\
& \(p=\)3000 & \farbetter{$8.846^{1}$} & $1.138^{3}$ & $7.045^{2}$ & $2.066^{3}$ & $3.131^{2}$ & $8.167^{2}$ 
& \farbetter{$4.141^{1}$} & $4.414^{1}$ & $4.414^{1}$ & $4.414^{1}$ & $4.414^{1}$ & $4.414^{1}$ 
& \worse{$2.061^{3}$} & $7.819^{2}$ & $7.819^{2}$ & $7.819^{2}$ & $7.819^{2}$ & $7.819^{2}$ \\
& \(p=\)4000 & \farbetter{$5.232^{1}$} & $4.255^{3}$ & $2.890^{3}$ & $3.315^{3}$ & $1.212^{3}$ & $2.626^{4}$ 
& \farbetter{$4.669^{1}$} & $5.317^{1}$ & $5.317^{1}$ & $5.317^{1}$ & $5.317^{1}$ & $5.317^{1}$ 
& \similar{$2.667^{3}$} & $1.117^{3}$ & $1.117^{3}$ & $1.117^{3}$ & $1.118^{3}$ & $1.117^{3}$ \\
\hline
\end{tabular}
}
\end{table}

\end{landscape}

\begin{figure}[H]
\centering 
\begin{minipage}[t]{0.95\textwidth}
\centering
\includegraphics[scale=0.39]{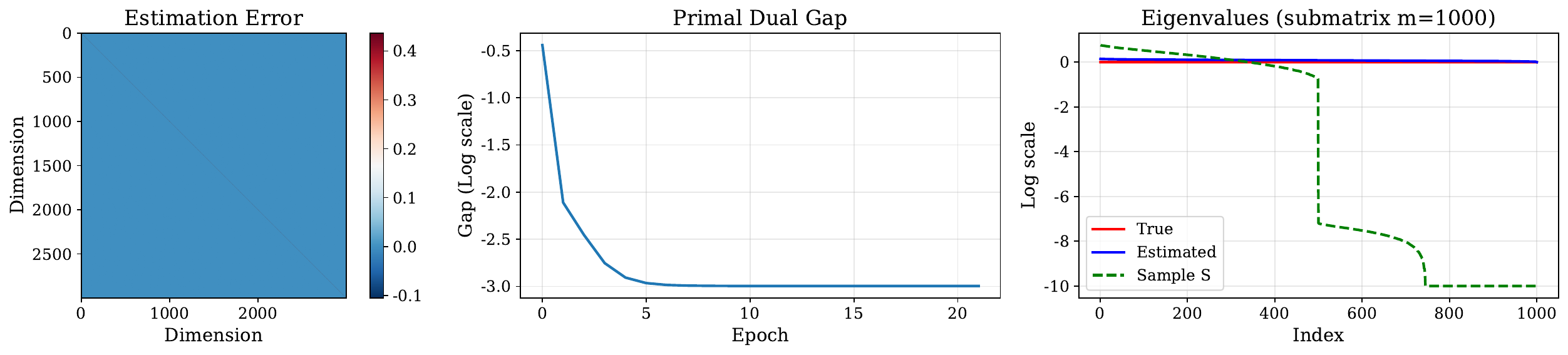}
\end{minipage}
	
\begin{minipage}[t]{0.95\textwidth}
\centering
\includegraphics[scale=0.39]{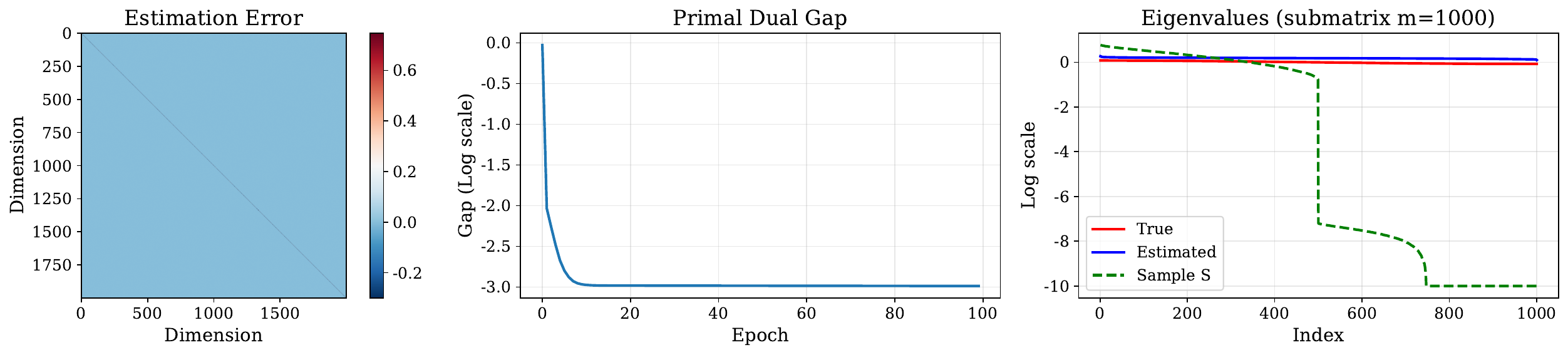}
\end{minipage}
\caption{Visualization of LBO for Toeplitz covariance matrix estimation. The three columns display (from left to right) the visualization of the error matrix, the trajectory of the primal dual gap versus epochs, and the leading 1000 eigenvalues. The first row corresponds to $\varrho=0.1$ and $p=3000$, while the second row corresponds to $\varrho=0.3$ and $p=2000$.}\label{LBO_toeplitzexample} 
\end{figure}

Table~\ref{Toeplitztable} shows that LBO provides a uniformly superior accuracy--efficiency trade-off on Toeplitz covariance estimation.
In terms of runtime, LBO is typically the fastest or among the fastest methods in moderate to large dimensions: for example, at $\varrho=0.1$ and $p=3000$, LBO finishes in $1.529^{1}$ seconds, while ADMM and LADMM require $1.517^{3}$ and $1.750^{2}$ seconds, respectively; at $\varrho=0.1$ and $p=4000$, LBO takes $2.312^{1}$ seconds compared with $4.029^{3}$ (ADMM) and $3.610^{2}$ (LADMM).
More importantly, LBO achieves substantially better solution quality in both metrics: at $\varrho=0.1$ and $p=2000$, LBO attains a Frobenius norm of $9.074^{-2}$ versus $\approx 9.2^{0}$ for TOSA/PFBS/FISTA (about two orders of magnitude smaller), and a nuclear norm of $2.391^{0}$ versus $\approx 3.47^{2}$ (over $10^{2}$ times smaller); similarly, at $\varrho=0.1$ and $p=3000$, LBO yields $1.879^{-1}$ in Frobenius norm compared with $\approx 1.27^{1}$ for TOSA/PFBS/FISTA, and $8.285^{0}$ in nuclear norm compared with $\approx 5.85^{2}$.

\begin{figure}[H]
\centering 
\begin{minipage}[t]{0.95\textwidth}
\centering
\includegraphics[scale=0.39]{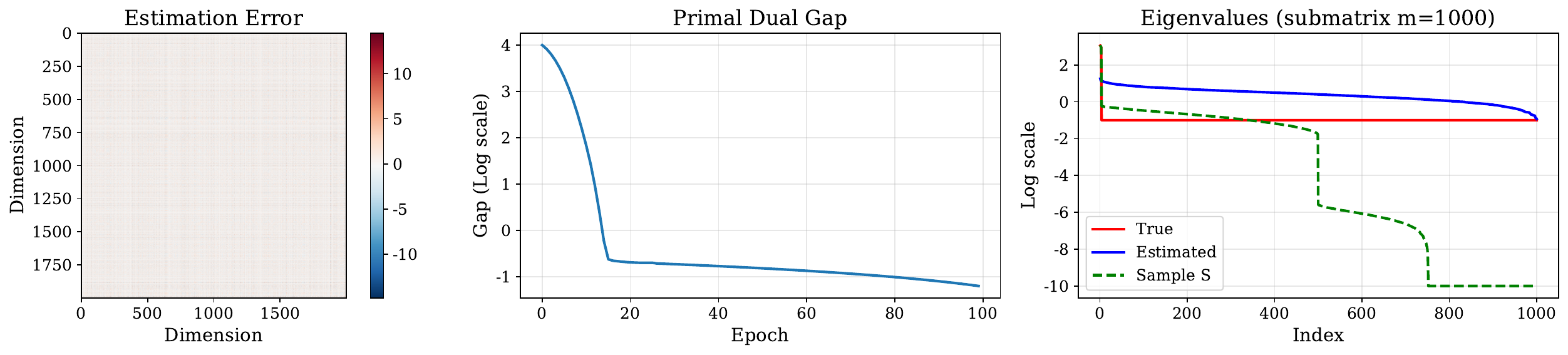}
\end{minipage}

\begin{minipage}[t]{0.95\textwidth}
\centering
\includegraphics[scale=0.39]{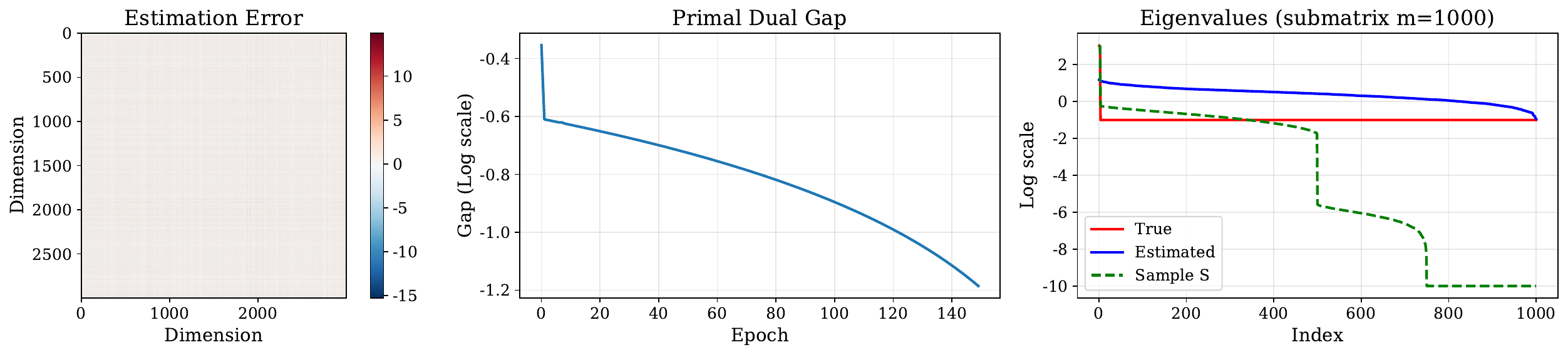}
\end{minipage}
\caption{Visualization of LBO for factor covariance matrix estimation. The three columns display (from left to right) the visualization of the error matrix, the trajectory of the primal--dual gap versus epochs, and the leading 1000 eigenvalues. The first row corresponds to $m=3$ and $p=2000$, while the second row corresponds to $m=3$ and $p=3000$.}\label{LBO_factorexample}
\end{figure}

Table~\ref{Factortable} demonstrates that LBO dominates all baselines on the factor covariance model across factor numbers $m\in\{3,5,7,9,10\}$ and dimensions $p\in\{1000,2000,3000,4000\}$.
The runtime advantage is striking in large-scale regimes: for instance, when $m=3$ and $p=4000$, LBO needs only $2.877^{1}$ seconds, whereas ADMM, LADMM, TOSA, PFBS, and FISTA take $1.040^{4}$, $1.437^{4}$, $2.114^{4}$, $5.231^{3}$, and $1.690^{5}$ seconds, respectively.
At the same time, LBO achieves markedly smaller errors: in the same setting $(m,p)=(3,4000)$, LBO reduces the Frobenius norm to $8.746^{1}$ compared with $5.919^{2}$ (ADMM/LADMM) and $\approx 5.93^{2}$ (TOSA/PFBS/FISTA), and reduces the nuclear norm to $3.018^{2}$ compared with $3.148^{3}$ (ADMM/LADMM) and $3.708^{3}$ (TOSA).
A similar gap persists for larger factor numbers, e.g., at $m=10$ and $p=3000$, LBO achieves $1.196^{2}$ (Frobenius) and $7.412^{2}$ (nuclear) versus $1.385^{3}$ and $7.687^{3}$ for ADMM/LADMM, while remaining orders of magnitude faster than the competing solvers.

\begin{figure}[H]
\centering 
\begin{minipage}[t]{0.95\textwidth}
\centering
\includegraphics[scale=0.39]{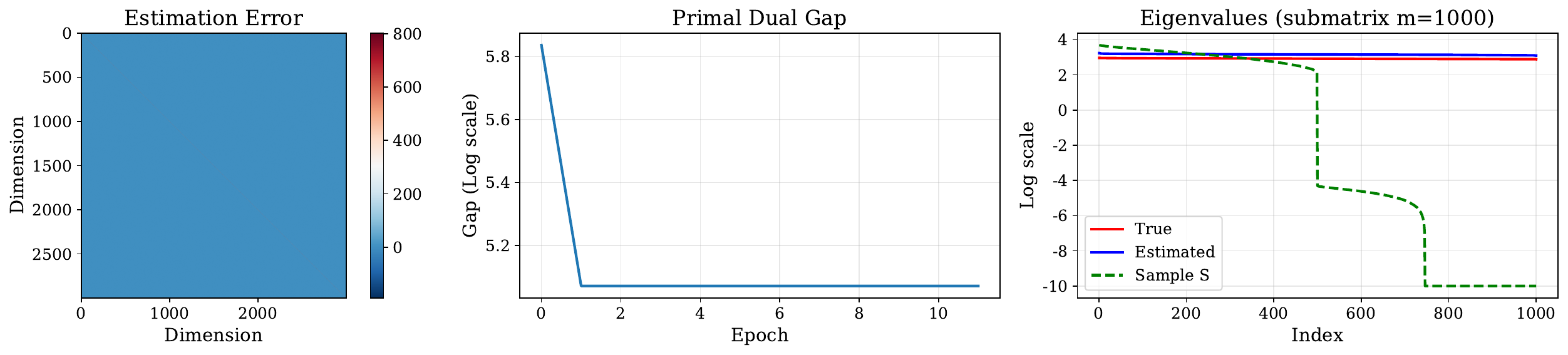}
\end{minipage}

\begin{minipage}[t]{0.95\textwidth}
\centering
\includegraphics[scale=0.39]{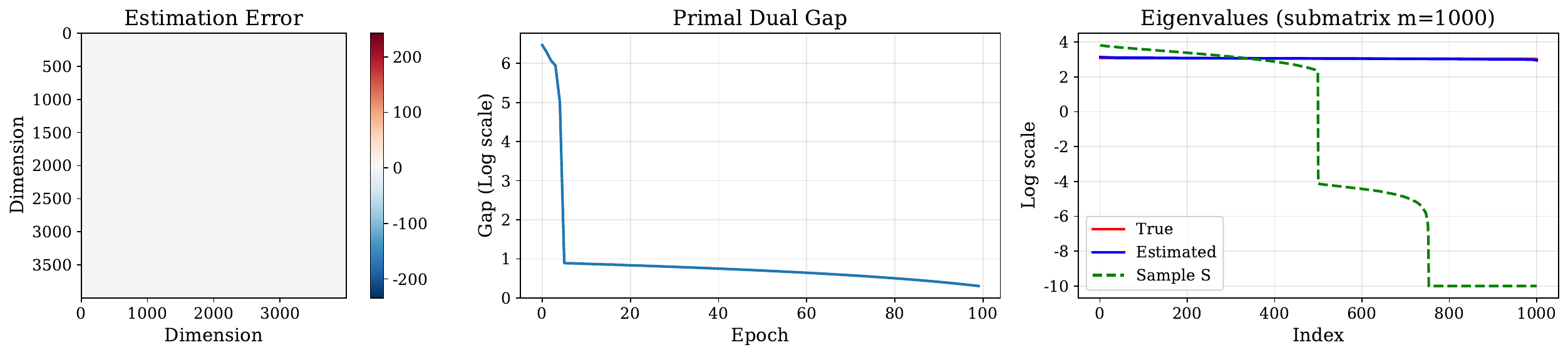}
\end{minipage}
\caption{Visualization of LBO for sparse covariance matrix estimation. The three columns display (from left to right) the visualization of the error matrix, the trajectory of the primal dual gap versus epochs, and the leading 1000 eigenvalues. The first row corresponds to $q=0.5$ and $p=3000$, while the second row corresponds to $q=0.5$ and $p=4000$.}\label{LBO_sparseexample}
\end{figure}

Table~\ref{Sparsetable} reports results under sparse covariance models with varying sparsity levels $q\in\{0.1,0.3,0.5,0.7,0.9\}$ and dimensions $p\in\{1000,2000,3000,4000\}$.
For moderately sparse regimes ($q\le 0.5$), LBO consistently dominates all competing methods in both efficiency and accuracy.
For example, at $q=0.1$ and $p=4000$, LBO terminates in $2.527^{1}$ seconds, whereas ADMM, LADMM, and TOSA require $9.748^{3}$, $5.041^{3}$, and $1.859^{4}$ seconds, respectively; meanwhile, LBO attains a Frobenius norm of $7.315^{3}$ versus $\approx 1.024^{4}$ for ADMM/LADMM/TOSA/PFBS/FISTA, and a nuclear norm of $6.878^{4}$ at $p=3000$ compared with $\approx 2.374^{5}$ for the competing methods.
A similar advantage persists at $q=0.3$ and $p=3000$, where LBO runs in $1.019^{1}$ seconds (versus $6.433^{3}$ for ADMM and $1.113^{4}$ for TOSA) and achieves markedly smaller Frobenius and nuclear norms ($2.635^{1}$ and $1.009^{3}$) than the baselines (on the order of $10^{3}$ and $10^{5}$, respectively).
As the sparsity level increases to dense regimes ($q\ge 0.7$), LBO remains substantially faster than all baselines, but its accuracy advantage becomes less uniform and can even deteriorate in the most challenging dense settings.
In particular, at $q=0.7$ and $p=4000$, LBO is still orders of magnitude faster ($1.185^{2}$ seconds versus $7.080^{3}$ for ADMM and $4.827^{3}$ for TOSA), yet its Frobenius and nuclear norms increase to $1.919^{4}$ and $1.212^{6}$, compared with $\approx 5.779^{3}$ and $\approx 2.172^{5}$ for the competing methods.
A similar phenomenon is observed at $q=0.9$ and $p=4000$, where LBO attains $8.990^{4}$ (Frobenius) and $5.683^{6}$ (nuclear), while the baselines remain around $3.2^{3}$ and $1.2^{5}$.

\begin{figure}[H]
\centering 
\begin{minipage}[t]{0.95\textwidth}
\centering
\includegraphics[scale=0.39]{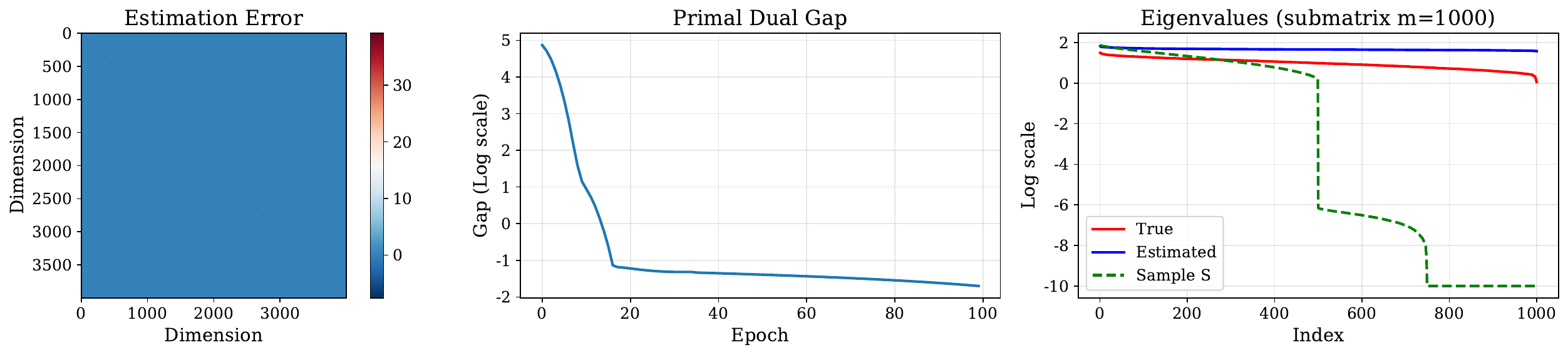}
\end{minipage}

\begin{minipage}[t]{0.95\textwidth}
\centering
\includegraphics[scale=0.39]{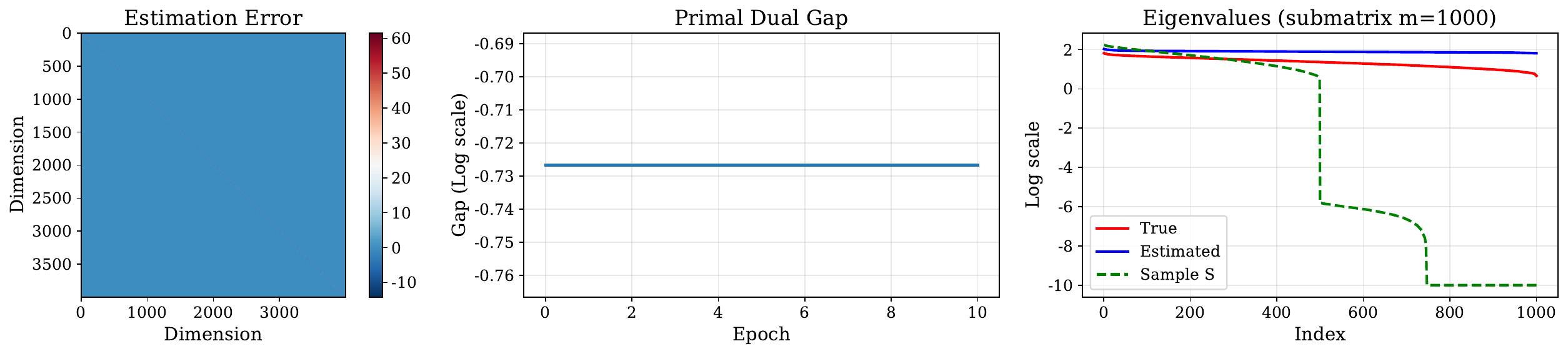}
\end{minipage}
\caption{Visualization of LBO for block covariance matrix estimation. The three columns display (from left to right) the visualization of the error matrix, the trajectory of the primal--dual gap versus epochs, and the leading 1000 eigenvalues. The first row corresponds to block size $20$ and $p=4000$, while the second row corresponds to block size $50$ and $p=4000$.}\label{LBO_blockexample}
\end{figure}

Table~\ref{Blocktable} reports results under block covariance models with varying block sizes and dimensions $p\in\{1000,2000,3000,4000\}$.
Overall, LBO remains highly competitive in runtime and achieves consistently smaller Frobenius norms across essentially all settings, indicating a robust advantage in estimation accuracy.
For instance, when the block size is $10$ and $p=4000$, LBO completes in $1.110^{2}$ seconds, whereas ADMM/LADMM/TOSA require $2.822^{3}$/$2.972^{3}$/$4.434^{3}$ seconds; meanwhile, LBO attains a Frobenius norm of $3.961^{1}$ compared with $\approx 6.1^{1}$ for the competing methods.
A similar pattern holds at block size $20$ and $p=4000$, where LBO runs in $2.131^{2}$ seconds (versus $5.439^{3}$ for ADMM and $1.088^{4}$ for LADMM) and yields a smaller Frobenius norm ($2.883^{1}$ versus $\approx 4.53^{1}$).

In terms of the nuclear norm, LBO performs well for small to moderate block sizes, often achieving the best or comparable results.
For example, at block size $10$ and $p=3000$, LBO reduces the nuclear norm to $1.356^{3}$ versus $\approx 1.77^{3}$ for ADMM/LADMM; at block size $10$ and $p=4000$, LBO attains $2.480^{3}$ compared with $\approx 2.96^{3}$.
However, as the block size increases, the nuclear-norm advantage becomes less uniform and can even reverse: for block size $40$ and $p=4000$, LBO attains a nuclear norm of $5.411^{3}$ whereas the competing methods are around $1.104^{3}$; similarly, for block size $50$ and $p=3000$, LBO yields $2.061^{3}$ compared with $\approx 7.819^{2}$.
Notably, even in these challenging regimes where the nuclear norm is not the smallest, LBO still preserves a substantial speed advantage (e.g., block size $40$, $p=4000$: $2.496^{2}$ seconds for LBO versus $1.417^{4}$ for ADMM) and remains highly competitive in Frobenius accuracy.

Across the four covariance structures, LBO consistently demonstrates a clear advantage in large-scale covariance estimation.
For Toeplitz and Factor models, LBO achieves the best overall performance: it is typically the fastest (or among the fastest) while simultaneously attaining substantially smaller Frobenius and nuclear norms, often by one to two orders of magnitude relative to ADMM/LADMM and the first-order baselines (TOSA/PFBS/FISTA).
For sparse models, LBO remains highly effective in sparse-to-moderately-sparse regimes ($q\le 0.5$), where it markedly improves both metrics with dramatically reduced runtime (e.g., at $q=0.1$, $p=4000$, LBO runs in $2.527^{1}$ seconds versus $9.748^{3}$ for ADMM and achieves a Frobenius norm of $7.315^{3}$ versus $\approx 1.024^{4}$), whereas in dense regimes ($q\ge 0.7$) its main advantage is speed and the accuracy benefit becomes less uniform.
For block models, LBO yields consistently smaller Frobenius norms and substantial runtime reductions across all block sizes (e.g., block size $10$, $p=4000$: $1.110^{2}$ seconds for LBO versus $2.822^{3}$ for ADMM), while its nuclear-norm advantage is strongest for small-to-moderate blocks and may diminish for very large blocks.
Overall, these results indicate that LBO offers superior scalability and robustness across diverse structural settings, delivering the most favorable balance between computational efficiency and estimation quality among all compared methods.

\subsection{High-dimensional precision matrix}
In this section, we apply the proposed LBO algorithm to high-dimensional precision matrix estimation. We first describe the precision matrix structures employed in the experiments, followed by an overview of the optimization algorithms used for comparison.

Our objective is to estimate the high-dimensional precision matrix $\Theta^\star = (\theta^\star_{ij})_{1\leq i,j\leq p}$ based on an observation matrix with \(n\) samples and \(p\) features sampled from the normal distribution $\mathcal{N}(0, (\Theta^\star)^{-1})$. To this end, we consider three representative structures for the precision matrix. Banded 1: A banded precision matrix where \(\theta^\star_{ii}=1\) and \(\theta^\star_{ij}=0.2\) for \(1 \leq |i-j| \leq 2\), and \(\theta^\star_{ij}=0\) otherwise. Banded 2: A banded precision matrix where \(\theta^\star_{ii}=1\) and \(\theta^\star_{ij}=0.2\) for \(1 \le |i-j| \le 4\), and \(\theta^\star_{ij}=0\) otherwise. Grid: Let $p$ be a perfect square, and set $\sqrt{p} \in \mathbb{N}$. The indices $i=1,\dots,p$ are arranged in row-major order into an $\sqrt{p} \times \sqrt{p}$ grid. The true precision matrix $\Theta^\star \in \mathbb{R}^{p \times p}$ is defined as
\[
\theta^\star_{ij} =
\begin{cases}
1, & i=j, \\
0.2, & \big(j=i+1 \ \text{and}\ \mathrm{mod}(i,\sqrt{p})\neq 0\big) \ \text{(horizontal neighbor)}, \\
0.2, & \big(j=i+\sqrt{p} \ \text{and}\ i\le p-\sqrt{p}\big) \ \text{(vertical neighbor)}, \\
0, & \text{otherwise},
\end{cases}
\]
where $\mathrm{mod}(i,\sqrt{p})\neq 0$ indicates that $i$ is not a multiple of $\sqrt{p}$ (i.e., not on the right boundary of each row). This construction requires $\sqrt{p}$ to be an integer. Since each node has at most four neighbors and each off-diagonal entry is $0.2$, the matrix is strictly row-diagonally dominant (with interior rows summing to at most $0.8$ off the diagonal), and thus $\Theta^\star$ is symmetric positive definite. This is the grid model from \citet{ravikumar2011high}. As an illustration, Figure \ref{precisionvis} visualizes the precision matrices for these three structures when $p=100$.
\begin{figure}[H]
\centering
\includegraphics[scale=0.38]{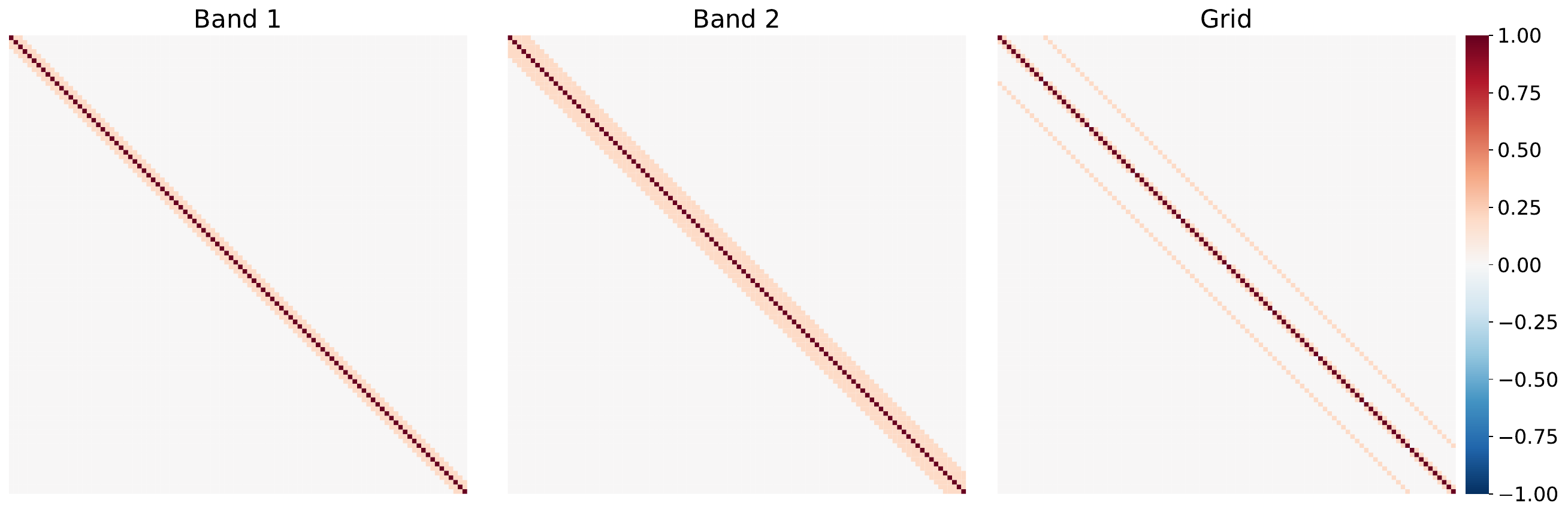}
\caption{Visualization of the precision matrices for the three structures at dimension 100.}
\label{precisionvis}
\end{figure}

Similar to Section \ref{secHighcovariance}, we fix the sample size at $n=500$ and vary the dimension $p \in \{1024, 2025, 3025, 3969\}$, where each $p$ is chosen to be a perfect square to meet the experimental setup requirements. We evaluate all methods using wall-clock convergence time, Frobenius norm, nuclear norm, and the duality gap. In addition to the ADMM and LADMM solvers considered above, we further include three representative baselines: a proximal gradient (ProxGrad) method \citep{grishchenko2021proximal} equipped with backtracking line \citep{pedregosa2020linearly} search to guarantee descent of the objective, which is simple to implement and enjoys standard convergence guarantees for general-purpose use on moderately sized problems; QUIC \citep{hsieh2013big, hsieh2014quic}, an approximate second-order approach that updates along (quasi-)Newton directions with line search, typically achieving substantially faster convergence in practice and being particularly effective when rapid convergence is needed for medium-scale instances; and the spectral projected gradient (SPG) method \citep{birgin2014spectral}, which combines Barzilai--Borwein step sizes \citep{tan2016barzilai} with a nonmonotone acceptance criterion, providing adaptive step-size selection that often leads to fast convergence and good scalability in large-scale settings. Table~\ref{precisiontable} reports the performance of LBO against ADMM, LADMM, ProxGrad, QUIC, and SPG on three representative precision-matrix structures (two banded cases and a grid case) across increasing dimensions. The visualization results of LBO on the four covariance structures are presented in Figures~\ref{LBO_band1example}--\ref{LBO_gridexample}.

\begin{landscape}
\begin{table}[H] 
\footnotesize
\centering
\caption{Experimental results of different methods on the three precision matrix structures}\label{precisiontable}
\renewcommand{\arraystretch}{1.2}
\setlength{\tabcolsep}{1.45mm}{
\begin{tabular}{c|c|cccccc|cccccc|cccccc}
\hline
\multirow{2}{*}{\tabincell{c}{Matrix \\ structures}} & \multirow{2}{*}{\tabincell{c}{Dimension}}& \multicolumn{6}{c|}{Time (s)} & \multicolumn{6}{c|}{Frobenius norm} & \multicolumn{6}{c}{Nuclear norm}\\
\cline{3-20}
& & \multicolumn{6}{c|}{LBO/ADMM/LADMM/ProxGrad/QUIC/SPG} & \multicolumn{6}{c|}{LBO/ADMM/LADMM/ProxGrad/QUIC/SPG} &\multicolumn{6}{c}{LBO/ADMM/LADMM/ProxGrad/QUIC/SPG}\\
\hline
\multirow{4}{*}{\tabincell{c}{Banded 1}} 
& \(p=\)1024 & \similar{$1.704^{1}$} & $1.879^{1}$ & $3.668^{1}$ & $1.209^{1}$ & $1.241^{1}$ & $8.694^{0}$ 
& \farbetter{$7.887^{0}$} & $9.874^{0}$ & $9.872^{0}$ & $9.875^{0}$ & $9.209^{0}$ & $9.875^{0}$ 
& \farbetter{$2.207^{2}$} & $2.535^{2}$ & $2.535^{2}$ & $2.535^{2}$ & $2.449^{2}$ & $2.535^{2}$ \\
& \(p=\)2025 & \farbetter{$4.268^{1}$} & $2.045^{3}$ & $4.592^{3}$ & $3.423^{2}$ & $4.649^{2}$ & $4.244^{2}$ 
& \farbetter{$1.129^{1}$} & $1.491^{1}$ & $1.490^{1}$ & $1.491^{1}$ & $1.581^{1}$ & $1.491^{1}$ 
& \farbetter{$4.444^{2}$} & $5.533^{2}$ & $5.532^{2}$ & $5.533^{2}$ & $5.663^{2}$ & $5.533^{2}$ \\
& \(p=\)3025 & \farbetter{$8.675^{1}$} & $1.958^{3}$ & $3.164^{3}$ & $9.287^{2}$ & $9.456^{2}$ & $8.344^{2}$ 
& \farbetter{$1.412^{1}$} & $1.936^{1}$ & $1.935^{1}$ & $1.936^{1}$ & $2.006^{1}$ & $1.936^{1}$ 
& \farbetter{$3.452^{2}$} & $8.902^{2}$ & $8.907^{2}$ & $8.902^{2}$ & $8.876^{2}$ & $8.902^{2}$ \\
& \(p=\)3969 & \farbetter{$1.491^{2}$} & $1.968^{3}$ & $4.395^{3}$ & $1.974^{3}$ & $1.261^{3}$ & $8.140^{2}$ 
& \farbetter{$1.632^{1}$} & $2.324^{1}$ & $2.330^{1}$ & $2.324^{1}$ & $2.371^{1}$ & $2.324^{1}$ 
& \farbetter{$4.340^{2}$} & $1.232^{3}$ & $1.238^{3}$ & $1.232^{3}$ & $1.209^{3}$ & $1.232^{3}$ \\
\hline
\multirow{4}{*}{\tabincell{c}{Banded 2}} 
& \(p=\)1024 & \farbetter{$1.783^{1}$} & $5.804^{2}$ & $2.281^{3}$ & $7.400^{2}$ & $7.344^{2}$ & $6.306^{2}$ 
& \farbetter{$1.266^{1}$} & $1.588^{1}$ & $1.588^{1}$ & $1.588^{1}$ & $1.636^{1}$ & $1.588^{1}$ 
& \similar{$3.408^{2}$} & $3.404^{2}$ & $3.404^{2}$ & $3.404^{2}$ & $3.412^{2}$ & $3.404^{2}$ \\
& \(p=\)2025 & \farbetter{$4.371^{1}$} & $1.452^{3}$ & $1.682^{3}$ & $3.213^{2}$ & $7.639^{2}$ & $4.266^{2}$ 
& \farbetter{$1.793^{1}$} & $2.299^{1}$ & $2.294^{1}$ & $2.299^{1}$ & $2.363^{1}$ & $2.299^{1}$ 
& \farbetter{$6.936^{2}$} & $7.343^{2}$ & $7.359^{2}$ & $7.343^{2}$ & $7.220^{2}$ & $7.343^{2}$ \\
& \(p=\)3025 & \farbetter{$8.431^{1}$} & $3.283^{3}$ & $2.984^{3}$ & $1.376^{3}$ & $1.217^{3}$ & $8.575^{2}$ 
& \farbetter{$2.182^{1}$} & $2.882^{1}$ & $2.877^{1}$ & $2.882^{1}$ & $2.897^{1}$ & $2.882^{1}$ 
& \farbetter{$5.114^{2}$} & $1.165^{3}$ & $1.186^{3}$ & $1.165^{3}$ & $1.101^{3}$ & $1.165^{3}$ \\
& \(p=\)3969 & \farbetter{$1.477^{2}$} & $2.609^{3}$ & $4.258^{3}$ & $1.623^{3}$ & $3.997^{2}$ & $8.492^{2}$ 
& \farbetter{$2.475^{1}$} & $3.373^{1}$ & $3.404^{1}$ & $3.373^{1}$ & $3.103^{1}$ & $4.079^{3}$ 
& \farbetter{$5.435^{2}$} & $1.599^{3}$ & $1.675^{3}$ & $1.599^{3}$ & $1.483^{3}$ & $2.252^{5}$ \\
\hline
\multirow{4}{*}{\tabincell{c}{Grid}} 
& \(p=\)1024 & \farbetter{$1.518^{1}$} & $1.457^{3}$ & $2.081^{3}$ & $2.303^{3}$ & $8.478^{1}$ & $4.549^{2}$ 
& \similar{$1.265^{1}$ } & $7.322^{0}$ & $7.283^{0}$ & $7.322^{0}$ & $6.836^{0}$ & $6.012^{1}$ 
& \worse{$3.291^{2}$} & $1.943^{2}$ & $1.938^{2}$ & $1.944^{2}$ & $1.832^{2}$ & $2.536^{2}$ \\
& \(p=\)2025 & \farbetter{$8.985^{1}$} & $1.874^{3}$ & $1.702^{3}$ & $1.490^{3}$ & $3.130^{2}$ & $8.899^{2}$ 
& \similar{$1.793^{1}$} & $1.172^{1}$ & $1.197^{1}$ & $1.172^{1}$ & $1.389^{1}$ & $1.172^{1}$ 
& \worse{$6.564^{2}$} & $4.420^{2}$ & $4.541^{2}$ & $4.420^{2}$ & $5.070^{2}$ & $4.420^{2}$ \\
& \(p=\)3025 & \farbetter{$8.658^{1}$} & $2.989^{3}$ & $2.434^{3}$ & $4.191^{3}$ & $1.090^{3}$ & $1.309^{3}$ 
& \similar{$2.192^{1}$} & $1.582^{1}$ & $1.714^{1}$ & $1.582^{1}$ & $1.741^{1}$ & $1.582^{1}$ 
& \farbetter{$4.836^{2}$} & $7.326^{2}$ & $7.960^{2}$ & $7.326^{2}$ & $7.804^{2}$ & $7.326^{2}$ \\
& \(p=\)3969 & \farbetter{$1.466^{2}$} & $2.121^{3}$ & $9.863^{2}$ & $1.495^{3}$ & $2.344^{2}$ & $3.247^{2}$ 
& \similar{$2.576^{1}$} & $1.942^{1}$ & $2.216^{1}$ & $1.942^{1}$ & $2.037^{1}$ & $1.942^{1}$ 
& \farbetter{$6.828^{2}$} & $1.031^{3}$ & $1.176^{3}$ & $1.031^{3}$ & $1.049^{3}$ & $1.031^{3}$ \\
\hline
\end{tabular}
}
\end{table}

\end{landscape}

\begin{figure}[H]
\centering 
\begin{minipage}[t]{0.95\textwidth}
\centering
\includegraphics[scale=0.39]{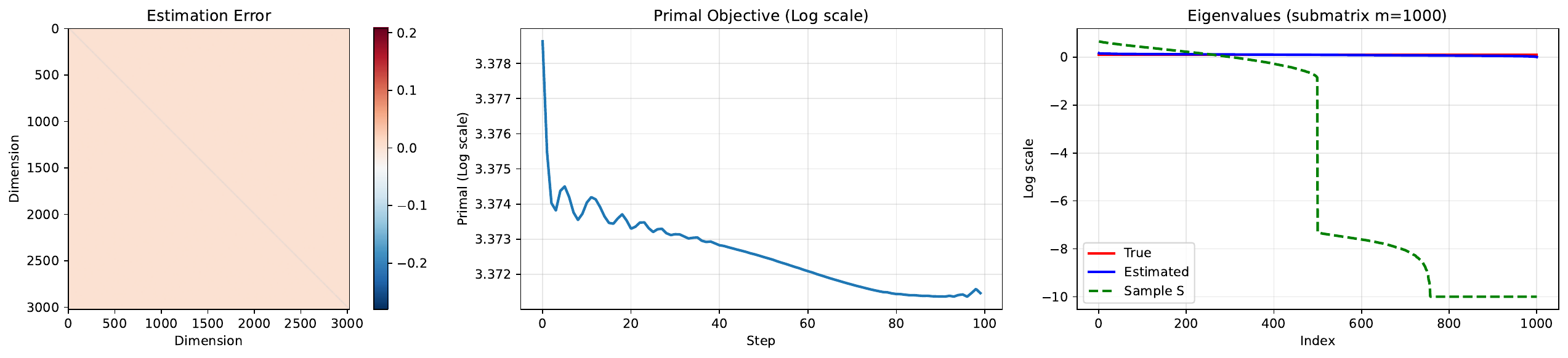}
\end{minipage}

\begin{minipage}[t]{0.95\textwidth}
\centering
\includegraphics[scale=0.39]{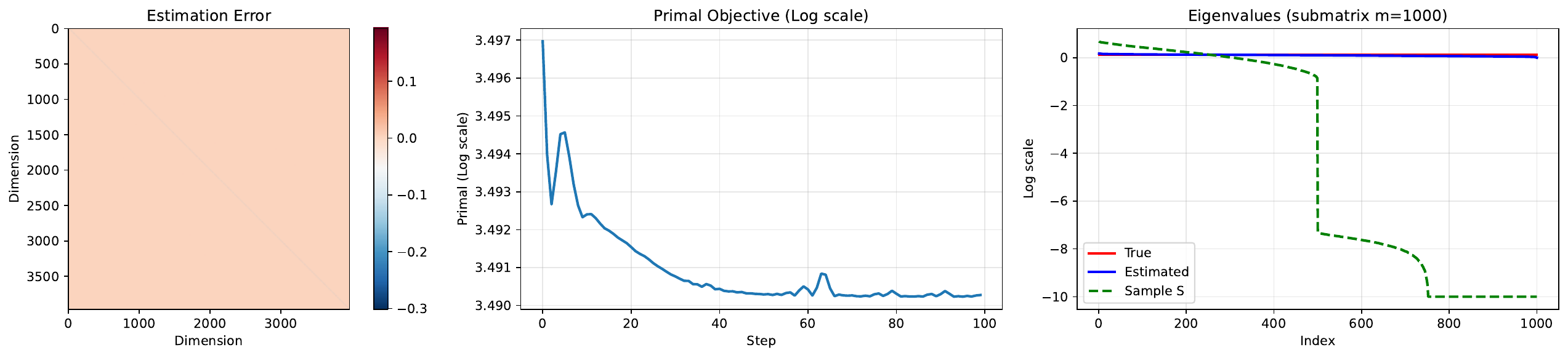}
\end{minipage}
\caption{Visualization of LBO for precision matrix estimation under the Banded 1 structure. The three columns display (from left to right) the visualization of the error matrix, the trajectory of the primal--dual gap versus epochs, and the leading 1000 eigenvalues. The first row corresponds to $p=3025$, while the second row corresponds to $p=3969$.}\label{LBO_band1example}

\end{figure}

For Banded 1, LBO is consistently the fastest method once $p$ becomes moderately large: for example, at $p=2025$ LBO takes $4.268^{1}$ seconds while ADMM and LADMM require $2.045^{3}$ and $4.592^{3}$ seconds, and at $p=3969$ LBO takes $1.491^{2}$ seconds compared with $1.968^{3}$ (ADMM) and $4.395^{3}$ (LADMM).
At the same time, LBO attains the smallest Frobenius and nuclear norms across all dimensions, e.g., at $p=3025$ the Frobenius norm is $1.412^{1}$ for LBO versus $1.936^{1}$ for the remaining methods, and the nuclear norm is $3.452^{2}$ for LBO versus $\approx 8.90^{2}$ for the baselines.

\begin{figure}[H]
\centering 
\begin{minipage}[t]{0.95\textwidth}
\centering
\includegraphics[scale=0.39]{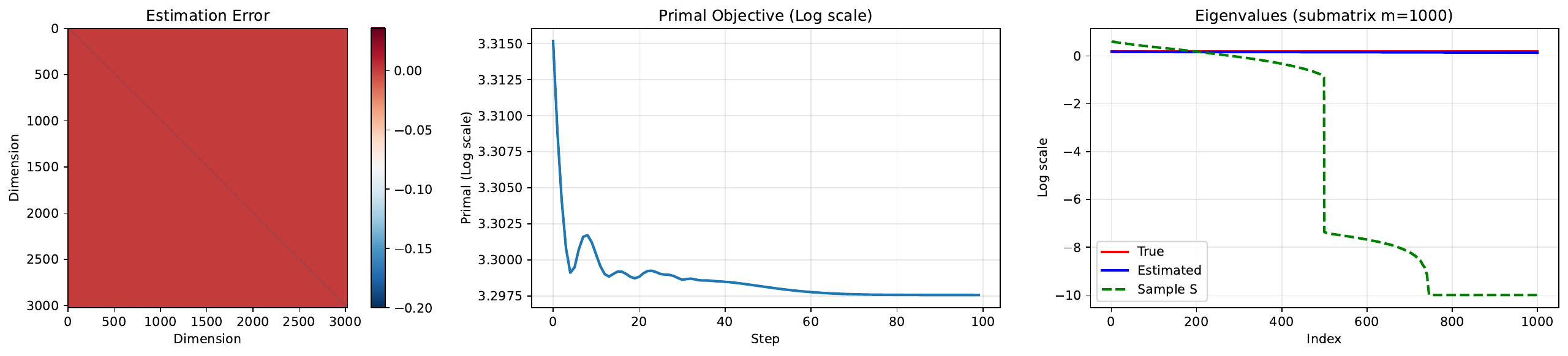}
\end{minipage}

\begin{minipage}[t]{0.95\textwidth}
\centering
\includegraphics[scale=0.39]{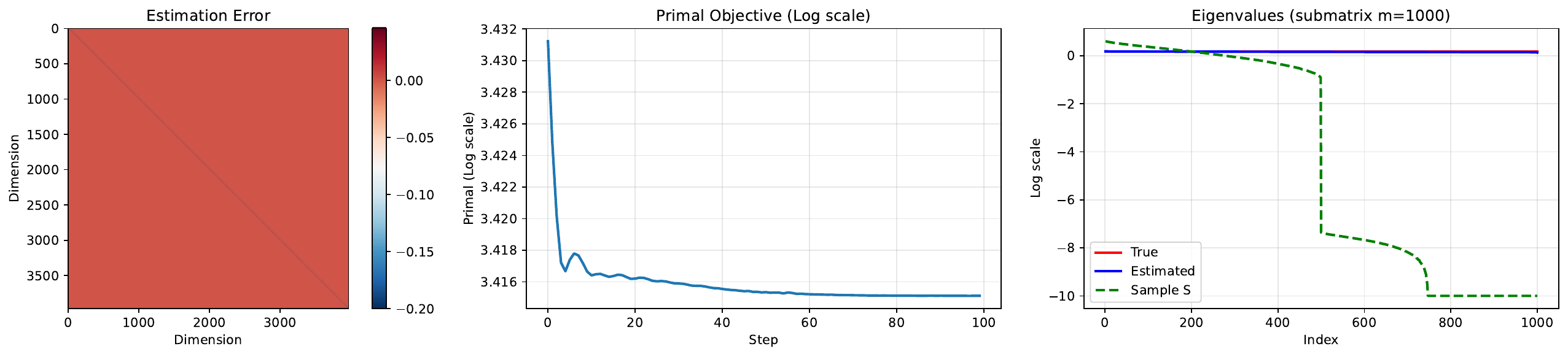}
\end{minipage}
\caption{Visualization of LBO for precision matrix estimation under the Banded 2 structure. The three columns display (from left to right) the visualization of the error matrix, the trajectory of the primal--dual gap versus epochs, and the leading 1000 eigenvalues. The first row corresponds to $p=3025$, while the second row corresponds to $p=3969$.}\label{LBO_band4example}

\end{figure}

For Banded~2, LBO again provides the best overall performance: it is orders of magnitude faster in every tested dimension (e.g., $p=1024$: $1.783^{1}$ for LBO versus $5.804^{2}$ for ADMM and $2.281^{3}$ for LADMM), while also achieving the smallest Frobenius norm throughout (e.g., $p=3025$: $2.182^{1}$ for LBO versus $\approx 2.882^{1}$ for the baselines).
In terms of nuclear norm, LBO is comparable at $p=1024$ ($3.408^{2}$ versus $3.404^{2}$), and becomes clearly favorable as $p$ increases (e.g., $p=3969$: $5.435^{2}$ for LBO versus $\approx 1.60^{3}$ for ADMM/ProxGrad and $\approx 1.48^{3}$ for QUIC).

\begin{figure}[H]
\centering 
\begin{minipage}[t]{0.95\textwidth}
\centering
\includegraphics[scale=0.39]{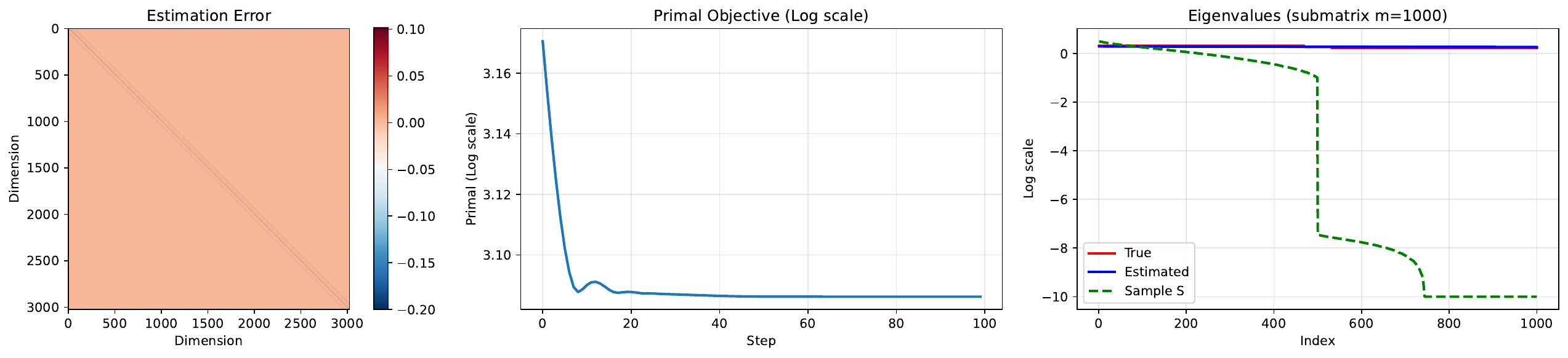}
\end{minipage}

\begin{minipage}[t]{0.95\textwidth}
\centering
\includegraphics[scale=0.39]{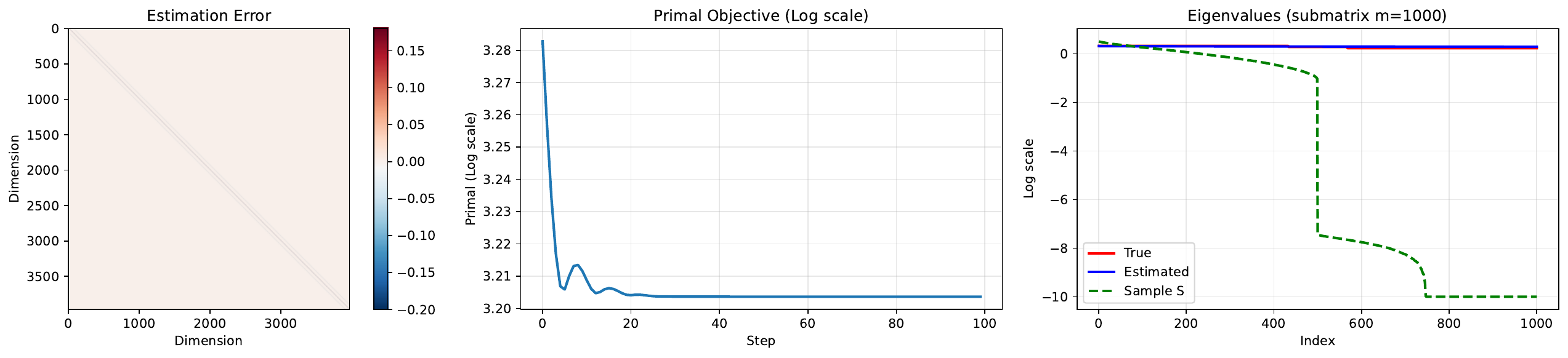}
\end{minipage}
\caption{Visualization of LBO for precision matrix estimation under the Grid structure. The three columns display (from left to right) the visualization of the error matrix, the trajectory of the primal--dual gap versus epochs, and the leading 1000 eigenvalues. The first row corresponds to $p=3025$, while the second row corresponds to $p=3969$.}\label{LBO_gridexample}
\end{figure}

For the grid precision matrices, LBO remains consistently faster than all competitors, particularly in larger dimensions (e.g., $p=1024$: $1.518^{1}$ for LBO versus $1.457^{3}$ for ADMM; $p=3969$: $1.466^{2}$ for LBO versus $2.121^{3}$ for ADMM).
Regarding accuracy, LBO achieves the best or near-best Frobenius norms across the tested dimensions (e.g., $p=1024$: $1.265^{1}$ for LBO compared with $7.322^{0}$ for ADMM), while its nuclear-norm behavior is structure-dependent: LBO is slightly worse at smaller sizes (e.g., $p=1024$: $3.291^{2}$ for LBO versus $1.943^{2}$ for ADMM), but becomes advantageous at larger dimensions (e.g., $p=3969$: $6.828^{2}$ for LBO versus $1.031^{3}$ for ADMM/ProxGrad and $1.049^{3}$ for QUIC).

Across the three precision-matrix structures (two banded cases and the grid case) in Table~\ref{precisiontable}, LBO consistently exhibits superior scalability and competitive estimation quality.
In the two banded settings, LBO is markedly faster as the dimension increases and simultaneously achieves the smallest Frobenius and nuclear norms in most cases. 
For the grid structure, LBO maintains a substantial runtime advantage throughout, while its estimation accuracy remains comparable in Frobenius norm and becomes more favorable in nuclear norm for larger dimensions.
Overall, these results suggest that LBO provides an effective and scalable solver for large-scale precision matrix estimation across diverse sparsity/structure patterns.

\section{Conclusion} \label{sec6}
In this paper, we investigated learning-assisted optimization for high-dimensional matrix estimation by integrating data-driven structures into a principled LADMM framework.
Starting from a baseline LADMM procedure, we introduced learnable parameters and implemented the resulting proximal operators via neural networks, which improved estimation accuracy and accelerated convergence in practice.
On the theoretical side, we established the convergence of LADMM and further proved the convergence, convergence rates, and monotonicity of the reparameterized scheme, showing that the reparameterized LADMM enjoyed a faster convergence rate.
Methodologically, the proposed reparameterization was general and applied to both covariance and precision matrix estimation.

Extensive experiments corroborated these theoretical findings.
Across multiple covariance structures (Toeplitz, factor, sparse, and block) and a range of dimensions, the proposed LBO approach consistently achieved a favorable accuracy--efficiency trade-off relative to classical solvers.
In particular, LBO delivered markedly smaller estimation errors (in Frobenius and, in many regimes, nuclear norms) while substantially reducing runtime, especially in moderate-to-large dimensions where several baselines became computationally expensive.
For precision matrix estimation under banded and grid structures, LBO similarly exhibited strong scalability and competitive accuracy when compared with ADMM-type methods and specialized baselines, supporting the broad applicability of our framework.

In the future, we will extend our framework to joint estimation of multiple matrices (e.g., multi-task or multi-group settings) to better exploit shared structures across related problems.
In addition, we will generalize our methodology from vector-valued observations to matrix-valued and tensor-valued data, developing learning-assisted estimators for structured matrix estimation and higher-order tensor estimation with theoretical guarantees and scalable implementations.


\newpage 
\begin{appendix}
\section{Proofs for Results}\label{appendix}

\subsection{Proof of Theorem \ref{convergence_of_LADMM}}

To proof Theorem \ref{convergence_of_LADMM}, we first show the following Lemma:
\begin{lemma}\label{lem_LADMM}
Denote $(X^{*},Y^{*},V^{*})$ is any KKT point of problem \eqref{unified_opt}, $\phi_1,\phi_2>1$, then we have following assertions\\[1ex]
(1). $\{(\phi_1-1)\|X^{(k)}-X^{*}\|^2+\|Y^{(k)}-Y^{*}\|^2+\rho^{-2}\|V^{(k)}-V^{*}\|^2\}$ is non-increasing. \\[1ex]
(2). $\|X^{(k+1)}-X^{(k)}\|\rightarrow{0}, \|Y^{(k+1)}-Y^{(k)}\|\rightarrow{0}, \|V^{(k+1)}-V^{(k)}\|\rightarrow{0}$.

\end{lemma}
\begin{proof}
We notice that
\begin{equation}\label{eq:lem1_Lyapunov_decrease}
\begin{aligned}
&(\phi_1-1)\|X^{(k+1)}-X^{*}\|^2+\|Y^{(k+1)}-Y^{*}\|^2+\rho^{-2}\|V^{(k+1)}-V^{*}\|^2 \\
= & (\phi_1-1)\|X^{(k)}-X^{*}\|^2+\|Y^{(k)}-Y^{*}\|^2+\rho^{-2}\|V^{(k)}-V^{*}\|^2- (\phi_1-1)\|X^{(k+1)}-X^{(k)}\|^2 \\
&-\left(\rho^{-2}\|V^{(k+1)}-V^{(k)}\|^2+\phi_2\|Y^{(k+1)}-Y^{(k)}\|^2+2\rho^{-1}\left\langle V^{(k+1)}-V^{(k)}, Y^{(k+1)}-Y^{(k)}\right\rangle\right) \\
&- 2\rho^{-1}\left\langle X^{(k+1)}-X^{*}, -\rho(X^{(k)}-Y^{(k)})-V^{(k)}+\rho\phi_1(X^{(k)}-X^{(k+1)})+V^*\right\rangle \\
&- 2\rho^{-1}\left\langle Y^{(k+1)}-Y^{*}, \rho(X^{(k+1)}-Y^{(k)})+V^{(k)}+\rho\phi_2(Y^{(k)}-Y^{(k+1)})-V^*\right\rangle.
\end{aligned}
\end{equation}

Since the optimal conditions of proximal operators in \eqref{unified_LADMM} give
\begin{align*}
&-\rho(X^{(k)}-Y^{(k)})-V^{(k)}+\rho\phi_1(X^{(k)}-X^{(k+1)})\in\partial F(X^{(k+1)}), \quad -V^*\in\partial F(X^*), \\
&\rho(X^{(k+1)}-Y^{(k)})+V^{(k)}+\rho\phi_2(Y^{(k)}-Y^{(k+1)})\in\partial G(Y^{(k+1)}), \quad V^*\in\partial G(Y^*),
\end{align*}
thus by the monotonicity of subgradients, we have 
\begin{align*}
\left\langle X^{(k+1)}-X^{*}, -\rho(X^{(k)}-Y^{(k)})-V^{(k)}+\rho\phi_1(X^{(k)}-X^{(k+1)})+V^*\right\rangle\ge 0, \\
\left\langle Y^{(k+1)}-Y^{*}, \rho(X^{(k+1)}-Y^{(k)})+V^{(k)}+\rho\phi_2(Y^{(k)}-Y^{(k+1)})-V^*\right\rangle\ge0.
\end{align*}
Moreover, $\phi_2>1$ implies that
\begin{align*}
&\rho^{-2}\|V^{(k+1)}-V^{(k)}\|^2+\phi_2\|Y^{(k+1)}-Y^{(k)}\|^2+2\rho^{-1}\left\langle V^{(k+1)}-V^{(k)}, Y^{(k+1)}-Y^{(k)}\right\rangle \\
=& \|\rho^{-1}(V^{(k+1)}-V^{(k)})+Y^{(k+1)}-Y^{(k)}\|^2 + (\phi_2-1)\|Y^{(k+1)}-Y^{(k)}\|^2\ge0,
\end{align*}
hence the first assertion holds. As $\{(\phi_1-1)\|X^{(k)}-X^{*}\|^2+\|Y^{(k)}-Y^{*}\|^2+\rho^{-2}\|V^{(k)}-V^{*}\|^2\}$ is non-increasing and non-negative, it has a finite limit, which implies that the nonnegative sum on the right-hand side of \eqref{eq:lem1_Lyapunov_decrease} that constitutes the one-step decrease converges to zero as $k\to\infty$. 
\begin{align*}
&(\phi_1-1)\|X^{(k+1)}-X^{(k)}\|^2\rightarrow{0}, \quad (\phi_2-1)\|Y^{(k+1)}-Y^{(k)}\|^2\rightarrow{0}, \\
&\|\rho^{-1}(V^{(k+1)}-V^{(k)})+Y^{(k+1)}-Y^{(k)}\|^2\rightarrow{0}, 
\end{align*}
thus the second assertion holds.

\end{proof}

\begin{proof}[Proof of Theorem \ref{convergence_of_LADMM}]
By Lemma \ref{lem_LADMM} (1), the sequence $\{(X^{(k)},Y^{(k)},V^{(k)})\}$ is bounded, thus there exists at least one accumulation point, denoted as $\{(X^{(\infty)},Y^{(\infty)},V^{(\infty)})\}$. Let the subsequence $\{(X^{(k_j)},Y^{(k_j)},V^{(k_j)})\}$ satisfy $(X^{(k_j)},Y^{(k_j)},V^{(k_j)})\rightarrow{(X^{(\infty)},Y^{(\infty)},V^{(\infty)})}$. First, we have \\$\lim_{k\rightarrow{\infty}} \left(X^{(k+1)} - Y^{(k+1)}\right) = \frac{1}{\rho}\lim_{k\rightarrow{\infty}}(V^{(k+1)} - V^{(k)}) = 0$, thus $X^{(\infty)}=Y^{(\infty)}$, i.e., $(X^{(\infty)},Y^{(\infty)},V^{(\infty)})$ is feasible. 

Second, by the definition of subgradient, we obtain
\begin{equation*}
F(X)\ge F(X^{(k_j)}) +
\left\langle X-X^{(k_j)}, -\rho(X^{(k_j-1)}-Y^{(k_j-1)})-V^{(k_j-1)}+\rho\phi_1(X^{(k_j-1)}-X^{(k_j)})\right\rangle.
\end{equation*}
Notice that $\lim_{j\rightarrow{\infty}}(X^{k_j-1},Y^{k_j-1},V^{k_j-1})=(X^{(\infty)},Y^{(\infty)},V^{(\infty)})$ holds, thus 
\begin{equation*}
F(X)\ge F(X^{(\infty)}) +
\left\langle X-X^{(\infty)}, -V^{(\infty)}\right\rangle,
\end{equation*}
which demonstrate that $-V^{(\infty)}\in\partial F(X^{(\infty)})$. Analogously, $V^{(\infty)}\in\partial G(X^{(\infty)})$, and thus
\\ $(X^{(\infty)},Y^{(\infty)},V^{(\infty)})$ is a KKT point of the problem \eqref{unified_opt}.
Thereby, $\{(\phi_1-1)\|X^{(k)}-X^{(\infty)}\|^2+\|Y^{(k)}-Y^{(\infty)}\|+\rho^{-2}\|V^{(k)}-V^{(\infty)}\|\}$ is non-increasing and has a limit, which must be zero since it has a subsequence whose limit is zero. Consequently, we have shown that $(X^{(k)},Y^{(k)},V^{(k)})\rightarrow{(X^{(\infty)},Y^{(\infty)},V^{(\infty)})}$.

\end{proof}

\subsection{Proof of Theorem \ref{convergence_of_neural_LADMM}}
\begin{lemma}
For any KKT point $\omega^*$, there exists proper $(\alpha_k,\beta_k)\in\mathcal{S}$ such that the sequence $\{\omega_k\}$ generated by \eqref{unified_neural_LADMM2} satisfy
\begin{equation}
\langle \omega_{k+1}-\omega^*, H_k(\omega_k-\omega_{k+1})\rangle\ge0, \quad\forall k\ge0.
\end{equation}
Thus we have
\begin{equation}\label{contraction_ineq}
\|\omega_{k}-\omega^*\|_{H_k}^2\ge\|\omega_{k+1}-\omega^*\|_{H_k}^2+\|\omega_{k}-\omega_{k+1}\|_{H_k}^2.
\end{equation}
\end{lemma}
\begin{proof}
The optimality conditions in \eqref{unified_neural_LADMM2} give
\begin{align*}
&\frac{1}{\alpha_k}\circ(X^{(k)}-X^{(k+1)})-\frac{1}{\beta_k}\circ(X^{(k)}-Y^{(k)})-V^{(k)}\in\partial F(X^{(k+1)}), \\
&V^{(k)}-\frac{1}{\beta_k}\circ(Y^{(k+1)}-X^{(k+1)})\in\partial G(Y^{(k+1)}),
\end{align*}
hence by the definition of subgradients,
\begin{equation}\label{lem2_1}
\begin{aligned}
&F(X)-F(X^{(k+1)})+\left\langle X-X^{(k+1)},\frac{1}{\alpha_k}\circ(X^{(k+1)}-X^{(k)})+V^{(k)}+\frac{1}{\beta_k}\circ(X^{(k)}-Y^{(k)})\right\rangle\ge0,\\
& G(Y)-G(Y^{(k+1)})+\left\langle Y-Y^{(k+1)},\frac{1}{\beta_k}\circ(Y^{(k+1)}-X^{(k+1)})-V^{(k)}\right\rangle\ge0.   
\end{aligned}
\end{equation}

Note that $V^{(k+1)}=V^{(k)}+\frac{1}{\beta_k}\circ (X^{(k+1)}-Y^{(k+1)})$, thus by summing up the two inequalities above, we obtain 
\begin{equation}\label{lem2_2}
\begin{aligned}
\langle \omega_{k+1}-\omega, H_k(\omega_k-\omega_{k+1})\rangle\ge & F(X^{(k+1)})+G(Y^{(k+1)})-F(X)-G(Y)+\langle\omega_{k+1}-\omega,\phi(\omega_{k+1})\rangle\\
& + \langle X^{(k+1)}-X-Y^{(k+1)}+Y, \frac{1}{\beta_k}\circ(Y^{(k+1)}-Y^{(k)})\rangle.   
\end{aligned}
\end{equation}

We take $\omega=\omega^*$, then $X^*=Y^*$. Since $G(\cdot)$ is convex and $V^{(k+1)}\in\partial G(Y^{(k+1)})$, we have 
\begin{equation}\label{lem2_3}
\langle X^{(k+1)}-Y^{(k+1)}, \frac{1}{\beta_k}\circ(Y^{(k+1)}-Y^{(k)})\rangle = \langle V^{(k+1)}-V^{(k)}, Y^{(k+1)}-Y^{(k)}\rangle \ge 0.
\end{equation}
On the other hand, 
\begin{equation}\label{lem2_4}
\begin{aligned}
&F(X^{(k+1)})+G(Y^{(k+1)})-F(X^*)-G(Y^*)+\langle\omega_{k+1}-\omega^*,\phi(\omega_{k+1})\rangle \\
&=F(X^{(k+1)})+G(Y^{(k+1)})-F(X^*)-G(Y^*)+\langle\omega_{k+1}-\omega^*,\phi(\omega^*)\rangle\ge0,
\end{aligned}
\end{equation}
here the first equality holds since $\langle\omega_{k+1}-\omega^*,\phi(\omega^*)-\phi(\omega_{k+1})\rangle=0.$ and the second inequality holds by the optimal condition of KKT point.

Combining \eqref{lem2_2}, \eqref{lem2_3} and \eqref{lem2_4}, we can conclude that 
\begin{equation*}
\langle \omega_{k+1}-\omega^*, H_k(\omega_k-\omega_{k+1})\rangle\ge0.
\end{equation*}
Consequently, we obtain
\begin{align*}
\|\omega_{k}-\omega^*\|_{H_k}^2&=\|\omega_{k}-\omega_{k+1}+\omega_{k+1}-\omega^*\|_{H_k}^2\\
&=\|\omega_{k+1}-\omega^*\|_{H_k}^2+\|\omega_{k}-\omega_{k+1}\|_{H_k}^2 +2\langle \omega_{k+1}-\omega^*, H_k(\omega_k-\omega_{k+1})\rangle \\
&\ge\|\omega_{k+1}-\omega^*\|_{H_k}^2+\|\omega_{k}-\omega_{k+1}\|_{H_k}^2. 
\end{align*} 
\end{proof}

\begin{proof}[Proof of Theorem \ref{convergence_of_neural_LADMM}]
Given some $\omega^*\in\Omega^*$, by the inequality \eqref{contraction_ineq}, we have
\begin{equation*}
\|\omega_{k}-\omega^*\|_{H_k}^2\ge\|\omega_{k+1}-\omega^*\|_{H_k}^2+\|\omega_{k}-\omega_{k+1}\|_{H_k}^2\ge \|\omega_{k+1}-\omega^*\|_{H_{k+1}}^2+\|\omega_{k+1}-\omega^*\|_{(H_k-H_{k+1})}^2.
\end{equation*}
We can take the appropriate parameters $(\alpha_k,\beta_k)\rightarrow{(\alpha^*,\beta^*)}$ such that  and $\|H_{k+1}-H_k\|\le c_0/(k+1)^2\|H_{k+1}\|$, where $c_0\ge0$ is small enough. Then $\exists c_1,c_2 >0,c_1\le\|H_k\|\le c_2$, and
\begin{equation*}
\|\omega_{k}-\omega^*\|_{H_k}^2\ge \left(1-\frac{c_0}{(k+1)^2}\right)\|\omega_{k+1}-\omega^*\|_{H_{k+1}}^2.
\end{equation*}
Thus for any $k\ge0$, we obtain
\begin{equation*}
\|\omega_{0}-\omega^*\|_{H_0}^2\ge \left(\prod_{u=0}^{k}\left(1-\frac{c_0}{(u+1)^2}\right)\right)\|\omega_{k+1}-\omega^*\|_{H_{k+1}}^2.    
\end{equation*}
Take small enough $c_0$, such that $\prod_{u=0}^{\infty}\left(1-\frac{c_0}{(u+1)^2}\right)=\exp\left(\sum_{u=0}^\infty\log\left(1-\frac{c_0}{(u+1)^2}\right)\right)>0$. Then we can conclude that $\{\omega_k\}$ is a bounded sequence. By the inequality \eqref{contraction_ineq}, we have
\begin{align*}
\sum_{u=0}^{k}\|\omega_{u}-\omega_{u+1}\|_{H_u}^2
&\le\sum_{u=0}^{k}(\|\omega_{u}-\omega^*\|_{H_u}^2-\|\omega_{u+1}-\omega^*\|_{H_u}^2)\\
&\le \sum_{u=0}^{k}\left(\|\omega_{u}-\omega^*\|_{H_u}^2-\|\omega_{u+1}-\omega^*\|_{H_{u+1}}^2
+\|\omega_{u+1}-\omega^*\|_{(H_{u+1}-H_u)}^2\right) \\
&=\|\omega_{0}-\omega^*\|_{H_0}^2-\|\omega_{k+1}-\omega^*\|_{H_{k+1}}^2+\sum_{u=0}^{k}\|\omega_{u+1}-\omega^*\|_{(H_{u+1}-H_u)}^2\\
&\le \|\omega_{0}-\omega^*\|_{H_0}^2+\sum_{u=0}^{k}|\|\omega_{u+1}-\omega^*\|_{(H_{u+1}-H_u)}^2|.
\end{align*}

Let $k\rightarrow{\infty}$, then $\sum_{k=0}^{\infty}\|\omega_{k}-\omega_{k+1}\|_{H_k}^2\le \|\omega_{0}-\omega^*\|_{H_0}^2 + \sum_{k=0}^{\infty} |\|\omega_{k+1}-\omega^*\|_{(H_{k+1}-H_k)}^2|.$ Since $\|H_{k+1}-H_k\|\le O(1/(k+1)^2)$, we have $\sum_{k=0}^{\infty}|\|\omega_{k+1}-\omega^*\|_{(H_{k+1}-H_k)}^2|<\infty$ and thus $\sum_{k=0}^{\infty}\|\omega_{k}-\omega_{k+1}\|_{H_k}^2<\infty$. It follows that $\|\omega_{k}-\omega_{k+1}\|_{H_k}^2\rightarrow{0}$.
Denote $\{\omega_{k_j}\}$ as the subsequence that satisfies $\omega_{k_j}\rightarrow{\omega_\infty}$. By \eqref{lem2_1}, if we take $k=k_j$ and note that $\|\omega_{k}-\omega_{k+1}\|_{H_k}^2\rightarrow{0}$ gives $\|\omega_{k}-\omega_{k+1}\|^2_{F}\rightarrow{0}$, we have
\begin{equation}
\begin{aligned}
&F(X)-F(X^{(\infty)})+\left\langle X-X^{(\infty)},V^{(\infty)}\right\rangle\ge0,\\
& G(Y)-G(Y^{(\infty)})+\left\langle Y-Y^{(\infty)},-V^{(\infty)}\right\rangle\ge0.   
\end{aligned}
\end{equation}
Thus $\omega_\infty$ is a KKT point of the problem \eqref{unified_opt}. Use the inequality \eqref{contraction_ineq} again, then $\forall j\ge 0,\forall k>k_j$,
\begin{align*}
\|\omega_{k}-\omega_\infty\|_{H_{k}}^2\le \|\omega_{k_j}-\omega_\infty\|_{H_0}^2+\sum_{u=k_j}^{k-1} \|\omega_{u+1}-\omega_\infty\|_{(H_{u+1}-H_u)}^2-\sum_{u=k_j}^{k-1} \|\omega_{u}-\omega_{u+1}\|_{H_u}^2.
\end{align*}
By the construction of $H_k$, $\sum_{u=0}^{\infty} |\|\omega_{u+1}-\omega_\infty\|_{(H_{u+1}-H_u)}^2|<\infty$ and $\sum_{u=0}^{\infty}\|\omega_{u}-\omega_{u+1}\|_{H_u}^2<\infty$, thus the canonical $\epsilon-\delta$ argument gives the desired result of $\|\omega_{k}-\omega_\infty\|_{F}\rightarrow{0}$.
\end{proof}

\subsection{Proof of Theorem \ref{monotonicity_of_neural_LADMM}}
\begin{proof}
We assume that there exists $\alpha_k,\beta_k$ such that $\omega_{k+1}\neq\omega_k$, otherwise we have completed the proof by Lemma \ref{absorbing_states}. Hence we have $\|\omega_k-\omega_{k+1}\|_{H_k}\neq0$, and $\exists \rho_k>0$, such that
\begin{equation*}
\operatorname{dist}^2_{H_k}(\omega_{k+1},\Omega^*)\le\rho_k \|\omega_k-\omega_{k+1}\|_{H_k}^2.   
\end{equation*}
Also, the inequality \eqref{contraction_ineq} gives
\begin{equation*}
\operatorname{dist}^2_{H_k}(\omega_{k+1},\Omega^*)\le \operatorname{dist}^2_{H_k}(\omega_{k},\Omega^*) -\|\omega_k-\omega_{k+1}\|_{H_k}^2.     
\end{equation*}
Thus we obtain
\begin{equation*}
\operatorname{dist}^2_{H_k}(\omega_{k+1},\Omega^*)\le (1+1/\rho_k)^{-1} \operatorname{dist}^2_{H_k}(\omega_{k},\Omega^*).   
\end{equation*}
For $\omega^*\in\Omega^*$ such that $\operatorname{dist}^2_{H_k}(\omega_{k+1},\Omega^*)=\|\omega_{k+1}-\omega^*\|^2_{H_k}$, then
\begin{equation*}
\operatorname{dist}^2_{H_{k+1}}(\omega_{k+1},\Omega^*)\le \|\omega_{k+1}-\omega^*\|^2_{H_k} + \|\omega_{k+1}-\omega^*\|^2_{H_{k+1}-H_k}. 
\end{equation*}
Choose $\alpha_k,\beta_k$ accordingly such that $\|H_{k+1}-H_k\|\le \tau_k\|H_k\|$, where $\tau_k$ is small enough, then $(1+\tau_k)(1+1/\rho_k)^{-1}<1$, and thus
\begin{equation*}
\operatorname{dist}^2_{H_{k+1}}(\omega_{k+1},\Omega^*)\le (1+\tau_k)(1+1/\rho_k)^{-1}\operatorname{dist}^2_{H_k}(\omega_{k},\Omega^*) < \operatorname{dist}^2_{H_k}(\omega_{k},\Omega^*), 
\end{equation*}
which completes the proof.
\end{proof}

\begin{lemma}[Absorbing states]\label{absorbing_states}
Given $\omega_k$ generated by \eqref{unified_neural_LADMM2}. If for all parameters $(\alpha_k,\beta_k)\in\mathcal{S}$, there  holds $\omega_{k+1}=\omega_k$, then $\forall j\ge k,\omega_j=\omega_k\in\Omega^*$.
\end{lemma}
\begin{proof}[Proof of Lemma \ref{absorbing_states}]
The update rule of $V^{(k+1)}$ directly implies that $X^{(k+1)}=Y^{(k+1)}$. Using $\omega_{k+1}=\omega_k$ and the optimal conditions in \eqref{unified_neural_LADMM2}, we have $-V^{(k)}\in\partial F(X^{(k)})$ and $V^{(k)}\in\partial G(Y^{(k)})$, thus $\omega_k\in\Omega^*$ is a KKT point. We can show that $\omega_k$ is a fixed point by induction.
\end{proof}

\subsection{Proof of Theorem \ref{convergence_rate_of_neural_LADMM}}
\begin{proof}
By Lemma \ref{absorbing_states}, we assume without loss of generality that $\forall k\ge0,\omega_{k+1}\neq\omega_k$. There exists $K_0>0$, such that for $k\ge K_0$, we have $\operatorname{dist}^2_{H^*}(\widetilde{\omega}_{k+1},\Omega^*)\le\frac{\kappa}{16}\|\widetilde{\omega}_{k+1}-\omega_k\|_{H^*}^2$. Take $\omega^*\in\Omega^*$ such that $\operatorname{dist}^2_{H^*}(\widetilde{\omega}_{k+1},\Omega^*)=\|\widetilde{\omega}_{k+1}-\omega^*\|^2_{H^*}$, then 
\begin{equation*}  
\operatorname{dist}^2_{H_{k+1}}(\widetilde{\omega}_{k+1},\Omega^*)\le \operatorname{dist}^2_{H^*}(\widetilde{\omega}_{k+1},\Omega^*) + \big|\|\widetilde{\omega}_{k+1}-\omega^*\|^2_{H_{k+1}-H^*}\big|.
\end{equation*}
There exists $(\alpha_k,\beta_k)\rightarrow{(\alpha^*,\beta^*)}$ with some decay rate, i.e., for large enough $k$, $\|H_{k+1}-H^*\|\le c_k\|H_{k}-H^*\|$, where $0\le c_k<1$ is to be determined. Also, we require that $\beta_k=\beta^*$, $\alpha_k\uparrow{\alpha^*}$, then $H_{k+1}\preceq H_k$. 
Then 
\begin{equation*} 
\operatorname{dist}^2_{H_{k+1}}(\widetilde{\omega}_{k+1},\Omega^*) < 2\operatorname{dist}^2_{H^*}(\widetilde{\omega}_{k+1},\Omega^*), \quad
\|\widetilde{\omega}_{k+1}-\omega_k\|^2_{H^*}<2\|\widetilde{\omega}_{k+1}-\omega_k\|^2_{H_k}.
\end{equation*}
Thus we obtain 
\begin{equation*} 
\operatorname{dist}^2_{H_{k+1}}(\widetilde{\omega}_{k+1},\Omega^*) < \frac{\kappa}{4} \|\widetilde{\omega}_{k+1}-\omega_k\|^2_{H_k}.
\end{equation*}
By the triangle inequality under $\|\cdot\|_{H}$, we have 
\begin{align*} 
\operatorname{dist}_{H_{k+1}}({\omega}_{k+1},\Omega^*)&\le \|\omega_{k+1}-\widetilde{\omega}_{k+1}\|_{H_{k+1}}+\operatorname{dist}_{H_{k+1}}(\widetilde{\omega}_{k+1},\Omega^*) \\ 
& < \|\omega_{k+1}-\widetilde{\omega}_{k+1}\|_{H_{k+1}}+\frac{\sqrt{\kappa}}{2}\|\widetilde{\omega}_{k+1}-\omega_k\|_{H_k} \\ 
& \le \|\omega_{k+1}-\widetilde{\omega}_{k+1}\|_{H_{k+1}}+\frac{\sqrt{\kappa}}{2}\big(\|\widetilde{\omega}_{k+1}-\omega_{k+1}\|_{H_k}+\|{\omega}_{k+1}-\omega_{k}\|_{H_k}\big).
\end{align*}
Since the operator $\mathcal{T}$ is lipschitz continuous
w.r.t. $\alpha_k,\beta_k$, thus 
\begin{equation*} 
\|\omega_{k+1}-\widetilde{\omega}_{k+1}\|_{H_{k+1}}+\frac{\sqrt{\kappa}}{2}\|\widetilde{\omega}_{k+1}-\omega_{k+1}\|_{H_k} \le C(\omega)O(\Delta_k),
\end{equation*}
where $\Delta_k= |\alpha_k-\alpha^*|+|\beta_k-\beta^*|$ and $C(\omega)$ is the upper bound of $\{\omega_k\}$. Let $c_k$ decay fast sufficiently, we obtain $C(\omega)O(\Delta_k)\le \frac{1}{2} \operatorname{dist}_{H_{k+1}}({\omega}_{k+1},\Omega^*)$ (There exists some parameters such that this holds, i.e., let $(\alpha_k,\beta_k)=(\alpha^*,\beta^*)$, which degenerates to LADMM). Thus 
\begin{equation*} 
\|\omega_{k+1}-\widetilde{\omega}_{k+1}\|_{H_{k+1}}+\frac{\sqrt{\kappa}}{2}\|\widetilde{\omega}_{k+1}-\omega_{k+1}\|_{H_k} \le \frac{1}{2} \operatorname{dist}_{H_{k+1}}({\omega}_{k+1},\Omega^*).
\end{equation*}
Hence we have for $k\ge K_0$, 
\begin{equation*} 
\operatorname{dist}^2_{H_{k+1}}({\omega}_{k+1},\Omega^*) \le \kappa\|{\omega}_{k+1}-\omega_{k}\|_{H_k}^2.
\end{equation*}
Since $\|\omega_{k+1}-\omega_k\|_{H_k} >0$ for all $k$, there exists $\widetilde{\kappa}>0$ such that $\forall k<K_0$, we have $\operatorname{dist}^2_{H_{k+1}}({\omega}_{k+1},\Omega^*) \le \widetilde{\kappa}\|{\omega}_{k+1}-\omega_{k}\|_{H_k}$. Denote $\rho_0=\max\{\kappa,\widetilde{\kappa}\}$, then for all $k$, 
\begin{equation*} 
\operatorname{dist}^2_{H_{k+1}}({\omega}_{k+1},\Omega^*) \le \rho_0\|{\omega}_{k+1}-\omega_{k}\|_{H_k}^2.
\end{equation*}
Since $H_{k+1}\preceq H_k$, the proof process of Theorem \ref{monotonicity_of_neural_LADMM} implies that 
\begin{align*} 
&\operatorname{dist}^2_{H_{k+1}}(\omega_{k+1},\Omega^*)\le \operatorname{dist}^2_{H_k}(\omega_{k+1},\Omega^*) + \|\omega_{k+1}-\omega^*\|^2_{H_{k+1}-H_k}\le \operatorname{dist}^2_{H_k}(\omega_{k+1},\Omega^*), \\ 
&\operatorname{dist}^2_{H_k}(\omega_{k+1},\Omega^*)\le \operatorname{dist}^2_{H_k}(\omega_{k},\Omega^*) -\|\omega_k-\omega_{k+1}\|_{H_k}^2.
\end{align*}
Hence we conclude that 
\begin{equation*} 
\operatorname{dist}^2_{H_{k+1}}({\omega}_{k+1},\Omega^*) \le (1+1/\rho_0)^{-1}\operatorname{dist}^2_{H_{k}}({\omega}_{k},\Omega^*).
\end{equation*}

\end{proof}

\subsection{Proof of Theorem \ref{superiority_of_neural_LADMM}}
\begin{proof}
Take $(w_1)_k=\frac{1}{\rho\phi_1}\mathbf{1},(w_2)_k=\frac{1}{\rho\phi_2}\mathbf{1},\alpha_k=\frac{1}{\rho\phi_1}\mathbf{1},\gamma_k=\rho\mathbf{1}$. Let $\widehat{\omega}_{k+1}=(\widehat{X}^{(k+1)},\widehat{Y}^{(k+1)},\widehat{V}^{(k+1)})^\top,$
$\widetilde\omega_{k+1}=(\widetilde{X}^{(k+1)},\widetilde{Y}^{(k+1)},\widetilde{V}^{(k+1)})^\top$. Then the re-parameterized LADMM update can be rewritten as
\begin{equation}
\left\{
\begin{aligned}
\widehat{X}^{(k+1)} &= \operatorname{prox}_{F/(\rho\phi_1)}\!\left(
X^{(k)} - \frac{1}{\phi_1}\big(X^{(k)} - Y^{(k)} + V^{(k)}/\rho\big)
\right),\\[6pt]
\widehat{Y}^{(k+1)} &= \operatorname{prox}_{G/(\rho\phi_2)}\!\left(Y^{(k)}+\beta_k \circ \left(V^{(k)}+\rho\big(\widehat{X}^{(k+1)} - Y^{(k)}\big)\right)
\right),\\[6pt]
\widehat{V}^{(k+1)} &= V^{(k)} + \rho\big(\widehat X^{(k+1)} - \widehat{Y}^{(k+1)}\big).
\end{aligned}
\right.
\end{equation}
Denote $\eta(\cdot)=\operatorname{prox}_{F/(\rho\phi_1)}(\cdot)$ , $\zeta(\cdot)=\operatorname{prox}_{G/(\rho\phi_2)}(\cdot)$, $\beta^*=\frac{1}{\rho\phi_2}\mathbf{1}$ and $R^{(k)}=V^{(k)}+\rho\big(\widehat{X}^{(k+1)} - Y^{(k)}\big)$. We claim that if $R^{(k)}$ satisfies that $(R^{(k)})_{ij}\neq0, \forall i,j$, then there exists $c\in(0,1)$ and $\beta_k$ such that
\begin{equation}\label{superiority_equality}
\zeta\left(Y^{(k)}+\beta_k\circ R^{(k)}\right)-\widetilde{Y}_{k+1}=-\frac{c}{1+\rho^2}\left(\widetilde{Y}_{k+1}-Y^*-\rho\left(\widetilde{V}_{k+1}-V^*\right)\right).
\end{equation}
Indeed, since $\zeta$ is bijective, the above equality can be transformed as 
\begin{equation*}
\beta_k\circ R^{(k)}=\zeta^{-1}\left(-\frac{c}{1+\rho^2}\left(\widetilde{Y}_{k+1}-Y^*-\rho\left(\widetilde{V}_{k+1}-V^*\right)\right)+\widetilde{Y}_{k+1}\right)-Y^{(k)}.
\end{equation*}
As we take a small enough $c>0$, we can conclude that there exists $\beta_k$  such that \eqref{superiority_equality} holds. Moreover, define $\Delta Y=\zeta\left(Y^{(k)}+\beta_k\circ R^{(k)}\right)-\widetilde{Y}_{k+1}$, if $\widetilde{Y}_{k+1}-Y^*-\rho\left(\widetilde{V}_{k+1}-V^*\right)\neq0$, then 
\begin{align*}
&\|\widetilde\omega_{k+1}-\omega^*\|_F^2-\|\widehat{\omega}_{k+1}-\omega^*\|_F^2\\&= \|\widetilde{Y}_{k+1}-Y^*\|_F^2-\|\widetilde{Y}_{k+1}-Y^*+\Delta Y\|_F^2+\|\widetilde{V}_{k+1}-V^*\|_F^2-\|\widetilde{V}_{k+1}-V^*-\rho\Delta Y\|_F^2 \\
&= \frac{2c-c^2}{1+\rho^2}\|\widetilde{Y}_{k+1}-Y^*-\rho\left(\widetilde{V}_{k+1}-V^*\right)\|_F^2 >0.
\end{align*}
Note that $\Omega_0=\{\omega_k:(R^{(k)})_{ij}=0\}\bigcup\{\omega_k:\widetilde{Y}_{k+1}-Y^*-\rho\left(\widetilde{V}_{k+1}-V^*\right)=0\}$ is a set of measure zero, we have completed the proof.
\end{proof}

\subsection{Proof of Theorem \ref{thm:entry}}
Before proving Theorem \ref{thm:entry}, we first present two fundamental facts in high-dimensional statistics: (i) If \(U,V\) are sub-Gaussian, then \(UV\) is sub-exponential and \(\normpsione{UV}\le C\,\normpsitwo{U}\,\normpsitwo{V}\); hence \(\normpsione{UV-\E(UV)}\le C\,\normpsitwo{U}\,\normpsitwo{V}\). (ii) Bernstein inequality for sub-exponential variables. If \(\widetilde{X}_1,\dots,\widetilde{X}_n\) are i.i.d. centered with
\(\lVert \widetilde{X}_k \rVert_{\psi_1}\le B\), then for all \(t>0\),
\[
\Pr\! \left(\left|\frac{1}{n}\sum_{k=1}^n \widetilde{X}_k\right| \ge t\right)
\;\le\; 2\exp\! \left(-c n \min \left(\frac{t^2}{B^2},\,\frac{t}{B}\right)\right).
\]
\begin{proof}
Fix indices $i,j\in\{1,\dots,p\}$ and define
\[
(S-\Sigma^\star)_{ij}
= \frac{1}{n}\sum_{k=1}^n \big(\widetilde{X}_{ki}\widetilde{X}_{kj}-\E[\widetilde{X}_i \widetilde{X}_j]\big)
= \frac{1}{n}\sum_{k=1}^n \widetilde{Z}_k^{(ij)},
\]
where $\widetilde{Z}_k^{(ij)} := \widetilde{X}_{ki}\widetilde{X}_{kj}-\E[\widetilde{X}_i \widetilde{X}_j]$ are i.i.d.\ with mean zero.
By the product-sub-exponential fact and Assumption~\ref{asmp:subg}, we have
\[
\normpsione{Z_k^{(ij)}} \;\le\; C\,\normpsitwo{X_{k,i}}\,\normpsitwo{X_{k,j}}
\;\le\; C K^2 \;=: B.
\]
Applying the sub-exponential Bernstein inequality yields, for any $t>0$,
\begin{equation}\label{eq:fixed-pair}
\Pr\!\left(\big|(S-\Sigma^\star)_{ij}\big|\ge t\right)
\;\le\; 2 \exp\! \left(-c n \min \left(\frac{t^2}{B^2},\,\frac{t}{B}\right)\right).
\end{equation}
This holds for all pairs $(i,j)$, including the diagonal $i=j$ (since $X_i^2-\E X_i^2$ is sub-exponential with the same $\psione$-norm bound). Now take the union bound over all $p^2$ entries:
\[
\Pr\!\left(\|S-\Sigma^\star\|_\infty \ge t\right)
\;\le\; \sum_{i,j=1}^p \Pr\!\left(\big|(S-\Sigma^\star)_{ij}\big|\ge t\right)
\;\le\; 2 p^2 \exp\!\Big(-c n \min\Big\{\frac{t^2}{B^2},\,\frac{t}{B}\Big\}\Big).
\]
Choose $t = C_0 \sqrt{\frac{\log p}{n}}$ with $C_0 = \widetilde C\,B$.
Provided $n\gtrsim \log p$, we are in the quadratic branch of the minimum, hence
\[
\Pr\!\left(\|S-\Sigma^\star\|_\infty \ge C_0 \sqrt{\frac{\log p}{n}}\right)
\;\le\; 2 p^2 \exp\!\Big(-c n \cdot \frac{C_0^2 \log p}{n B^2}\Big)
\;=\; 2 p^{\,2 - c\,\widetilde C^{\,2}}.
\]
Take $\widetilde C$ large enough so that $2 - c\widetilde C^2 \le -c_0$ for the desired $c_0>0$.
Equivalently,
\[
\Pr\!\left(\|S-\Sigma^\star\|_\infty \le C_0 \sqrt{\tfrac{\log p}{n}}\right)
\;\ge\; 1 - 2 p^{-c_0}.
\]
This completes the proof.
\end{proof}

\subsection{Proof of Theorem \ref{thm:total}}
\begin{proof}
Recall
\(
f_1(\Sigma)=\frac{1}{2}\|\Sigma-S\|_F^2+\lambda\|\Sigma\|_{1,\mathrm{off}}
\)
and let
\(
\widehat\Sigma_\star\in\arg\min_{\Sigma\succeq \epsilon I} f_1(\Sigma)
\).
Set \(\Delta:=\widehat\Sigma_\star-\Sigma^\star\) and denote the off–diagonal support
\(\gA:=\{(i,j):i\neq j,\ \sigma^\star_{ij}\neq 0\}\) with \(|\gA|\) elements. By optimality of \(\widehat\Sigma_\star\) for \eqref{raw_high_dimension}, for any feasible \(\Sigma\),
\[
\frac{1}{2}\|\widehat\Sigma_\star-S\|_F^2+\lambda\|\widehat\Sigma_\star\|_{1,\mathrm{off}}
\ \le\
\frac{1}{2}\|\Sigma-S\|_F^2+\lambda\|\Sigma\|_{1,\mathrm{off}}.
\]
Choosing \(\Sigma=\Sigma^\star\) and expanding
\(\|\widehat\Sigma_\star-S\|_F^2-\|\Sigma^\star-S\|_F^2
= \|\Delta\|_F^2 + 2\langle \Sigma^\star-S,\Delta\rangle\),
we obtain the basic inequality
\begin{equation}\label{eq:A-basic}
\frac{1}{2}\|\Delta\|_F^2
\ \le\
\langle S-\Sigma^\star,\Delta\rangle
+\lambda\big(\|\Sigma^\star\|_{1,\mathrm{off}}-\|\widehat\Sigma_\star\|_{1,\mathrm{off}}\big).
\end{equation}

Decompose \(\Delta=\Delta_{\gA}+\Delta_{\gA^c}\) on off–diagonals
(diagonals are unpenalized). If \(\lambda\ge 2\|S-\Sigma^\star\|_\infty\),
then by Hölder inequality,
\(
\langle S-\Sigma^\star,\Delta\rangle
\le \|S-\Sigma^\star\|_\infty \|\Delta\|_1
\le \frac{\lambda}{2}\|\Delta\|_1
\).
By decomposability of \(\|\cdot\|_{1,\mathrm{off}}\),
\(
\|\Sigma^\star\|_{1,\mathrm{off}}-\|\widehat\Sigma_\star\|_{1,\mathrm{off}}
\le \|\Delta_{\gA}\|_1-\|\Delta_{\gA^c}\|_1 + 2\|(\Sigma^\star)_{\gA^c}\|_1
\).
Plugging these into \eqref{eq:A-basic} and dropping the nonnegative
\(\frac{1}{2}\|\Delta\|_F^2\) yields the cone inequality
\begin{equation}\label{eq:A-cone}
\|\Delta_{\gA^c}\|_1 \ \le\ 3\|\Delta_{\gA}\|_1 + 4\|(\Sigma^\star)_{\gA^c}\|_1.
\end{equation}

Returning to \eqref{eq:A-basic} and using the bounds above,
\[
\begin{aligned}
\tfrac12\|\Delta\|_F^2
\ &\le\
\tfrac{\lambda}{2}\big(\|\Delta_{\gA}\|_1+\|\Delta_{\gA^c}\|_1\big)
+\lambda\big(\|\Delta_{\gA}\|_1-\|\Delta_{\gA^c}\|_1+2\|(\Sigma^\star)_{\gA^c}\|_1\big) \\
& s'd= \tfrac{3\lambda}{2}\|\Delta_{\gA}\|_1 - \tfrac{\lambda}{2}\|\Delta_{\gA^c}\|_1
+ 2\lambda\|(\Sigma^\star)_{\gA^c}\|_1.
\end{aligned}
\]
Drop the negative term and use \(\|\Delta_{\gA}\|_1\le \sqrt{|\gA|}\,\|\Delta\|_F\) to obtain
\[
\tfrac12\|\Delta\|_F^2 \ \le\ \frac{3\lambda}{2}\sqrt{|\gA|}\,\|\Delta\|_F
\ +\ 2\lambda\|(\Sigma^\star)_{\gA^c}\|_1.
\]
Solving the quadratic inequality in \(x=\|\Delta\|_F\) gives
\begin{equation}\label{eq:A-frob}
\|\widehat\Sigma_\star-\Sigma^\star\|_F
\ \le\
4\sqrt{|\gA|}\,\lambda \;+\; 4\,\frac{\|(\Sigma^\star)_{\gA^c}\|_1}{\sqrt{|\gA|}}.
\end{equation}

From the KKT conditions for \eqref{raw_high_dimension},
\(0\in (\widehat\Sigma_\star-S)+\lambda Z + \gI_{\{\Sigma\succeq \epsilon I\}}(\widehat\Sigma_\star)\),
with \(Z\in\partial \|\cdot\|_{1,\mathrm{off}}(\widehat\Sigma_\star)\).
On off–diagonals, the normal cone vanishes, hence
\[
|(\widehat\Sigma_\star-\Sigma^\star)_{ij}|
\le |(S-\Sigma^\star)_{ij}| + \lambda |Z_{ij}|
\le \tfrac{\lambda}{2} + \lambda \le 3\lambda,
\qquad (i\neq j),
\]
using \(\lambda\ge 2\|S-\Sigma^\star\|_\infty\) and \(|Z_{ij}|\le 1\). Write \(f_1(\Sigma)\) with \(\frac{1}{2}\|\Sigma-S\|_F^2\) (1–strongly convex)
and \(\lambda\|\Sigma\|_{1,\mathrm{off}}+\gI_{\{\Sigma\succeq \epsilon I\}}(\Sigma)\) (convex).
Hence \(f_1(\Sigma)\) is 1–strongly convex and, for any \(\Sigma\),
\[
f_1(\Sigma)-f_1(\widehat\Sigma_\star)\ \ge\ \frac{1}{2}\|\Sigma-\widehat\Sigma_\star\|_F^2.
\]
For the iterate \(\Sigma^{(k)}\), this yields
\begin{equation}\label{eq:A-opt-to-param}
\|\Sigma^{(k)}-\widehat\Sigma_\star\|_F \ \le\ \sqrt{2\,\varepsilon_{\mathrm{opt}}(k)}.
\end{equation}

By the triangle inequality and \eqref{eq:A-opt-to-param},
\[
\|\Sigma^{(k)}-\Sigma^\star\|_F
\ \le\
\|\Sigma^{(k)}-\widehat\Sigma_\star\|_F
+\|\widehat\Sigma_\star-\Sigma^\star\|_F
\ \le\
\sqrt{2\,\varepsilon_{\mathrm{opt}}(k)}
\;+\;
\Big(4\sqrt{|\gA|}\,\lambda + 4\,\|(\Sigma^\star)_{\gA^c}\|_1/\sqrt{|\gA|}\Big),
\]
which is exactly the claimed bound \eqref{eq:total_bound}.
Under Assumption~\ref{asmp:subg}, Theorem~\ref{thm:entry} ensures
\(\|S-\Sigma^\star\|_\infty \lesssim \sqrt{\tfrac{\log p}{n}}\) with probability at least \(1-2p^{-c_0}\),
so choosing \(\lambda \asymp \sqrt{\tfrac{\log p}{n}}\) makes the statistical term optimal (up to constants).
If we stop when \(\varepsilon_{\mathrm{opt}}(k)\le C^2\,|\gA|\,\tfrac{\log p}{n}\),
the first term matches the statistical rate, yielding
\(
\|\Sigma^{(k)}-\Sigma^\star\|_F
\lesssim \sqrt{|\gA|\,\tfrac{\log p}{n}}
+ \|(\Sigma^\star)_{\gA^c}\|_1/\sqrt{|\gA|}
\).
This completes the proof.
\end{proof}

\subsection{Proof of theorem \ref{thm:total-gl}}
\begin{proof}
The graphical Lasso objective is \(f_2(\Theta) = \Tr(S\Theta)-\log\det\Theta + \lambda\|\Theta\|_{1,\mathrm{off}}\). Let $\Delta:=\widehat\Theta^\star-\Theta^\star$ and denote the off–diagonal support $\gB=\{(i,j):i\neq j,\ \theta^\star_{ij}\neq 0\}$, with size $|\gB|$. For any $\Theta\succ0$ and direction $U$,
\[
\nabla^2(\Tr(S\Theta)-\log\det\Theta)[U,U]
=\langle \Theta^{-1}U\Theta^{-1},U\rangle
=\|\Theta^{-1/2}U\Theta^{-1/2}\|_F^2.
\]
Hence on the SPD band $\{\Theta:\ \epsilon I\preceq \Theta\preceq MI\}$,
\[
\nabla^2(\Tr(S\Theta)-\log\det\Theta)[U,U]\ \ge\ \frac{1}{M^2}\,\|U\|_F^2.
\]
By the mean–value form of the Bregman divergence,
\begin{equation}\label{eq:gl-breg-lb}
\Tr(S\widehat\Theta^\star)-\log\det \widehat\Theta^\star
\ \ge\
\Tr(S\Theta^\star)-\log\det \Theta^\star
+\big\langle S-(\Theta^\star)^{-1},\,\Delta\big\rangle
+\frac{1}{2M^2}\|\Delta\|_F^2,\ \Delta:=\widehat\Theta^\star-\Theta^\star .
\end{equation}

Optimality of $\widehat\Theta^\star$ gives
\[
0\ \ge\ f_2(\widehat\Theta^\star)-f_2(\Theta^\star)
= \Big(\Tr(S\widehat\Theta^\star)-\log\det \widehat\Theta^\star\Big)
- \Big(\Tr(S\Theta^\star)-\log\det \Theta^\star\Big)
+ \lambda\big(\|\widehat\Theta^\star\|_{1,\mathrm{off}}-\|\Theta^\star\|_{1,\mathrm{off}}\big).
\]
Combining with \eqref{eq:gl-breg-lb} and using $\nabla (\Tr(S\Theta^\star)-\log\det\Theta^\star) = S-\Sigma^\star$,
\begin{equation}\label{eq:key-start}
\frac{1}{2M^2}\|\Delta\|_F^2
\ \le\ -\langle S-\Sigma^\star,\Delta\rangle
-\lambda\big(\|\widehat\Theta^\star\|_{1,\mathrm{off}}-\|\Theta^\star\|_{1,\mathrm{off}}\big).
\end{equation}

On the event $\|S-\Sigma^\star\|_\infty\le \lambda/2$ (which holds w.p.\ $\ge 1-2p^{-c_0}$ when $\lambda\ge 2\|S-\Sigma^\star\|_\infty$),
\[
|\langle S-\Sigma^\star,\Delta\rangle|
\ \le\ \|S-\Sigma^\star\|_\infty\,\|\Delta\|_1
\ \le\ \frac{\lambda}{2}\,\|\Delta\|_1.
\]
For the $\ell_1$ term, by decomposability on off–diagonals and allowing approximate sparsity,
\[
\|\widehat\Theta^\star\|_{1,\mathrm{off}}-\|\Theta^\star\|_{1,\mathrm{off}}
\ \ge\ \|\Delta_{\gB^c}\|_1-\|\Delta_{\gB}\|_1-2\|(\Theta^\star)_{\gB^c}\|_{1}.
\]
Plugging these two bounds into \eqref{eq:key-start} yields
\begin{equation}\label{eq:key-ineq}
\frac{1}{2M^2}\|\Delta\|_F^2
\ \le\ \frac{\lambda}{2}\big(\|\Delta_{\gB}\|_1+\|\Delta_{\gB^c}\|_1\big)
+ \lambda\big(\|\Delta_{\gB}\|_1-\|\Delta_{\gB^c}\|_1+2\|(\Theta^\star)_{\gB^c}\|_1\big).
\end{equation}
Rearranging,
\begin{equation}\label{eq:cone}
\|\Delta_{\gB^c}\|_1 \ \le\ 3\|\Delta_{\gB}\|_1 + 4\|(\Theta^\star)_{\gB^c}\|_1
\qquad\text{(cone condition).}
\end{equation}

From \eqref{eq:key-ineq}, drop the negative term $-(\lambda/2)\|\Delta_{\gB^c}\|_1$, and use $\|\Delta_{\gB}\|_1\le \sqrt{|\gB|}\,\|\Delta\|_F$:
\[
\frac{1}{2M^2}\|\Delta\|_F^2
\ \le\ \frac{3\lambda}{2}\sqrt{|\gB|}\,\|\Delta\|_F + 2\lambda\|(\Theta^\star)_{\gB^c}\|_1.
\]
Solving the quadratic inequality in $x=\|\Delta\|_F$ gives
\begin{equation}\label{eq:stat-frob}
\|\widehat\Theta^\star-\Theta^\star\|_F
\ \le\ M^2\Big(3\lambda\sqrt{|\gB|}+ 4\,\|(\Theta^\star)_{\gB^c}\|_1/\sqrt{|\gB|}\Big).
\end{equation}

On the band $\{\Theta:\ \epsilon I\preceq \Theta\preceq MI\}$, the function $f_2$ is $1/M^2$–strongly convex. Therefore, for the $k$-th iterate,
\[
\varepsilon_{\mathrm{opt}}^{\mathrm{GL}}(k)
:= f_2(\Theta^{(k)})-f_2(\widehat\Theta^\star)
\ \ge\ \frac{1}{2M^2}\,\|\Theta^{(k)}-\widehat\Theta^\star\|_F^2
\Longrightarrow
\|\Theta^{(k)}-\widehat\Theta^\star\|_F \ \le
\sqrt{2M^2\,\varepsilon_{\mathrm{opt}}^{\mathrm{GL}}(k)}.
\]
Finally, by the triangle inequality and \eqref{eq:stat-frob},
\[
\|\Theta^{(k)}-\Theta^\star\|_F
\ \le\ \|\Theta^{(k)}-\widehat\Theta^\star\|_F + \|\widehat\Theta^\star-\Theta^\star\|_F
\ \le\ \sqrt{2M^2\,\varepsilon_{\mathrm{opt}}^{\mathrm{GL}}(k)}
\;+\; M^2\Big(3\lambda\sqrt{|\gB|}+ 4\,\|(\Theta^\star)_{\gB^c}\|_1/\sqrt{|\gB|}\Big).
\]
Choosing $\lambda\asymp \sqrt{\tfrac{\log p}{n}}$ on the concentration event and stopping when
$\varepsilon_{\mathrm{opt}}^{\mathrm{GL}}(k)\ \lesssim\ |\gB|\,\tfrac{\log p}{n}$ yields the stated rate.
\end{proof}

\subsection{Proof of Proposition \ref{prop:lipschitz_continuity}}
\begin{proof}
Recall that the weighted proximal operator $\operatorname{prox}_{w,F}:\mathbb{R}^{p\times p}\to\mathbb{R}^{p\times p}$ by
\[
\operatorname{prox}_{w,F}(M)
:=\arg\min_{X\in\mathbb{R}^{p\times p}}
\Big\{F(X)+\tfrac{1}{2}\big\|\tfrac{1}{\sqrt{w}}\circ(X-M)\big\|_F^2\Big\},
\]
where $\tfrac{1}{\sqrt{w}}$ is taken element-wise. Define the weighted inner product
\[
\langle X, \tfrac{1}{w}\circ Y\rangle := \sum_{i,j} X_{ij}\,\frac{Y_{ij}}{w_{ij}},
\]
and the associated norm
\[
\|X\|_{w^{-1}}^2 := \langle X, \tfrac{1}{w}\circ X\rangle.
\]
We observe that
\begin{equation*}
\begin{aligned}
\big\|\tfrac{1}{\sqrt{w}}\circ(X-M)\big\|_F^2
&=
\sum_{i,j}\Big(\frac{X_{ij}-M_{ij}}{\sqrt{w_{ij}}}\Big)^2=
\sum_{i,j}\frac{(X_{ij}-M_{ij})^2}{w_{ij}}
\\&=
\big\langle X-M,\tfrac{1}{w}\circ(X-M)\big\rangle=
\|X-M\|_{w^{-1}}^2.
\end{aligned}
\end{equation*}

\medskip\noindent
\textbf{Step 1: Nonexpansiveness in the weighted norm $\|\cdot\|_{w^{-1}}$.}


Let $M_1,M_2\in\mathbb{R}^{p\times p}$ and $X_1 := \operatorname{prox}_{w,F}(M_1),\;
X_2 := \operatorname{prox}_{w,F}(M_2).$
The optimal condition yields
\begin{equation*}
\begin{aligned}
0 \in \partial F(X_1) + \tfrac{1}{w}\circ(X_1 - M_1),\quad
0 \in \partial F(X_2) + \tfrac{1}{w}\circ(X_2 - M_2).
\end{aligned}
\end{equation*}
Thus there exist $G_1\in\partial F(X_1)$ and $G_2\in\partial F(X_2)$ such that
\begin{equation*}
G_1 + \tfrac{1}{w}\circ(X_1 - M_1) = 0, \quad
G_2 + \tfrac{1}{w}\circ(X_2 - M_2) = 0.
\end{equation*}
Equivalently,
\begin{equation*}
G_1 = \tfrac{1}{w}\circ(M_1 - X_1),\quad
G_2 = \tfrac{1}{w}\circ(M_2 - X_2).
\end{equation*}
Since $F$ is convex, its subdifferential $\partial F$ is a monotone operator. Hence
\begin{equation*}
\big\langle X_1 - X_2,\; G_1 - G_2\big\rangle \;\ge\; 0.
\end{equation*}
Denote that
\begin{equation*}
\begin{aligned}
\Delta X := X_1 - X_2,\qquad
\Delta M := M_1 - M_2,
\end{aligned}
\end{equation*}
then we have
\begin{equation*}
\begin{aligned}
0
&\le
\big\langle \Delta X,\; G_1 - G_2\big\rangle\\
&=
\big\langle \Delta X,\; \tfrac{1}{w}\circ(\Delta M - \Delta X)\big\rangle\\
&=
\big\langle \Delta X,\; \tfrac{1}{w}\circ\Delta M\big\rangle
-
\big\langle \Delta X,\; \tfrac{1}{w}\circ\Delta X\big\rangle.
\end{aligned}
\end{equation*}
Rearranging gives
\begin{equation*}
\begin{aligned}
\|\Delta X\|_{w^{-1}}^2
\;\le\;
\big\langle \Delta X,\; \tfrac{1}{w}\circ\Delta M\big\rangle.
\end{aligned}
\end{equation*}
Now apply the Cauchy--Schwarz inequality:
\begin{equation*}
\begin{aligned}
\big|\big\langle \Delta X,\; \tfrac{1}{w}\circ\Delta M\big\rangle\big|=\big|\big\langle \tfrac{1}{\sqrt{w}}\circ \Delta X,\; \tfrac{1}{\sqrt{w}}\circ \Delta M\big\rangle\big|
\;\le\; \|\tfrac{1}{\sqrt{w}}\circ \Delta X\|\;\|\tfrac{1}{\sqrt{w}}\circ \Delta M\|=
\|\Delta X\|_{w^{-1}}\;\|\Delta M\|_{w^{-1}}.
\end{aligned}
\end{equation*}
Hence
\begin{equation*}
\begin{aligned}
\|\Delta X\|_{w^{-1}}^2
\;\le\;
\|\Delta X\|_{w^{-1}}\;\|\Delta M\|_{w^{-1}}.
\end{aligned}
\end{equation*}
This implies that 
\begin{equation*}
\begin{aligned}
\|\Delta X\|_{w^{-1}}
\;\le\;
\|\Delta M\|_{w^{-1}}.
\end{aligned}
\end{equation*}
That is,
\begin{equation*}
\begin{aligned}
\big\|\operatorname{prox}_{w,F}(M_1) - \operatorname{prox}_{w,F}(M_2)\big\|_{w^{-1}}
\;\le\;
\|M_1 - M_2\|_{w^{-1}},
\qquad\forall\,M_1,M_2.
\end{aligned}
\end{equation*}
Thus $\operatorname{prox}_{w,F}$ is nonexpansive in the norm $\|\cdot\|_{w^{-1}}$.

\medskip\noindent
\textbf{Step 2: From $\|\cdot\|_{w^{-1}}$ to the Frobenius norm $\|\cdot\|_F$.}

By the elementwise bounds on $w$, for any $Z\in\mathbb{R}^{p\times p}$ we have
\begin{equation*}
\begin{aligned}
\frac{1}{c_2}\|Z\|_F^2\le\|Z\|_{w^{-1}}^2\le
\frac{1}{c_1}\|Z\|_F^2,
\end{aligned}
\end{equation*}
Then
\begin{equation*}
\begin{aligned}
\|\Delta X\|_F
\le
\sqrt{c_2}\,\|\Delta X\|_{w^{-1}}\le
\sqrt{c_2}\,\|\Delta M\|_{w^{-1}}\le
\sqrt{\frac{c_2}{c_1}}\;\|\Delta M\|_F.
\end{aligned}
\end{equation*}
That is,
\begin{equation*}
\begin{aligned}
\big\|\operatorname{prox}_{w,F}(M_1) - \operatorname{prox}_{w,F}(M_2)\big\|_F
\;\le\;
\sqrt{\frac{c_2}{c_1}}\;\|M_1 - M_2\|_F
\qquad\forall\,M_1,M_2\in\mathbb{R}^{p\times p}.
\end{aligned}
\end{equation*}
Therefore $\operatorname{prox}_{w,F}$ is lipschitz continuous with lipschitz constant at most $\sqrt{{c_2}/c_1}$. This completes the proof. 

\end{proof}

\subsection{Proof of Theorem \ref{thm:nn_approximate}}
\begin{proof}
Proposition \ref{prop:lipschitz_continuity} implies that $\operatorname{vec}(\operatorname{prox}_{w,F}(\cdot))$ is $\sqrt{c_2/c_1}$-lipschitz continuous. Then the desired result is a direct corollary of Propositon 1 and Propositon 6 of \cite{Bach2017jmlr}.
\end{proof}

\section{Performance of CVXPY and MMA under the sparse covariance structure} \label{CVXPYMMAsparse}
In this section, Table \ref{SparsetableMMACVXPY} reports the simulation results for ADMM, LADMM, TOSA, PFBS, FISTA, as well as CVXPY and MMA, under a sparse covariance structure (sparsity level \(q = 0.1\)) with dimensions \(p = 1000, 2000, 3000\).

\begin{table}[H]
\scriptsize
\centering
\caption{Experimental results of different methods on the sparse covariance structure}\label{SparsetableMMACVXPY}
\setlength{\tabcolsep}{1.5mm}{
\begin{tabular}{c|c|ccccccc}
\hline
Metric & Dimension & ADMM & LADMM & TOSA & PFBS & FISTA & CVXPY & MMA\\
\hline
\multirow{3}{*}{\tabincell{c}{Time (s)}} 
& \(p=1000\) & 1.696$\times 10^{2}$ & 1.110$\times 10^{2}$ & 2.662$\times 10^{2}$ & 4.011$\times 10^{1}$ & 2.374$\times 10^{2}$ & 2.427$\times 10^{2}$ & 1.157$\times 10^{4}$ \\
& \(p=2000\) & 6.248$\times 10^{2}$ & 3.582$\times 10^{2}$ & 9.583$\times 10^{2}$ & 1.582$\times 10^{2}$ & 1.082$\times 10^{3}$ & 1.602$\times 10^{3}$ & 3.817$\times 10^{4}$ \\
& \(p=3000\) & 1.452$\times 10^{3}$ & 8.957$\times 10^{2}$ & 2.432$\times 10^{3}$ & 3.849$\times 10^{2}$ & 3.093$\times 10^{3}$ & 5.038$\times 10^{3}$ & 6.702$\times 10^{5}$ \\
\hline
\multirow{3}{*}{\tabincell{c}{Frobenius norm}} 
& \(p=1000\) & 1.252$\times 10^{3}$ & 1.252$\times 10^{3}$ & 1.251$\times 10^{3}$ & 1.252$\times 10^{3}$ & 1.252$\times 10^{3}$ & 1.252$\times 10^{3}$ & 1.251$\times 10^{3}$ \\
& \(p=2000\) & 3.582$\times 10^{3}$ & 3.582$\times 10^{3}$ & 3.577$\times 10^{3}$ & 3.582$\times 10^{3}$ & 3.582$\times 10^{3}$ & 3.582$\times 10^{3}$ & 3.600$\times 10^{3}$ \\
& \(p=3000\) & 6.649$\times 10^{3}$ & 6.649$\times 10^{3}$ & 6.639$\times 10^{3}$ & 6.649$\times 10^{3}$ & 6.649$\times 10^{3}$ & 6.649$\times 10^{3}$ & 6.738$\times 10^{3}$ \\
\hline
\multirow{3}{*}{\tabincell{c}{Nuclear norm}} 
& \(p=1000\) & 3.002$\times 10^{4}$ & 3.002$\times 10^{4}$ & 3.000$\times 10^{4}$ & 3.002$\times 10^{4}$ & 3.002$\times 10^{4}$ & 3.002$\times 10^{4}$ & 3.002$\times 10^{4}$ \\
& \(p=2000\) & 1.138$\times 10^{5}$ & 1.138$\times 10^{5}$ & 1.133$\times 10^{5}$ & 1.138$\times 10^{5}$ & 1.138$\times 10^{5}$ & 1.138$\times 10^{5}$ & 1.144$\times 10^{5}$ \\
& \(p=3000\) & 2.374$\times 10^{5}$ & 2.374$\times 10^{5}$ & 2.360$\times 10^{5}$ & 2.374$\times 10^{5}$ & 2.374$\times 10^{5}$ & 2.374$\times 10^{5}$ & 2.410$\times 10^{5}$ \\
\hline
\end{tabular}
}
\end{table}

From Table~\ref{SparsetableMMACVXPY}, it is evident that under the sparse covariance structure with \(q=0.1\), neither MMA nor CVXPY provides a practically meaningful advantage. We therefore exclude these two methods from the remaining experiments. For MMA, the main issue is computational efficiency. MMA is consistently the slowest method, and its runtime grows dramatically with the dimension: \(1.157\times 10^{4}\) seconds at \(p=1000\), \(3.817\times 10^{4}\) seconds at \(p=2000\), and \(6.702\times 10^{5}\) seconds at \(p=3000\). In contrast, first-order splitting methods such as PFBS/LADMM/ADMM typically complete within \(10^{1}\)-\(10^{3}\) seconds at the same dimensions. This severe lack of scalability renders MMA impractical in our large-scale simulations. For CVXPY, the main issue is solution quality rather than speed. CVXPY does not yield any noticeable improvement in either the Frobenius norm or the nuclear norm: across all three dimensions, the Frobenius and nuclear norms produced by CVXPY are essentially identical to those achieved by ADMM/LADMM/FISTA. For instance, at \(p=1000\), CVXPY attains a Frobenius norm of \(1.252\times 10^{3}\) and a nuclear norm of \(3.002\times 10^{4}\), which are nearly the same as the corresponding values of the competing methods; similar behavior is observed for \(p=2000\) and \(p=3000\). Moreover, CVXPY is also slower than efficient alternatives (e.g., \(5.038\times 10^{3}\) seconds at \(p=3000\) versus \(3.849\times 10^{2}\) seconds for PFBS), leading to a strictly worse time--accuracy trade-off.

\end{appendix} 

\newpage
\bibliography{./bib/bib.bib}
\bibliographystyle{elsarticle-num-names}

\end{document}